\PassOptionsToPackage{numbers,sort&compress}{natbib}
\PassOptionsToPackage{main}{neurips_2026}
\documentclass{article}

\usepackage[final]{neurips_2026}
\renewcommand{\acksection}{\section*{Acknowledgements}}
\usepackage[utf8]{inputenc}
\usepackage[T1]{fontenc}
\usepackage{hyperref}
\usepackage{url}
\usepackage{booktabs}
\usepackage{amsmath,amssymb,amsfonts,amsthm,mathtools}
\usepackage{bm}
\usepackage{nicefrac}
\usepackage{microtype}
\usepackage{xcolor}
\usepackage{colortbl}
\usepackage{neurips_gold_boxes}
\usepackage{graphicx}
\usepackage{subcaption}
\usepackage{placeins}
\usepackage{needspace}
\usepackage{tabularx}
\usepackage{array}
\usepackage{enumitem}
\usepackage{tikz}
\usetikzlibrary{arrows.meta,positioning,fit,shapes.geometric}

\makeatletter
\@ifundefined{nolinenumbers}{\newenvironment{nolinenumbers}{}{}}{}
\makeatother

\definecolor{paperblue}{HTML}{2F5D9B}
\definecolor{grapefruit}{HTML}{E06A6A}
\definecolor{keyteal}{HTML}{159C90}
\definecolor{keyteallight}{HTML}{F7FCFB}
\definecolor{keyshadow}{HTML}{C9D0CD}
\definecolor{resultblue}{HTML}{6F8296}
\definecolor{resultblueback}{HTML}{F5F8FA}
\definecolor{resultgreen}{HTML}{74846C}
\definecolor{resultgreenback}{HTML}{F7F9F5}
\definecolor{resultgold}{HTML}{7C866F}
\definecolor{resultgoldback}{HTML}{F7F9F5}
\definecolor{resultteal}{HTML}{6F6296}
\definecolor{resulttealback}{HTML}{F7F5FC}
\definecolor{boxink}{HTML}{2A2238}
\definecolor{figurebarrier}{HTML}{2F5D86}
\newcommand{\boxtextbf}[1]{{\bfseries\color{boxink}#1}}
\providecolor{resnethead}{HTML}{F7FAFE}
\providecolor{resnetbest}{HTML}{EDF4FF}
\providecolor{resnetsecond}{HTML}{F7FAFF}
\providecolor{densehead}{HTML}{F7FBF8}
\providecolor{densebest}{HTML}{EEF8F2}
\providecolor{densesecond}{HTML}{F8FCFA}
\providecolor{distilhead}{HTML}{FFF9F0}
\providecolor{distilbest}{HTML}{FFF3E3}
\providecolor{distilsecond}{HTML}{FFFAF3}
\providecolor{qwenhead}{HTML}{FFF6F6}
\providecolor{qwenbest}{HTML}{FFEDED}
\providecolor{qwensecond}{HTML}{FFF8F8}

\newtcolorbox{keyideabox}{
  enhanced,
  colback=keyteallight,
  colframe=keyteal,
  boxrule=0.55pt,
  arc=1.4mm,
  left=8pt,
  right=8pt,
  top=13pt,
  bottom=6pt,
  before skip=15pt,
  after skip=6pt,
  drop shadow=keyshadow,
  overlay={
    \node[
      anchor=west,
      fill=keyteal,
      text=white,
      font=\bfseries\large,
      rounded corners=1.2mm,
      inner xsep=10pt,
      inner ysep=4.5pt
    ] at ([xshift=7mm,yshift=0.5mm]frame.north west) {Key Idea};
  }
}

\newtcolorbox{mainresultbox}[3]{
  enhanced,
  breakable,
  colback=#3,
  colframe=#2!38!white,
  colbacktitle=#2,
  boxrule=0.35pt,
  arc=1.8mm,
  outer arc=1.8mm,
  left=7pt,
  right=7pt,
  top=6pt,
  bottom=6pt,
  before skip=7pt,
  after skip=7pt,
  fonttitle=\bfseries\color{white},
  fontupper=\color{boxink},
  before upper={\let\textbf\boxtextbf},
  coltitle=white,
  title={#1}
}

\title{Pairwise AUC Optimization Needs Corrective Power:\\[0.20em]
A Unified View}
\author{%
  Jia Chen$^{1}$ \quad
  Zhiyong Yang$^{1,*}$ \quad
  Shilong Bao$^{2}$ \quad
  Qianqian Xu$^{2}$ \quad
  Qingming Huang$^{1,*}$ \\
  $^{1}$School of Computer Science and Technology, University of Chinese Academy of Sciences \\
  $^{2}$Institute of Computing Technology, Chinese Academy of Sciences \\
  Beijing, China \\
  \texttt{chenjia26@mails.ucas.ac.cn}, \texttt{\{yangzhiyong21,qmhuang\}@ucas.ac.cn} \\
  \texttt{\{baoshilong,xuqianqian\}@ict.ac.cn}
}

\newtheorem{theorem}{Theorem}
\newtheorem{proposition}{Proposition}
\newtheorem{corollary}{Corollary}
\newtheorem{lemma}{Lemma}

\newenvironment{mainboxedtheorem}[2]{%
  \Needspace{4\baselineskip}%
  \refstepcounter{theorem}%
  \begin{nolinenumbers}%
  \begin{mainresultbox}{Theorem~\thetheorem\ (#1)}{resultblue}{resultblueback}%
  \label{#2}\ignorespaces%
}{%
  \end{mainresultbox}%
  \end{nolinenumbers}\ignorespacesafterend%
}
\newenvironment{mainboxedproposition}[2]{%
  \Needspace{4\baselineskip}%
  \refstepcounter{proposition}%
  \begin{nolinenumbers}%
  \begin{mainresultbox}{Proposition~\theproposition\ (#1)}{resultgreen}{resultgreenback}%
  \label{#2}\ignorespaces%
}{%
  \end{mainresultbox}%
  \end{nolinenumbers}\ignorespacesafterend%
}
\newenvironment{mainboxedcorollary}[2]{%
  \Needspace{4\baselineskip}%
  \refstepcounter{corollary}%
  \begin{nolinenumbers}%
  \begin{mainresultbox}{Corollary~\thecorollary\ (#1)}{resultteal}{resulttealback}%
  \label{#2}\ignorespaces%
}{%
  \end{mainresultbox}%
  \end{nolinenumbers}\ignorespacesafterend%
}

\newcommand{\E}{\mathbb{E}}
\newcommand{\Pbb}{\mathbb{P}}
\newcommand{\R}{\mathbb{R}}
\newcommand{\1}{\mathbf{1}}
\newcommand{\Var}{\mathrm{Var}}
\newcommand{\Cov}{\mathrm{Cov}}
\newcommand{\dist}{\mathrm{dist}}

\newcolumntype{Y}{>{\raggedright\arraybackslash}X}

\DeclareTextFontCommand{\textbf}{\bfseries\color{grapefruit}}

\hypersetup{
  hypertexnames=false,
  colorlinks=true,
  linkcolor=paperblue,
  citecolor=grapefruit,
  urlcolor=paperblue
}

\begin{document}

\maketitle
\begingroup
\renewcommand{\thefootnote}{}
\footnotetext{\textsuperscript{*}: corresponding authors}
\endgroup

\begin{abstract}
Direct pairwise AUC-surrogate training is attractive because it trains toward the ranking metric used at test time, especially under class imbalance. This creates a direct alignment between the training objective and the evaluation metric, but such metric alignment alone does not guarantee that stochastic training remains corrective. We study this gap through \emph{Corrective Power}: whether pairwise mini-batch updates produce stronger and more reliable improvements for more severe positive-negative ranking errors. Our main observation is that direct pairwise training can enter regions where severe errors remain common but no longer produce useful corrective motion. In such regions, even the mini-batch noise that would normally help SGD move along the steepest severe-error correction direction can become much weaker. 
In this paper, we seek a theoretical explanation for this by exploring the degenerate property of the U-statistics formed from the pairwise gradients therein. Our results provide a mechanistic explanation for several empirical patterns: (a) squared-hinge loss is often an effective AUC surrogate; (b) cross-entropy warm-up can help models that enter the AUC phase from a poor starting region; and (c) parameter-efficient fine-tuning from a pretrained model can reduce that need by starting closer to a favorable ranking geometry. Experiments on 8 CV tasks and 8 NLP tasks, using ResNet50, DenseNet121, DistilBERT-base, and Qwen2.5-1.5B-Instruct, support these predictions.
\end{abstract}

\section{Introduction}

AUC has a simple but important appeal: it is a threshold-free measure of ranking quality, evaluating whether positives receive higher scores than negatives rather than fixing a single operating point~\citep{hanley1982meaning,bradley1997use,fawcett2006introduction}. This perspective is especially important under class imbalance, where error rate and pointwise classification losses can be poor proxies for the positive-negative ordering that AUC measures~\citep{cortes2003auc,menon2016bipartite,yang2022auc}. A natural response is to optimize pairwise AUC surrogates directly, as studied in both classical and modern AUC maximization~\citep{yan2003optimizing,herschtal2004optimising,rakotomamonjy2004optimizing,liu2019stochastic,yuan2021large}. The difficulty we study is that the same pairwise structure that aligns training with AUC can also make the stochastic optimization dynamics fragile.

Several empirical patterns indicate that pairwise AUC training cannot be achieved simply by replacing pointwise empirical risk minimization (ERM). First, direct pairwise AUC-surrogate training from an ordinary starting point can be slow or unstable. Deep-AUC pipelines therefore often use cross-entropy (CE) warm-up or staged optimization before switching to an AUC-family objective~\citep{liu2019stochastic,yuan2021large,yuan2022compositional,yang2021deep,yang2021all,yang2022optimizing}. Second, the choice of a pairwise surrogate can determine whether training makes progress once the AUC phase begins. On difficult tasks, squared-hinge-style AUC margins have shown strong empirical performance. A deep AUC system built around this margin achieved first place in the Stanford CheXpert competition~\citep{yuan2021large}. Third, the dependence on warm-up becomes much weaker when performing parameter-efficient fine-tuning (\textit{e.g.}, using LoRA or adapters) for LLMs~\citep{houlsby2019parameter,hu2021lora,dettmers2023qlora}. Taken together, these patterns highlight the difficulty of training: the model may already contain many ranking errors, but the next pairwise update might not focus on the worst ones. A useful explanation must therefore look beyond objective alignment and ask whether the stochastic update still allocates its corrective effort to severely misranked pairs. Rather than treating the observations as separate heuristics or implementation choices, we argue that they are symptoms of one mechanism.

\begin{figure}[t]
  \centering
  \includegraphics[width=0.90\linewidth,clip]{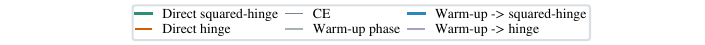}\vspace{0.1em}

  \captionsetup[subfigure]{font=footnotesize,labelfont=bf,justification=centering,singlelinecheck=false,skip=1pt}
  \begin{subfigure}[t]{0.315\linewidth}
  \centering
  \includegraphics[width=\linewidth,clip]{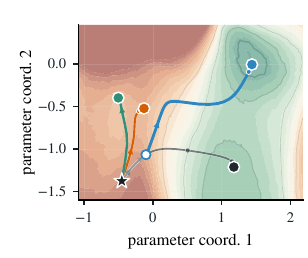}
  \caption{Bad initialization}
  \end{subfigure}\hfill
  \begin{subfigure}[t]{0.315\linewidth}
  \centering
  \includegraphics[width=\linewidth,clip]{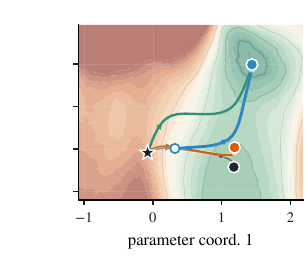}
  \caption{Slightly better initialization}
  \end{subfigure}\hfill
  \begin{subfigure}[t]{0.315\linewidth}
  \centering
  \includegraphics[width=\linewidth,clip]{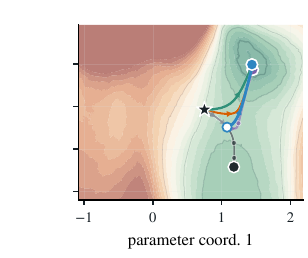}
  \caption{Good initialization}
  \end{subfigure}
  \caption{\textbf{AUC dynamics comparison under different initialization.}
  The background color shows AUC on a fixed two-dimensional parameter slice. Curves are simulated training paths; stars and large circles mark starts and final iterates. Panels \textbf{(a)}, \textbf{(b)}, and \textbf{(c)} use bad, slightly better, and good initializations. Details are in Appendix~\ref{app:fig1-simulation}.}
  \label{fig:intro-controlled-dynamics}
\end{figure}

Figure~\ref{fig:intro-controlled-dynamics} checks the training dynamics with the AUC landscape fixed while the initialization and update rule vary. In this controlled simulation, from a bad start, direct pairwise training moves weakly, and CE warm-up reaches a basin where squared-hinge succeeds. From a slightly better start, direct squared-hinge can already make corrective progress while hinge remains weak. From a good start, both hinge and squared-hinge can work without a warm-up. This motivates \emph{Corrective Power}: reliable direct pairwise training requires severe ranking errors not merely to incur pairwise loss, but to generate finite-batch updates that actively push the model toward correcting them.

This paper examines under what conditions such corrective power becomes weak. The main clue is the useful motion along a severe-error correction direction, meaning a direction that locally decreases the loss contributed by severely misranked pairs. In a \textbf{bad region}, many severely misranked pairs do not automatically create such motion. Their gradient contributions may cancel, point away from the correction direction, or receive too little surrogate weight. Along this direction, SGD has a projected mean drift and finite-batch noise. In a bad region, severe errors
remain common, but the projected mean drift is weak. Then progress depends on the noise in the same direction. For pairwise AUC training, the U-statistic structure can reduce this useful projected variance from order $B^{-1}$ to order $B^{-2}$, changing the projected escape-time scale from order $B$ to order $B^2$. 

Based on this, we figure out two ways to enter or leave the bad regions.  The first is the choice of surrogate losses. A useful pairwise loss should put enough update mass on severe active violations while ignoring already separated pairs. Squared-hinge has this profile within the standard surrogate family. The second is changing where the pairwise dynamics begin. Warm-up and pretrained parameter-efficient fine-tuning exactly serve this purpose. They can therefore help avoid the poor-start issue.

In short, this paper uses \emph{Corrective Power} to explain when direct pairwise AUC-surrogate training works and when it loses corrective strength. The contributions are:
\begin{enumerate}[leftmargin=5mm]
  \item \textbf{A mechanism for loss of Corrective Power.} We define bad regions through two observable ingredients: many severe positive-negative misrankings and weak projected motion along the steepest severe-error correction direction (Section~\ref{sec:bad-regions}).
  \item \textbf{A weight-based explanation of surrogate behavior.} We analyze standard pairwise losses by examining the weights they assign to margins during updates. Squared-hinge is well matched to difficult active pairs within this family: it increases correction with violation severity while turning off easy pairs (Section~\ref{sec:surrogate-weights}).
  \item \textbf{Understanding warm-up and pretraining.} The corrective-power view predicts that CE warm-up can help when the AUC-surrogate phase starts in a poor region, while pretrained low-dimensional trainable subspaces can make a separate warm-up less necessary. We test these predictions on computer vision and NLP tasks with ResNet, DenseNet, DistilBERT, and Qwen models (Sections~\ref{sec:warmup-pretraining} and~\ref{sec:experiments}).
\end{enumerate}

\section{Bad Regions in the Stochastic Dynamics of Pairwise AUC Optimization}\label{sec:bad-regions}

We now formalize the failure mode from the introduction. Let \(f_\theta:\mathcal X\to\mathbb R\) be the scoring function. Let
\(X^+\sim P^+\) and \(X^-\sim P^-\) be independent positive and negative examples.
We use the positive-margin convention
\[
m_\theta(x^+,x^-)=f_\theta(x^+)-f_\theta(x^-).
\]
For a pairwise surrogate \(\ell\), define
\[
R_\ell(\theta)=\mathbb E[\ell(m_\theta(X^+,X^-))],
\qquad
w_\ell(x)=-\ell'(x).
\]
Although training minimizes the surrogate risk \(R_\ell\), the quantity we ultimately seek to reduce is the hard pairwise error \(\mathbb P(m_\theta(X^+,X^-)\le 0)\), with particular attention to its severe-error subset.
For sampled examples, let \(x_i^+\) and \(x_j^-\) denote the \(i\)-th sampled positive example and the \(j\)-th sampled negative example, respectively, and write
\[
m_{ij}=f_\theta(x_i^+)-f_\theta(x_j^-),\quad
g_i^+=\nabla_\theta f_\theta(x_i^+),\quad
g_j^-=\nabla_\theta f_\theta(x_j^-).
\]
Given \(B_+\) positives and \(B_-\) negatives in a batch, the pairwise mini-batch update is
\[
\widehat g^\ell_{B_+,B_-}(\theta)
=
\frac{1}{B_+B_-}
\sum_{i=1}^{B_+}\sum_{j=1}^{B_-}
w_\ell(m_{ij}) (g^-_j-g^+_i),
\tag{1}
\]
With fresh mini-batches, the estimator is unbiased in the sense that
$\mathbb E[\widehat g^\ell_{B_+,B_-}(\theta)]=\nabla R_\ell(\theta).
$

Pairwise SGD updates the parameters by
\[
\theta_{t+1}=\theta_t-\eta \widehat g^\ell_{B_+,B_-}(\theta_t),
\]
where \(\widehat g^\ell_{B_+,B_-}(\theta_t)\) is a stochastic estimate of
\(\nabla R_\ell(\theta_t)\). Thus the update direction has two parts: the mean descent
direction \(-\nabla R_\ell(\theta_t)\) and the finite-batch noise around it.
Here, the finite-batch noise is the mini-batch gradient fluctuation \(\widehat g^\ell_{B_+,B_-}(\theta_t)-\nabla R_\ell(\theta_t)\) relative to the full-sample gradient. We ask
whether these two parts still give useful motion in a direction that reduces severe
ranking errors.

Here, \emph{local correction direction} is terminology introduced in this paper for a direction that locally reduces the loss contributed by severely misranked pairs. We are now ready to measure the useful motion along the local correction direction \(\bm v\). Along this
direction, the mean drift is
\[
d_{\bm v}(\theta)=\langle-\nabla R_\ell(\theta),\bm v\rangle .
\]
If \(d_{\bm v}(\theta)\) is large, SGD can progress even with small noise. Many severe misranked pairs do not by themselves make \(d_{\bm v}(\theta)\) large. The
mean drift is a weighted directional average, and these terms may not add up along
\(\bm v\). The hard
case is when severe errors remain common but \(d_{\bm v}(\theta)\) is small. Then progress
depends on the finite-batch noise along the same direction.

The next question is which direction to measure. We focus on hard, severely misranked pairs. To isolate this part, we introduce an intermediate loss
\(\Phi_{M,\ell}(\theta)\). For \(M>0\), define
\[
\Phi_{M,\ell}(\theta)
=
\mathbb E\!\left[
\mathbf 1\{m_\theta(X^+,X^-)\le -M\}
\,\ell(m_\theta(X^+,X^-))
\right].
\]
This is the severe-error part of the current pairwise objective. At a reference point
\(\theta_0\), if \(\nabla\Phi_{M,\ell}(\theta_0)\ne 0\), set
\[
\bm v_{\rm sev}(\theta_0)
=
-\frac{\nabla\Phi_{M,\ell}(\theta_0)}
{\|\nabla\Phi_{M,\ell}(\theta_0)\|_2}.
\]
The direction is fixed before projected drift or variance is measured. It is the
steepest local direction for reducing the hard pairs, which makes perturbations along
\(\bm v_{\rm sev}\) a targeted probe of local corrective motion. \textbf{Motion in other directions may
change easy pairs or other parts of the objective, but it does not directly reduce the loss from hard examples.} We therefore measure corrective motion along
\(\bm v_{\rm sev}\) and interpret finite interventions in light of the step size, local
curvature, and subsequent trajectory. See Appendix~\ref{app:mechanism-diagnostics} for the implementation details.

We now study the finite-batch noise along the fixed correction direction \(\bm v\).
For this direction, define the scalar pairwise-gradient kernel
\[
\psi_{\theta,\bm v}(x^+,x^-)
=
\left\langle
w_\ell(m_\theta(x^+,x^-))
\big(\nabla_\theta f_\theta(x^-)-\nabla_\theta f_\theta(x^+)\big),
\bm v
\right\rangle .
\]
For the sampled positive-negative pair indexed by \((i,j)\), define
\[
\phi_{\theta,\bm v}(i,j)
=
\psi_{\theta,\bm v}(x_i^+,x_j^-)
=
\left\langle
w_\ell(m_{ij})(g_j^- - g_i^+),
\bm v
\right\rangle .
\]
A two-sample U-statistic averages a kernel over all sampled positive-negative pairs; its Hoeffding decomposition separates the resulting fluctuation into positive-example, negative-example, and residual pair-interaction components.
In the equal-size case, write
\(\widehat g^\ell_B(\theta)=\widehat g^\ell_{B,B}(\theta)\). Then
\(\langle \widehat g^\ell_B(\theta),\bm v\rangle\) is the average of \(\phi_{\theta,\bm v}(i,j)\), equivalently a two-sample U-statistic built
from \(\psi_{\theta,\bm v}\). The unequal-size case only changes the normalizing factors:
the two first-order terms scale as \(B_+^{-1}\) and \(B_-^{-1}\), and the residual
interaction term scales as \((B_+B_-)^{-1}\). All conclusions below remain the same.

The next result identifies the variance structure by decomposing the first-order Hoeffding projections and the residual pair-interaction term~\citep{lee2019u,serfling2009approximation}.
The first two components are the leading \(B^{-1}\) stochastic term. They are the only Hoeffding terms that contribute at order \(B^{-1}\) below.

\begin{mainboxedproposition}
    {Projected Hoeffding variance split}{prop:main-hoeffding-collapse}
Fix \(\theta\) and a unit direction \(\bm v\). Assume
\(\mathbb E[\psi_{\theta,\bm v}(X^+,X^-)^2]<\infty\). Let
\(\psi^{(1,0)}_{\theta,\bm v}\), \(\psi^{(0,1)}_{\theta,\bm v}\), and
\(\psi^{(1,1)}_{\theta,\bm v}\) be the positive-example, negative-example, and residual
Hoeffding components of \(\psi_{\theta,\bm v}\). Then, with \(B\) positives and \(B\)
negatives,
\[
\operatorname{Var}\!\left(
\left\langle \widehat g^\ell_B(\theta),\bm v\right\rangle
\right)
=
\frac{
\operatorname{Var}(\psi^{(1,0)}_{\theta,\bm v}(X^+))
+
\operatorname{Var}(\psi^{(0,1)}_{\theta,\bm v}(X^-))
}{B}
+
\frac{
\operatorname{Var}(\psi^{(1,1)}_{\theta,\bm v}(X^+,X^-))
}{B^2}.
\]
Thus the projected variance has an order-\(B^{-1}\) term unless the first-order
components vanish; exact order-\(B^{-2}\) scaling additionally requires a nonzero
residual pair-interaction variance. 
\end{mainboxedproposition}

See Appendix~\ref{app:pairwise-bad-region-dynamics} for the proof. In ordinary mini-batch behavior, replacing one example changes the projected update at the order \(B^{-1}\) variance scale. Pairwise AUC mini-batches have a second possibility: the two first-order projections can vanish, leaving only the weaker pair-interaction randomness at order \(B^{-2}\). Proposition~\ref{prop:app-severe-compression-small-projections} gives a sufficient condition: the severe-truncated row and column effects are nearly constant across examples; a large severe-misranking mass alone does not imply collapse. Partial compression leaves a small but nonzero order-\(B^{-1}\) coefficient, so real checkpoints can move toward without attaining exact \(B^{-2}\) behavior. Thus the update may remain noisy at the raw pair level while being weak in the parameter direction needed for hard example correction. Appendix~\ref{app:pairwise-bad-region-dynamics} gives the local argument, and Appendix~\ref{app:transformer-degenerate-region} gives Transformer realizations with the same projected bad-component structure.

Recall that if the projected descent drift \(-\mathbb{E}\!\left[
\left\langle \widehat g^\ell_B(\theta),\bm v\right\rangle
\right]\) is weak, then \textbf{the remaining mini-batch variance is what can move the iterate along that direction}. When the projected conditional second moment changes from order \(B^{-1}\) to order \(B^{-2}\), the time needed to move a fixed projected distance should increase substantially. The next corollary makes this projected escape-time calculation explicit.

\begin{mainboxedcorollary}{Projected escape-time scale}{cor:main-escape-time-scale}
Fix a unit direction \(\bm v\), and let \(\tau_{\bm v,r}=\inf\{t\ge0:\ |\langle\theta_t-\theta_0,\bm v\rangle|\ge r\}\) be the first time the stopped SGD trajectory moves distance \(r\) along \(\bm v\). In the bad-region application, \(\bm v\) is the optimistic steepest severe-error correction direction at the reference point. Under the weak-drift and bounded-second-moment conditions stated precisely in Appendix~\ref{app:localized-escape-time}, if the projected increment has conditional second-moment scale \(\nu_B^2\), then \(\E[\tau_{\bm v,r}]=\Omega(r^2/\nu_B^2)\). With fixed step size \(\eta\), Proposition~\ref{prop:main-hoeffding-collapse} makes the projected gradient variance contribute at scale \(\eta^2/B\) when the first-order term is present and at scale \(\eta^2/B^2\) after first-order projection collapse. Thus the optimistic projected escape-time lower-bound scale changes from order \(B\) to order \(B^2\).
\end{mainboxedcorollary}

Appendix~\ref{app:localized-escape-time} proves the stopped-process escape-time bound. Figure~\ref{fig:surrogate-weight-geometry}(a) illustrates it with a one-dimensional potential \(U(z)\): the left curve marks the start well, barrier, and target well, while the right heatmaps show where simulated iterates concentrate on the same state axis. With first-order \(B^{-1}\) noise, \(76\%\) of runs cross the barrier by the plotted horizon; with collapsed \(B^{-2}\) noise, only \(3\%\) do. This echoes the main claim: a bad region is not merely high-error; it is a region where severe errors remain but the current update loses projected motion along a severe-error correction direction.

Above all, we can now identify a bad region in the following.

A \textbf{bad region} is not merely a high-error region and need not be a local minimum. It is a
high-severe-error region where SGD loses useful motion along the severe-error
correction direction \(\bm v_{\rm sev}\): \textbf{the projected mean drift is weak, and the leading
projected noise term is small or, in the exact case, vanishes}. Because these ingredients are defined independently of escape time, the result above quantifies their local consequence: under the stated conditions, variance contraction slows motion over a fixed distance along \(\bm v_{\rm sev}\), even though the full trajectory may still progress through other directions.

We next ask how the update can be changed to enter or leave a bad region.

\section{Surrogate Losses and Their Weight Geometry}
\label{sec:surrogate-weights}

\textbf{The first way to change the update is to change the surrogate loss.} The surrogate loss choice changes the weight function in the pairwise gradient. For a differentiable surrogate, the key object is the scalar pairwise weight \(w_\ell(x)=-\ell'(x)\).
More severe ranking mistakes have more negative \(x\). For a margin \(m>0\), active pairs satisfy \(x<m\), and easy pairs satisfy \(x\ge m\). A corrective weight profile emphasizes severe active mistakes while stopping already separated pairs from consuming update mass.

\begin{figure}[t]
\centering
\captionsetup[subfigure]{font=footnotesize,labelfont=bf,justification=centering,singlelinecheck=false,skip=1pt}
\begin{subfigure}[t]{0.50\linewidth}
\centering
\includegraphics[width=\linewidth,clip]{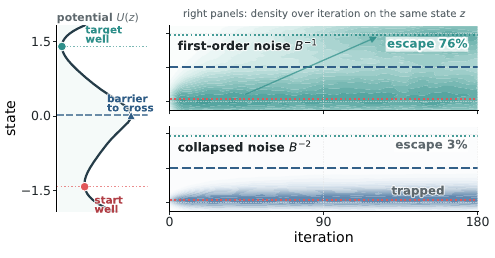}
\caption{Escape-time simulation}
\end{subfigure}\hfill
\begin{subfigure}[t]{0.356\linewidth}
\centering
\includegraphics[width=\linewidth,clip]{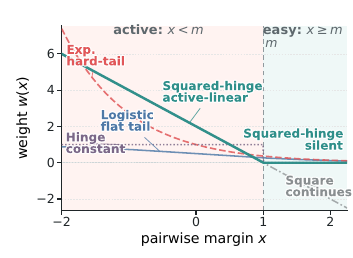}
\caption{Canonical surrogate weights}
\end{subfigure}
\caption{\textbf{Noise scale and surrogate weights illustrate two ingredients of corrective motion.}
\textbf{(a)} A one-dimensional projected escape simulation comparing a first-order \(B^{-1}\) variance regime with a collapsed \(B^{-2}\) variance regime. The left potential and the right heatmaps use the same state axis \(z\). \textbf{(b)} Update weights for five standard pairwise surrogates as functions of the pairwise margin \(x\). Simulation details are in Appendix~\ref{app:fig2-noise-escape}.
\label{fig:surrogate-weight-geometry}}
\end{figure}

Figure~\ref{fig:surrogate-weight-geometry}(b) compares popular surrogates in this sense. The relevant quantity is the allocation of update mass across pairs. The comparison is therefore organized by update weight, because the weight determines which pairs enter the mini-batch gradient and how strongly they enter. The five canonical losses lead to visibly different structures. Hinge is constant on active pairs, logistic flattens in the severe tail, exponential over-amplifies the hard tail, and square keeps nonzero easy-side weight. Within this five-loss comparison, squared-hinge realizes the desired combination of active-side severity sensitivity and easy-side silence: it grows linearly on active pairs while becoming exactly silent for pairs that already satisfy the margin. For the active-batch comparison below, a \emph{fully active} mini-batch means that every sampled positive-negative pair violates the margin. Appendix~\ref{app:surrogate-weight-calculations} gives the exact weights, the within-family characterization, and the active-batch decomposition used below.

\begin{mainboxedtheorem}{Informal surrogate-weight geometry}{thm:main-surrogate-weight-geometry}
Among logistic, exponential, square, hinge, and squared-hinge, squared-hinge is the only one in this canonical set that both ignores easy pairs and upweights active violations linearly by severity. On a fully active mini-batch, hinge contains only the class-mean score-gradient direction. Squared-hinge contains this direction plus a severity-dependent within-class correction. Thus squared-hinge, and more generally any surrogate with the same severity-increasing, easy-pair-suppressing weight structure, can retain corrective motion outside the hinge direction. 
\end{mainboxedtheorem}

The theorem gives a loss-level reason for the empirical gap between hinge and squared-hinge. Hinge ignores easy pairs, but on a fully active mini-batch it treats all active violations equally, so the update is driven only by the two class-mean score-gradient averages. If that class-mean direction is weak or poorly aligned with the direction that would correct the severe misranking pairs, hinge has no additional severity-dependent term to use. Squared-hinge also ignores easy pairs, but its active-side weight grows with violation size, letting the mini-batch update react to the distribution of active margins inside the positive and negative classes. Appendix~\ref{app:surrogate-weight-calculations} gives the calculation showing where this within-class correction enters.

The other four canonical losses considered here miss either the active-side severity weighting or the easy-pair silence that make squared-hinge useful in this comparison. Logistic has bounded hard-tail weights, so increasingly severe violations eventually receive almost the same weight. Exponential has the opposite problem: it can concentrate the update on the extreme tail. Square has the same affine active-side response as squared-hinge, but it continues to spend update mass on already separated pairs. Appendix~\ref{app:surrogate-weight-calculations} gives the hard-tail, easy-side, active-batch, and projected-variance calculations behind these comparisons. In a fully active setting, hinge can have zero projected mini-batch variance, while squared-hinge keeps a nonzero order-\(B^{-1}\) term under class-score variation. Thus the surrogate can change both deterministic correction and projected stochastic motion along the steepest correction direction.

\section{Changing Where Pairwise AUC Optimization Begins}
\label{sec:warmup-pretraining}

Section~\ref{sec:surrogate-weights} treated surrogate choice as a way to change the pairwise update. \textbf{A second intervention is to change where the AUC phase begins.} This is the role of CE warm-up and pretrained parameter-efficient fine-tuning. CE warm-up can move an ordinary model away from many severe misranking pairs before pairwise training starts. A pretrained representation with a low-dimensional trainable module can instead start in a region where severe misrankings are already controlled and first-order stochastic motion remains present. This starting-region effect is the second lesson of Figure~\ref{fig:intro-controlled-dynamics}: warm-up helps in the bad-start panel because it changes the starting region of the pairwise phase, while the good-start panel shows that the same pairwise objective can succeed without a separate warm-up when training already begins on the favorable side of the bad-region mechanism.

Define the balanced cross-entropy risk, obtained by averaging the positive and negative classes separately, as \(R_{\mathrm{CE}}(\theta)=
\E[\log(1+e^{-f_\theta(X^+)})]
+\E[\log(1+e^{f_\theta(X^-)})]\). For \(M>0\), let \(\pi_M(\theta)=\Pbb(m_\theta(X^+,X^-)\le -M)\) be the \textbf{misranking mass}. The following result links the pointwise warm-up and the severe-misranking mass.

\begin{mainboxedproposition}{CE risk controls severe misrankings}{prop:main-ce-risk-severe-mass}
For every \(M>0\), \(\pi_M(\theta)\le R_{\mathrm{CE}}(\theta)/\log(1+e^{M/2})\). Consequently, if a CE warm-up endpoint \(\theta_{\rm wu}\) satisfies \(R_{\mathrm{CE}}(\theta_{\rm wu})\le\delta\cdot\log(1+e^{M/2})\), then \(\pi_M(\theta_{\rm wu})\le\delta\).
\end{mainboxedproposition}

The proof appears in Appendix~\ref{app:ce-warmup-start}: if a positive is scored far below a negative by at least \(M\), then either the positive score is low or the negative score is high, and CE penalizes these events. Thus CE warm-up can improve the later pairwise AUC-surrogate phase even though CE itself does not optimize a pairwise AUC surrogate. The role of CE warm-up is to reduce the mass of severe misranking pairs before the pairwise phase starts.

This gives a precise condition under which warm-up improves the pairwise starting state. The bound controls the severe-misranking mass at the AUC starting point. When the subsequent AUC trajectory remains near the warm-up endpoint, local continuity of the margins transfers this low severe-misranking mass to the early pairwise phase; Appendix~\ref{app:ce-warmup-start} gives the finite-horizon statement. A poor warm-up endpoint, or a pairwise trajectory that quickly leaves that neighborhood, falls outside this local argument. We therefore treat warm-up as a starting-region intervention.

Pretrained parameter-efficient fine-tuning addresses the same question about starting regions in a different way. Let \(\bm\xi\) denote the trainable coordinates of a head, adapter, or LoRA module, with the pretrained base absorbed into \(f_{\bm\xi}\). The pairwise gradient identity and Hoeffding split remain valid after replacing \(\theta\) by \(\bm\xi\), because they use only the score map and its gradient with respect to the trainable variable. The relevant question is therefore whether the pretrained starting neighborhood is already favorable for AUC optimization. In this subspace, write \(\pi_M(\bm\xi)=\Pbb(m_{\bm\xi}(X^+,X^-)\le -M)\), where \(m_{\bm\xi}(x^+,x^-)=f_{\bm\xi}(x^+)-f_{\bm\xi}(x^-)\).

\begin{mainboxedproposition}{Finite-horizon transfer from a favorable pretrained subspace}{prop:main-peft-favorable-subspace}
Let \(\psi_{\bm\xi,\bm v}\) be the scalar projected pairwise-gradient kernel from Section~\ref{sec:bad-regions}, with \(\theta\) replaced by the trainable coordinate \(\bm\xi\). Suppose a trainable-subspace neighborhood \(\mathcal U\) satisfies \(\sup_{\bm\xi\in\mathcal U}\pi_M(\bm\xi)\le\delta_M\). Suppose also that, for some unit direction \(\bm v\), \(\Var(\psi_{\bm\xi,\bm v}^{(1,0)}(X^+))+\Var(\psi_{\bm\xi,\bm v}^{(0,1)}(X^-))\ge\sigma_{1,\bm v}^2>0\) for \(\bm\xi\in\mathcal U\). If the direct pairwise AUC-surrogate trajectory remains in \(\mathcal U\) for the finite horizon of interest, then every iterate on that horizon satisfies \(\pi_M(\bm\xi_t)\le\delta_M\) and \(\Var(\langle \widehat{\bm g}_B^\ell(\bm\xi_t),\bm v\rangle)\ge\sigma_{1,\bm v}^2/B\).
\end{mainboxedproposition}

The proposition is a local finite-horizon transfer statement conditional on a verified favorable trainable-subspace neighborhood and trajectory containment; Appendix~\ref{app:starting-region} gives sufficient score-perturbation, confinement, and local descent checks for this verification. Along such a trajectory, severe misrankings stay controlled and the first-order projected stochastic component remains present, making a separate warm-up stage less necessary.

Under the same LoRA-style fine-tuning protocol, a pretrained base with more favorable starting geometry can place the trainable trajectory closer to a favorable ranking region, reducing the marginal value of a separate CE warm-up; a weaker starting geometry leaves more room for warm-up to help. The experiments test this starting-state explanation across CV and NLP tasks.

\section{Experiments}
\label{sec:experiments}

The experiments focus on a specific question: when the evaluation metric is held fixed as AUC, which training choices preserve Corrective Power in finite-batch deep optimization? We compare CE, direct pairwise AUC-surrogate training, and CE-warm-started pairwise AUC-surrogate training on binary CV and NLP ranking tasks using ResNet50~\citep{he2016deep}, DenseNet121~\citep{huang2017densely}, DistilBERT-base~\citep{sanh2019distilbert}, and Qwen2.5-1.5B-Instruct~\citep{qwen2025qwen25technicalreport}. Table~\ref{tab:main-surrogate-family-results} reports the compact DenseNet121 and Qwen direct and warm-up rows with sample standard deviations, while Table~\ref{tab:app-full-resnet50-distilbert} gives the corresponding ResNet50 and DistilBERT-base blocks. Figure~\ref{fig:experiment-trajectory-diagnostics} gives trajectory-level mechanism measurements, and Figure~\ref{fig:experiment-aggregate-diagnostics} gives aggregate warm-up and weight-allocation summaries. Appendix~\ref{app:experimental-details} gives the data construction, trainable-module details, optimization protocol, and full reporting rules. We use the tables and trajectory measurements to examine four questions about when pairwise AUC training works reliably.

\definecolor{densehead}{HTML}{F7FBF8}
\definecolor{densebest}{HTML}{EEF8F2}
\definecolor{densesecond}{HTML}{F8FCFA}
\definecolor{qwenhead}{HTML}{FFF6F6}
\definecolor{qwenbest}{HTML}{FFEDED}
\definecolor{qwensecond}{HTML}{FFF8F8}

\newcommand{\mainmodelrow}[2]{\rowcolor{#1}\multicolumn{8}{c}{{\bfseries #2}}\\}
\newcommand{\mainaucstd}[2]{#1$_{\scriptstyle \pm #2}$}
\newcommand{\mainbestcell}[3]{\cellcolor{#1}{\bfseries #2}$_{\scriptstyle \pm #3}$}
\newcommand{\mainsecondcell}[3]{\cellcolor{#1}\underline{#2}$_{\scriptstyle \pm #3}$}

\begin{table}[!tbp]
\centering
\caption{\textbf{Held-out AUC (\%) for direct and CE-warm-started pairwise training.}
Representative DenseNet121 and Qwen2.5-1.5B-Instruct results are shown; ResNet50 and DistilBERT-base are reported in Table~\ref{tab:app-full-resnet50-distilbert}. Direct starts pairwise training immediately; Warm-up first runs CE and then switches to the indicated pairwise loss. Values are mean AUC (\%) with sample standard deviations as subscripts; bold and underline mark the best and second-best results within each task.}
\label{tab:main-surrogate-family-results}
\scriptsize
\setlength{\tabcolsep}{1.15pt}
\renewcommand{\arraystretch}{1.04}
\begin{tabularx}{\textwidth}{@{}>{\raggedright\arraybackslash}p{0.180\textwidth}>{\centering\arraybackslash}p{0.080\textwidth}*{6}{>{\centering\arraybackslash}X}@{}}
\toprule
Task & Mode & CE & Square & Hinge & Exp. & Logistic & Sq.-hinge \\
\midrule
\mainmodelrow{densehead}{DenseNet121}
\midrule
Apple Disease & Direct & \mainaucstd{60.64}{2.63} & \mainaucstd{61.10}{1.85} & \mainaucstd{58.81}{2.46} & \mainaucstd{60.18}{1.21} & \mainaucstd{60.46}{3.09} & \mainbestcell{densebest}{62.69}{4.86} \\
 & Warm-up & -- & \mainaucstd{61.10}{1.85} & \mainaucstd{59.91}{2.60} & \mainaucstd{61.03}{1.92} & \mainaucstd{60.54}{2.02} & \mainsecondcell{densesecond}{62.69}{4.86} \\
\addlinespace[1.2pt]
Rice Disease & Direct & \mainaucstd{72.45}{6.57} & \mainaucstd{71.43}{7.23} & \mainaucstd{70.92}{7.66} & \mainaucstd{70.61}{6.81} & \mainaucstd{71.63}{7.30} & \mainaucstd{71.53}{7.30} \\
 & Warm-up & -- & \mainsecondcell{densesecond}{74.90}{6.31} & \mainaucstd{73.98}{6.26} & \mainaucstd{74.90}{6.41} & \mainaucstd{71.63}{7.30} & \mainbestcell{densebest}{75.71}{8.61} \\
\addlinespace[1.2pt]
APTOS Sev./Mod. & Direct & \mainaucstd{82.75}{0.95} & \mainaucstd{83.83}{1.41} & \mainbestcell{densebest}{85.10}{0.76} & \mainaucstd{83.68}{1.66} & \mainaucstd{84.18}{1.37} & \mainsecondcell{densesecond}{84.20}{1.12} \\
 & Warm-up & -- & \mainaucstd{83.96}{1.04} & \mainaucstd{83.86}{0.99} & \mainaucstd{83.93}{1.25} & \mainaucstd{83.76}{1.47} & \mainaucstd{83.08}{1.49} \\
\addlinespace[1.2pt]
HAM10000 M/N & Direct & \mainaucstd{96.50}{0.58} & \mainaucstd{96.88}{0.57} & \mainaucstd{94.13}{0.81} & \mainaucstd{95.97}{0.50} & \mainaucstd{95.35}{0.56} & \mainsecondcell{densesecond}{97.32}{0.42} \\
 & Warm-up & -- & \mainaucstd{96.81}{1.02} & \mainaucstd{96.08}{0.64} & \mainaucstd{96.29}{0.46} & \mainaucstd{96.51}{0.74} & \mainbestcell{densebest}{97.70}{0.21} \\
\addlinespace[2pt]\midrule\addlinespace[1pt]
\mainmodelrow{qwenhead}{Qwen2.5-1.5B-Instruct}
\midrule
GoEmotions Annoy. & Direct & \mainaucstd{82.97}{1.27} & \mainaucstd{83.74}{1.45} & \mainaucstd{83.44}{1.13} & \mainaucstd{81.16}{1.63} & \mainaucstd{80.71}{1.87} & \mainsecondcell{qwensecond}{84.89}{0.65} \\
 & Warm-up & -- & \mainaucstd{84.34}{0.69} & \mainaucstd{83.27}{1.03} & \mainaucstd{83.36}{0.62} & \mainaucstd{83.07}{0.72} & \mainbestcell{qwenbest}{84.92}{0.66} \\
\addlinespace[1.2pt]
GoEmotions Fear & Direct & \mainaucstd{95.19}{1.22} & \mainaucstd{94.93}{1.02} & \mainaucstd{93.42}{0.77} & \mainaucstd{93.75}{1.28} & \mainaucstd{92.68}{0.76} & \mainaucstd{96.33}{0.91} \\
 & Warm-up & -- & \mainsecondcell{qwensecond}{96.69}{0.54} & \mainaucstd{94.88}{0.99} & \mainaucstd{94.28}{1.26} & \mainaucstd{94.51}{0.60} & \mainbestcell{qwenbest}{96.98}{0.62} \\
\addlinespace[1.2pt]
ANLI R1 E/N & Direct & \mainaucstd{86.01}{0.74} & \mainaucstd{84.87}{1.87} & \mainaucstd{70.88}{12.94} & \mainaucstd{79.01}{1.53} & \mainaucstd{77.82}{2.72} & \mainbestcell{qwenbest}{86.85}{0.73} \\
 & Warm-up & -- & \mainaucstd{85.41}{1.40} & \mainaucstd{83.23}{1.70} & \mainaucstd{83.44}{1.42} & \mainaucstd{83.09}{1.72} & \mainsecondcell{qwensecond}{86.45}{0.84} \\
\addlinespace[1.2pt]
ANLI R3 C/N & Direct & \mainbestcell{qwenbest}{72.39}{1.49} & \mainaucstd{72.04}{0.97} & \mainaucstd{52.02}{2.09} & \mainaucstd{69.12}{2.11} & \mainaucstd{69.33}{1.22} & \mainaucstd{71.99}{0.94} \\
 & Warm-up & -- & \mainaucstd{71.71}{0.70} & \mainaucstd{70.37}{1.65} & \mainaucstd{70.80}{1.21} & \mainaucstd{70.47}{0.97} & \mainsecondcell{qwensecond}{72.15}{0.64} \\
\addlinespace[1.2pt]
ANLI R3 E/N & Direct & \mainsecondcell{qwensecond}{79.98}{0.94} & \mainaucstd{78.27}{0.87} & \mainaucstd{63.97}{12.49} & \mainaucstd{74.46}{1.69} & \mainaucstd{75.03}{1.68} & \mainaucstd{79.72}{0.44} \\
 & Warm-up & -- & \mainaucstd{79.40}{0.62} & \mainaucstd{77.03}{1.45} & \mainaucstd{77.86}{0.88} & \mainaucstd{78.35}{0.79} & \mainbestcell{qwenbest}{80.60}{1.27} \\
\addlinespace[1.2pt]
SNLI E/N & Direct & \mainaucstd{96.69}{0.27} & \mainsecondcell{qwensecond}{96.79}{0.44} & \mainaucstd{80.58}{14.39} & \mainaucstd{90.06}{4.27} & \mainaucstd{91.21}{2.01} & \mainaucstd{96.64}{0.33} \\
 & Warm-up & -- & \mainbestcell{qwenbest}{96.93}{0.31} & \mainaucstd{94.73}{2.30} & \mainaucstd{94.11}{1.75} & \mainaucstd{94.60}{0.71} & \mainaucstd{96.65}{0.36} \\
\addlinespace[1.2pt]
WANLI C/N & Direct & \mainsecondcell{qwensecond}{89.81}{0.46} & \mainaucstd{89.35}{0.56} & \mainaucstd{67.57}{14.73} & \mainaucstd{79.72}{15.49} & \mainaucstd{85.46}{2.52} & \mainaucstd{89.77}{0.60} \\
 & Warm-up & -- & \mainaucstd{89.47}{0.39} & \mainaucstd{85.87}{3.05} & \mainaucstd{87.48}{1.11} & \mainaucstd{88.12}{0.75} & \mainbestcell{qwenbest}{89.83}{0.43} \\
\addlinespace[1.2pt]
WANLI E/N & Direct & \mainaucstd{86.50}{0.25} & \mainbestcell{qwenbest}{86.61}{0.17} & \mainaucstd{62.20}{5.34} & \mainaucstd{75.41}{14.83} & \mainaucstd{71.09}{16.10} & \mainaucstd{86.40}{0.75} \\
 & Warm-up & -- & \mainsecondcell{qwensecond}{86.50}{0.20} & \mainaucstd{82.66}{0.85} & \mainaucstd{83.88}{1.25} & \mainaucstd{83.90}{1.38} & \mainaucstd{86.29}{0.41} \\
\bottomrule
\end{tabularx}
\end{table}

\noindent\textbf{RQ1. Is metric-aligned pairwise training enough by itself?}\\
\textbf{No.} Directly optimizing a pairwise AUC surrogate can improve held-out AUC, but it is not automatically reliable. Table~\ref{tab:main-surrogate-family-results} should be read task by task: the CE column is the pointwise baseline, the direct row starts pairwise training immediately, and the warm-up row first runs CE before switching to the specified pairwise surrogate. Bold and underlined entries mark the best and second-best validation-selected test AUC values within each task. This comparison shows that pairwise losses can beat CE, but it also shows failures in which a metric-aligned pairwise objective falls below CE or becomes much worse than other pairwise surrogates. The trajectory evidence gives the same message from inside training: Figure~\ref{fig:experiment-trajectory-diagnostics}(c) plots the logged-batch misranking ratio, namely the fraction of logged positive-negative training pairs with \(f_\theta(x^+)-f_\theta(x^-)<0\). Direct hinge training can keep this ratio high for many pairwise updates, meaning that the loss is aligned with AUC but the trajectory is not quickly correcting the wrong pairs. The matched intervention in Appendix~\ref{app:matched-directional-intervention} provides a direct mechanism check: perturbations confined to \(\widehat{\bm v}_{\rm sev}\) produce markedly more finite-horizon changes than equal-energy isotropic controls. This identifies \(\widehat{\bm v}_{\rm sev}\) as a mechanism-relevant direction while showing that the realized effect depends on the perturbation sequence and subsequent trajectory. The experiments therefore support the paper's central claim that objective--metric alignment must be paired with Corrective Power.

\noindent\textbf{RQ2. Which surrogate most consistently preserves useful correction?}\\
\textbf{Squared-hinge gives the highest AUC values in the representative summary.} Among the displayed DenseNet121 and Qwen rows in Table~\ref{tab:main-surrogate-family-results}, the best result across CE and all direct/warm-started pairwise settings is usually obtained by squared-hinge, while the second-best result varies across CE and square. The weight plots in Figure~\ref{fig:experiment-aggregate-diagnostics}(c,d) explain how to read this beyond the final AUC numbers. For each logged positive-negative pair, we compute the current margin \(m_\theta=f_\theta(x^+)-f_\theta(x^-)\); wrong pairs have \(m_\theta<0\), and more negative intervals correspond to more severe misranking. For an interval \(I\), \(P_I^{\rm rel}\) is the proportion of total gradient weight assigned to pairs in \(I\), divided by the frequency of pairs in \(I\). Thus \(P_I^{\rm rel}>1\) means that the interval receives more update weight than its frequency alone would imply, while \(P_I^{\rm rel}<1\) means that it is relatively down-weighted. The useful visual signal is therefore not just whether a loss has many wrong pairs, but whether the leftmost wrong-margin intervals receive extra corrective weight. Squared-hinge assigns larger relative weight to more severe wrong-margin intervals, while hinge treats active violations uniformly; this is the empirical weight-profile counterpart of the surrogate analysis. Appendix~\ref{app:mechanism-diagnostics} gives the finite-batch formula, row/column projection definitions, and logging protocol, and Table~\ref{tab:app-full-resnet50-distilbert} reports the corresponding ResNet50 and DistilBERT-base splits.

\begin{figure}[!tbp]
\centering
\includegraphics[width=0.63\linewidth,clip]{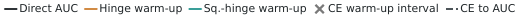}\vspace{0.2em}

\captionsetup[subfigure]{font=footnotesize,labelfont=bf,justification=centering,singlelinecheck=false,skip=1pt}
\begin{subfigure}[t]{0.238\linewidth}
\centering
\includegraphics[width=\linewidth,clip]{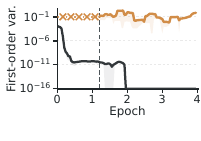}
\caption{Hinge variance term}
\end{subfigure}\hfill
\begin{subfigure}[t]{0.238\linewidth}
\centering
\includegraphics[width=\linewidth,clip]{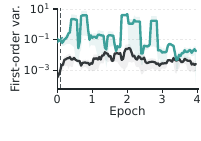}
\caption{Sq.-hinge variance term}
\end{subfigure}\hfill
\begin{subfigure}[t]{0.238\linewidth}
\centering
\includegraphics[width=\linewidth,clip]{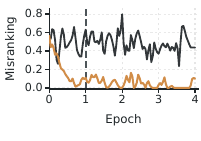}
\caption{Hinge trajectory}
\end{subfigure}\hfill
\begin{subfigure}[t]{0.238\linewidth}
\centering
\includegraphics[width=\linewidth,clip]{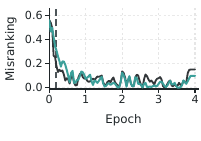}
\caption{Sq.-hinge trajectory}
\end{subfigure}
\caption{\textbf{Real-run trajectory measures support the Corrective Power explanation.}
Four training-trajectory quantities computed from experiment runs. Panels (a,b) show the scaled first-order Hoeffding variance term on SNLI E/N along the steepest severe-error correction direction; shaded bands show one standard error after the same smoothing used for the displayed mean. \(x\) marks denote the CE warm-up interval before the pairwise variance term is recorded. Panels (c,d) show ANLI R1 E/N trajectories and mark the CE-to-AUC switch. Definitions are in Appendix~\ref{app:mechanism-diagnostics}.}
\label{fig:experiment-trajectory-diagnostics}
\end{figure}

\begin{figure}[!tbp]
\centering
\includegraphics[width=0.44\linewidth,clip]{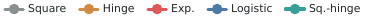}\vspace{0.2em}

\captionsetup[subfigure]{font=footnotesize,labelfont=bf,justification=centering,singlelinecheck=false,skip=1pt}
\begin{subfigure}[t]{0.238\linewidth}
\centering
\includegraphics[width=\linewidth,clip]{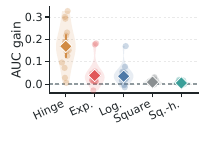}
\caption{Warm-up gain by loss}
\end{subfigure}\hfill
\begin{subfigure}[t]{0.238\linewidth}
\centering
\includegraphics[width=\linewidth,clip]{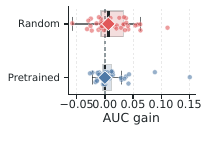}
\caption{Warm-up gain by initialization}
\end{subfigure}\hfill
\begin{subfigure}[t]{0.238\linewidth}
\centering
\includegraphics[width=\linewidth,clip]{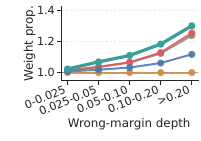}
\caption{Direct weight distribution}
\end{subfigure}\hfill
\begin{subfigure}[t]{0.238\linewidth}
\centering
\includegraphics[width=\linewidth,clip]{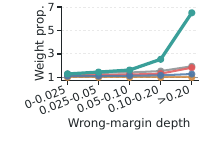}
\caption{Warm-up weight distribution}
\end{subfigure}
\caption{\textbf{Aggregate real-run measures support the Corrective Power explanation.}
Four aggregate quantities computed from the same experiment runs. Panels (a,b) report warm-up gains by loss and initialization, respectively. Panels (c,d) report relative weight proportions over wrong-margin intervals. Definitions are in Appendix~\ref{app:mechanism-diagnostics}.}
\label{fig:experiment-aggregate-diagnostics}
\end{figure}

\noindent\textbf{RQ3. When is CE warm-up helpful?}\\
\textbf{Warm-up helps most when it changes a poor pairwise starting state.} The direct/warm-up rows in Table~\ref{tab:main-surrogate-family-results} show that warm-up can give large gains for fragile pairwise settings, especially hinge on harder NLP tasks, but the gains are not uniform across losses or models. Figure~\ref{fig:experiment-aggregate-diagnostics}(a) summarizes this with warm-up gain, defined as validation-selected test AUC after CE-warm-started pairwise training minus validation-selected test AUC after direct pairwise training. Positive values mean that CE improves the starting point for the later pairwise phase; values near zero mean that the direct pairwise run already starts from, or quickly reaches, a usable region. The largest and most variable gains appear for hinge, matching the view that warm-up is most valuable when direct pairwise optimization enters a poor high-misranking state. Figure~\ref{fig:experiment-trajectory-diagnostics}(a,b) then checks the stochastic side of the same explanation: it records the scaled first-order Hoeffding variance term along the steepest severe-error correction direction. This term measures whether the mini-batch update still has first-order stochastic motion in the favorable projected direction; when it is small, the leading mini-batch component along that direction is weak. Figure~\ref{fig:experiment-trajectory-diagnostics}(c,d) records the logged-batch misranking ratio, which measures how many logged training pairs are still ordered incorrectly. In these trajectory panels, the black curve is direct pairwise AUC-surrogate training, the colored curve is CE warm-up followed by pairwise training, and the vertical dashed line marks the CE-to-AUC switch. The hinge trajectory shows the intended warm-up effect: CE lowers the misranking state before the pairwise phase begins, after which pairwise training starts from an easier region. The squared-hinge trajectory shows a different regime: direct squared-hinge can already reduce misranking quickly and preserve stronger corrective motion, so a separate warm-up has smaller marginal value. Appendix~\ref{app:mechanism-diagnostics} gives the finite-batch construction and independent logging-batch protocol, and Figure~\ref{fig:app-all-loss-trajectories} in Appendix~\ref{app:training-details} reports the all-loss trajectory counterpart.

Fixed-checkpoint batch-size scaling provides a complementary check of this stochastic explanation. Across APTOS, SNLI, and ANLI-R1, weakly corrective regions show more attenuated severe-component variance scaling than their matched corrective regions, matching the predicted shift toward the residual pair-interaction component as first-order projected motion weakens. Appendix~\ref{app:fixed-checkpoint-variance-scaling} gives the protocol and complete results.

The two trajectory quantities should be interpreted together. A low misranking ratio alone would show that the run is improving the ordering, but it would not explain whether the pairwise mini-batch update still has the stochastic component emphasized by the theory. The warm-up panels therefore support the mechanism when both signals move in the expected direction: CE changes the state before the AUC phase, the later pairwise updates operate from a lower-misranking region, and the projected first-order term remains more visible along the steepest correction direction.

\noindent\textbf{RQ4. What changes under pretrained initialization and parameter-efficient tuning?}\\
\textbf{A pretrained start can reduce the need for a separate warm-up stage when it places the trainable trajectory in a favorable region.} The Qwen block in Table~\ref{tab:main-surrogate-family-results} often has small squared-hinge warm-up gains, and the direct squared-hinge entries are already close to the corresponding warm-up entries on several NLP tasks. Figure~\ref{fig:experiment-aggregate-diagnostics}(b) provides a matched same-backbone comparison between random and ImageNet-pretrained initialization. The aggregate shift and reduced dispersion of warm-up gains under pretrained initialization support that pretraining places optimization closer to a favorable ranking state and reduces the average marginal benefit of an additional warm-up stage. This is consistent with the starting-state explanation: when the pretrained representation and trainable subspace place AUC training closer to a favorable ranking geometry, warm-up has smaller marginal value, although it remains useful in some settings.

Together, the experiments support a conditional prescription: pairwise AUC-surrogate training is most dependable when the surrogate keeps corrective gradient weight on severe wrong pairs, and warm-up or pretraining is most useful when it moves the trajectory into a state where such corrective updates remain available.

\section{Conclusion}
\label{sec:conclusion}

This paper studies why direct pairwise AUC-surrogate training can still be unreliable
despite matching the evaluation metric. The key issue is \emph{Corrective Power}:
whether SGD keeps useful motion on severe ranking errors. We identify bad regions
where severe misranked pairs remain common, but the projected mean drift along the
severe-error correction direction is weak. In these regions, progress depends on the
finite-batch noise in the same direction. The pairwise mini-batch gradient is a
two-sample U-statistic, so this projected noise has a leading order-\(B^{-1}\) term
and a residual order-\(B^{-2}\) interaction term in the equal-size case. When the leading
term vanishes, the useful noise drops to order \(B^{-2}\), and the projected
escape-time scale changes from order \(B\) to order \(B^2\). This view explains two
ways to improve corrective power. First, the surrogate loss controls the gradient
weight assigned to severe wrong pairs; squared-hinge is effective because it gives
larger weight to severe active violations while ignoring easy pairs. Second, warm-up
and PEFT change where pairwise training begins; CE warm-up can move a model away
from a poor start, while pretrained PEFT can start closer to a favorable ranking
region. The CV and NLP experiments are consistent with these two predictions.


\begin{ack}
This work was supported in part by National Natural Science Foundation of China: 62525212, U23B2051, 62236008, 62441232, 62521007, U21B2038, 62502496, and 62576332, in part by Youth Innovation Promotion Association CAS, in part by the Strategic Priority Research Program of the Chinese Academy of Sciences, Grant No. XDB0680201, in part by Beijing Natural Science Foundation–Xiaomi Innovation Joint Fund (Grant No. 26L030072), in part by the project ZR2025ZD01 supported by Shandong Provincial Natural Science Foundation, and in part by the Young Elite Scientists Sponsorship Program of the Beijing High Innovation Plan No. 20250687.
The authors declare no competing interests.
\end{ack}

\FloatBarrier
\clearpage
\bibliographystyle{unsrtnat-preservecase}
\bibliography{references}

\clearpage
\def\INMAINPAPER{1}
\section*{Appendix}
\appendix

\section*{Table of Contents}
\begingroup
\noindent\hyperref[app:related-work]{A\quad Related Work}\dotfill \pageref{app:related-work}\par
\noindent\hyperref[app:experimental-details]{B\quad Experimental Details}\dotfill \pageref{app:experimental-details}\par
\noindent\hspace*{1.6em}\hyperref[app:fig1-simulation]{B.1\quad Details of the Controlled Two-Dimensional Simulation}\dotfill \pageref{app:fig1-simulation}\par
\noindent\hspace*{1.6em}\hyperref[app:fig2-noise-escape]{B.2\quad Details of the Noise-Scale Escape Simulation}\dotfill \pageref{app:fig2-noise-escape}\par
\noindent\hspace*{1.6em}\hyperref[app:real-experiment-protocol]{B.3\quad Main Experimental Protocol}\dotfill \pageref{app:real-experiment-protocol}\par
\noindent\hspace*{1.6em}\hyperref[app:datasets-tasks]{B.4\quad Data Sources and Task Construction}\dotfill \pageref{app:datasets-tasks}\par
\noindent\hspace*{1.6em}\hyperref[app:training-details]{B.5\quad Models, Optimization, and Reporting Protocol}\dotfill \pageref{app:training-details}\par
\noindent\hspace*{1.6em}\hyperref[app:extended-auc-experiments]{B.6\quad Multiclass, Multilabel, and Deep-AUC Extensions}\dotfill \pageref{app:extended-auc-experiments}\par
\noindent\hspace*{1.6em}\hyperref[app:assets-licenses]{B.7\quad Assets, Licenses, and Terms}\dotfill \pageref{app:assets-licenses}\par
\noindent\hspace*{1.6em}\hyperref[app:ethics-impacts]{B.8\quad Ethics, Broader Impacts, and Safeguards}\dotfill \pageref{app:ethics-impacts}\par
\noindent\hyperref[app:pairwise-sgd-dynamics]{C\quad Direct Pairwise SGD as a Stochastic Dynamics}\dotfill \pageref{app:pairwise-sgd-dynamics}\par
\noindent\hspace*{1.6em}\hyperref[app:extensions-beyond-binary-auc]{C.1\quad Extensions Beyond Binary AUC}\dotfill \pageref{app:extensions-beyond-binary-auc}\par
\noindent\hyperref[app:pairwise-bad-region-dynamics]{D\quad Bad Regions and Hoeffding Projection Collapse}\dotfill \pageref{app:pairwise-bad-region-dynamics}\par
\noindent\hyperref[app:surrogate-weight-calculations]{E\quad Surrogate Weight Calculations}\dotfill \pageref{app:surrogate-weight-calculations}\par
\noindent\hspace*{1.6em}\hyperref[app:canonical-surrogate-weights]{E.1\quad Canonical Weights}\dotfill \pageref{app:canonical-surrogate-weights}\par
\noindent\hspace*{1.6em}\hyperref[app:active-side-surrogate-geometry]{E.2\quad Active-Side Correction Versus Class-Mean Collapse}\dotfill \pageref{app:active-side-surrogate-geometry}\par
\noindent\hspace*{1.6em}\hyperref[app:easy-side-surrogate-geometry]{E.3\quad Easy-Side Silence Versus Overshoot}\dotfill \pageref{app:easy-side-surrogate-geometry}\par
\noindent\hspace*{1.6em}\hyperref[app:hard-tail-surrogate-geometry]{E.4\quad Hard-Tail Alignment}\dotfill \pageref{app:hard-tail-surrogate-geometry}\par
\noindent\hspace*{1.6em}\hyperref[app:projected-stochastic-surrogate-separation]{E.5\quad Projected Stochastic Separation}\dotfill \pageref{app:projected-stochastic-surrogate-separation}\par
\noindent\hyperref[app:starting-region]{F\quad Changing Where Pairwise AUC Optimization Begins}\dotfill \pageref{app:starting-region}\par
\noindent\hspace*{1.6em}\hyperref[app:ce-warmup-start]{F.1\quad CE Warm-Up and Severe-Misranking Mass}\dotfill \pageref{app:ce-warmup-start}\par
\noindent\hspace*{1.6em}\hyperref[app:peft-stable-subspace]{F.2\quad Parameter-Efficient Fine-Tuning in a Stable Trainable Subspace}\dotfill \pageref{app:peft-stable-subspace}\par
\noindent\hspace*{1.6em}\hyperref[app:peft-favorable-subspace]{F.3\quad Local Pairwise AUC-Surrogate Descent in a Favorable Subspace}\dotfill \pageref{app:peft-favorable-subspace}\par
\noindent\hyperref[app:transformer-degenerate-region]{G\quad Model-Class Realizations of Projected Bad Components}\dotfill \pageref{app:transformer-degenerate-region}\par
\noindent\hspace*{1.6em}\hyperref[app:transformer-one-layer]{G.1\quad Data-Dependent One-Layer Construction}\dotfill \pageref{app:transformer-one-layer}\par
\noindent\hspace*{1.6em}\hyperref[app:transformer-trainable-neighborhoods]{G.2\quad Projected Bad Components in Trainable Transformer Neighborhoods}\dotfill \pageref{app:transformer-trainable-neighborhoods}\par
\noindent\hyperref[app:localized-escape-time]{H\quad Localized Escape-Time Bounds for Stopped SGD}\dotfill \pageref{app:localized-escape-time}\par
\noindent\hyperref[app:limitations]{I\quad Limitations}\dotfill \pageref{app:limitations}\par
\endgroup
\vspace{0.5em}

\paragraph{How to read the appendix.}
Appendix~\ref{app:related-work} gives the extended related work. Appendix~\ref{app:experimental-details} records the controlled simulations, the full protocol behind the reported CV and NLP experiments, and the asset, license, and ethics information relevant to the experiments. The remaining appendices give the proofs: Appendix~\ref{app:pairwise-sgd-dynamics} fixes the direct pairwise SGD process and its local Langevin interface; Appendix~\ref{app:pairwise-bad-region-dynamics} proves the bad-region and Hoeffding-projection statements used in Section~\ref{sec:bad-regions}; Appendix~\ref{app:surrogate-weight-calculations} proves the surrogate-weight claims used in Section~\ref{sec:surrogate-weights}; Appendix~\ref{app:starting-region} proves the warm-up and parameter-efficient fine-tuning statements behind Section~\ref{sec:warmup-pretraining}; Appendix~\ref{app:transformer-degenerate-region} shows that the projected bad-component conditions can be realized inside Transformer-style model classes; Appendix~\ref{app:localized-escape-time} gives the stopped-SGD escape-time tool that converts directional drift and second-moment bounds into the projected \(B\) versus \(B^2\) scale contrast; and Appendix~\ref{app:limitations} states the scope and limitations of the theory.

\section{Related Work}
\label{app:related-work}

This appendix organizes the literature behind the paper's main question: when direct pairwise AUC-surrogate training uses finite mini-batch SGD, do severe ranking mistakes still generate useful parameter motion? We group the discussion by statistical target, scalable AUC algorithms, surrogate validity, U-statistic structure, related auxiliary-state methods, partial-AUC variants, warm-up and pretraining practice, and local stochastic dynamics.

\paragraph{AUC, ranking, and non-decomposable learning.}
AUC is a bipartite-ranking objective rather than a decomposable classification loss: it evaluates whether positives receive larger scores than negatives across opposite-class pairs. Classical and structural views of ROC area and multivariate performance measures formalize this non-decomposable nature~\citep{joachims2005support,narasimhan2015optimizing}. Broader surveys place AUC beside bipartite ranking, partial AUC, top-focused ranking, and precision-style criteria~\citep{werner2022review}. Related work also adapts the AUC target itself to multiclass learning, cost-sensitive decision making, open-set recognition, and pixel-level long-tail segmentation~\citep{yang2021learning,shao2023weighted,wang2022openauc,han2024aucseg}. Cross-media retrieval provides another ranking-oriented application context, including unsupervised cross-media hashing based on knowledge graphs~\citep{ye2022unsupervised}. Adjacent learning-based vision problems include under-sampled single-pixel reconstruction and joint denoising and demosaicing~\citep{tian2023transformer,zhang2024deep}, but these are reconstruction and restoration tasks rather than the binary-ranking problem studied here. This literature fixes the target of interest. Our analysis focuses on the training dynamics induced by that target: a pairwise SGD update must convert opposite-class ranking errors into corrective motion of the scorer.

\paragraph{AUC optimization algorithms.}
The quadratic coupling between positives and negatives makes AUC maximization structurally different from pointwise ERM. This has motivated online AUC optimization, pairwise stochastic methods, one-pass and incomplete-pair methods, proximal methods, triply stochastic methods, distributed methods, and nonlinear extensions~\citep{zhao2011online,kar2013generalization,gao2013one,clemenccon2016scaling,natole2018stochastic,lei2021stochastic,guo2020communication,9200544}. Deep AUC training adds representation learning, robust surrogates, warm-up, and task-dependent optimization design~\citep{liu2019stochastic,yuan2022compositional}. AUC-oriented algorithms have also been developed for long-tail adversarial training and domain-robust or distributionally robust settings~\citep{hou2022adauc,dai2023drauc}. These works make pairwise AUC computable at scale. Adjacent long-tailed learning research addresses class imbalance beyond AUC through oriented calibration for diffusion models, parameter-space rebalancing for recognition, and explicit capacity manipulation for class-imbalanced diffusion training~\citep{zhang2024longtailed,NEURIPS2025_c4889bd7,hong2026improving}. Related imbalance-aware objective design also appears in rubric-based LLM reinforcement learning, where criterion-level rewards are dynamically rebalanced according to their saturation~\citep{huang2026focalrewardbalancedreinforcement}. Our contribution is complementary: even when the pairwise update is computable, its local drift and covariance can stop turning severe misrankings into effective motion.

\paragraph{Surrogate validity versus surrogate dynamics.}
Population-level theory explains why several losses are meaningful for ranking. Bayes-optimal bipartite ranking, univariate surrogate control, proper-loss regret transfers, and pairwise consistency results address calibration, regret, and statistical validity~\citep{menon2014bayes,kotlowski2011bipartite,agarwal2013surrogate,agarwal2014surrogate,menon2016bipartite,ijcai2015p137}. The present paper studies a different property of the same losses. Logistic, exponential, square, hinge, and squared-hinge losses assign different weights to pairwise margins; those weights determine whether severe mistakes receive stronger correction, whether easy pairs continue to consume update mass, and whether useful stochastic variation remains available along difficult trajectories.

\paragraph{U-statistics and finite-batch stochastic structure.}
Empirical AUC risks are two-sample U-statistics. Classical U-statistic theory gives the Hoeffding projection decomposition and the corresponding first-order and canonical components~\citep{lee2019u,serfling2009approximation}. We use this decomposition as a stochastic-dynamics tool. For pairwise mini-batches with \(B_+\) positives and \(B_-\) negatives, the first-order projection scales are \(B_+^{-1}\) and \(B_-^{-1}\), while the residual pair-interaction scale is \((B_+B_-)^{-1}\). When the first-order projections vanish in a bad region, the remaining canonical term has scale \((B_+B_-)^{-1}\), which becomes \(B^{-2}\) in the equal-size case. This variance split is the technical source of the escape-time contrast used in the main text.

\paragraph{Auxiliary-state AUC methods.}
Several scalable AUC methods reformulate the update process rather than optimize a raw pairwise objective in the most direct mini-batch form. Stochastic saddle-point AUC maximization~\citep{ying2016stochastic}, stochastic proximal AUC methods~\citep{natole2018stochastic}, deep AUC~\citep{liu2019stochastic}, robust deep-AUC surrogates~\citep{yuan2021large}, compositional AUC training~\citep{wang2017stochastic,yuan2022compositional,pmlr-v162-wang22ak}, and adversarial or distributionally robust AUC formulations~\citep{hou2022adauc,dai2023drauc} introduce dual variables, proximal structure, auxiliary trackers, nested stochastic updates, or robustness-oriented inner problems. These methods are related because they redesign the stochastic process used for AUC training. The main text studies direct pairwise mini-batch SGD; auxiliary-state methods provide broader algorithmic context.

\paragraph{Partial AUC and top-focused variants.}
Partial AUC and top-focused ranking sharpen the same non-decomposable structure by emphasizing selected portions of the ROC curve or the highest-ranked examples. Structural partial-AUC optimization~\citep{pmlr-v28-narasimhan13}, generic and end-to-end two-way partial-AUC frameworks~\citep{yang2021all,yang2022optimizing}, asymptotically unbiased instance-wise partial-AUC optimization~\citep{shao2022asymptotically}, distributionally robust deep partial-AUC methods~\citep{zhu2022auc}, and stochastic primal-dual algorithms for two-way partial AUC~\citep{zhou2025stochastic} introduce selection, sorting, distributional robustness, min-max reformulations, or primal-dual state. We use full pairwise AUC as the clean setting for the proofs, while these variants reinforce the need for optimization analyses that go beyond pointwise ERM.

\paragraph{Warm-up, pretraining, and parameter-efficient fine-tuning.}
Practical deep-AUC pipelines often use pointwise warm-up, pretrained backbones, or low-dimensional trainable modules before or during AUC optimization~\citep{liu2019stochastic,yuan2021large,yuan2022compositional}. In large-model settings, adapters, prefix tuning, prompt tuning, LoRA, and QLoRA are standard forms of parameter-efficient fine-tuning from pretrained models~\citep{houlsby2019parameter,li-liang-2021-prefix,lester2021power,hu2021lora,dettmers2023qlora}. These questions become increasingly relevant as pretrained systems extend to fine-grained multimodal large models and as resource-efficient architectures and hardware--software co-design receive growing attention~\citep{peng2026survey,xiao2026survey}. Initialization and alignment results further show that the starting point can change learnability and optimization behavior~\citep{erhan2010does,abbe2022initial}. Our theory separates the roles of these practices: CE warm-up changes the starting region of the AUC phase, while the pretrained base and trainable subspace can place the pairwise phase closer to a favorable starting region.

\paragraph{Local stochastic dynamics and escape.}
The bridge from discrete SGD updates to local stochastic dynamics uses standard stochastic-approximation and Markov-process tools~\citep{robbins1951stochastic,kushner2003stochastic,ethier2009markov,JMLR:v20:17-526,fontaine2021convergence}. In our setting, the drift is governed by severity-weighted pairwise correction, while the covariance is governed by the Hoeffding projection structure of the mini-batch U-statistic. Appendix~\ref{app:localized-escape-time} turns the \(B^{-1}\) versus \(B^{-2}\) directional second-moment regimes into the projected escape-time statement used in the main text.

\section{Experimental Details}
\label{app:experimental-details}

\subsection{Details of the Controlled Two-Dimensional Simulation}
\label{app:fig1-simulation}

This subsection gives the full configuration for Figure~\ref{fig:intro-controlled-dynamics}. The simulation makes the local dynamical mechanism in the introduction reproducible on a fixed two-dimensional parameter plane. All curves in the figure are generated by running the stated update rules; no trajectory is hand drawn or smoothed.

\paragraph{Data generation.}
The controlled simulation uses a fixed synthetic binary ranking dataset generated from the following configuration. The random seed is \(1136848004\). Each group parameter is \(n_{\mathrm{per\ group}}=12\), giving four groups with counts
\begin{equation*}
\max(2,\lfloor n_{\mathrm{per\ group}}/3\rfloor),\quad
2\cdot n_{\mathrm{per\ group}},\quad
n_{\mathrm{per\ group}},\quad
4\cdot n_{\mathrm{per\ group}},
\end{equation*}
where groups \(0\) and \(2\) are positive and groups \(1\) and \(3\) are negative. The data-generation parameters are
\begin{equation*}
\alpha=0.012,\qquad
\beta=0.64,\qquad
\texttt{hard\_offset}=0.16,\qquad
\texttt{nonlinear}=0.16,
\end{equation*}
\begin{equation*}
\texttt{bump}=0.46,\qquad
\texttt{rbf\_scale}=0.08,\qquad
\texttt{fourier\_scale}=0.15.
\end{equation*}
For a point \(\theta=(\theta_1,\theta_2)\), the sample score has the form
\begin{equation*}
f_\theta(x_i)
=
a_i\cdot \theta_1+b_i\cdot \theta_2+c_{1i}\cdot h_{1i}(\theta_1,\theta_2)
+c_{2i}\cdot h_{2i}(\theta_1,\theta_2)
+c_{3i}\cdot q_i(\theta_1,\theta_2)
+r_i(\theta_1,\theta_2)+s_i(\theta_1,\theta_2)+o_i.
\end{equation*}
Here \(h_{1i}\) is a tanh feature, \(h_{2i}\) is a sinusoidal feature, \(q_i\) is a Gaussian bump, \(r_i\) is a sum of \(14\) radial-basis features, \(s_i\) is a sum of \(8\) Fourier features, and \(o_i\) is a sample offset. Their coefficients are generated once from the seed above and then held fixed for all panels and all update rules.

\paragraph{Landscape definition.}
The displayed landscape is the empirical AUC of the same fixed sample set on a two-dimensional parameter slice. For each grid point \((\theta_1,\theta_2)\), the positive-margin convention is
\begin{equation*}
m_\theta(x_i^+,x_j^-)=f_\theta(x_i^+)-f_\theta(x_j^-),
\end{equation*}
and the plotted value is
\begin{equation*}
\widehat{\operatorname{AUC}}(\theta)
=
\frac{1}{N_+\cdot N_-}
\sum_{i:y_i=+1}\sum_{j:y_j=-1}
\left\{
\1_{\{m_\theta(x_i,x_j)>0\}}
+
\frac{1}{2}\cdot\1_{\{m_\theta(x_i,x_j)=0\}}
\right\}.
\end{equation*}
The contour is evaluated on the fixed window
\begin{equation*}
\theta_1\in[-1.08,2.18],\qquad \theta_2\in[-1.60,0.46],
\end{equation*}
using a \(300\times250\) grid. The filled contour levels are evenly spaced between \(0.50\) and \(1.00\).

\paragraph{Surrogate losses and update rules.}
Figure~\ref{fig:intro-controlled-dynamics} displays direct hinge and direct squared-hinge AUC trajectories, together with CE warm-up trajectories that switch to hinge or squared-hinge after the warm-up phase. The pairwise losses used in the displayed AUC phases are
\begin{equation*}
\ell_{\mathrm{hinge}}(x)=(m-x)_+,\qquad
\ell_{\mathrm{squared\mbox{-}hinge}}(x)=(m-x)_+^2,
\end{equation*}
with margin \(m=0.8\). The AUC vector field uses the full finite-sample pair average
\begin{equation*}
\frac{1}{N_+\cdot N_-}
\sum_{i:y_i=+1}\sum_{j:y_j=-1}
w_\ell(m_\theta(x_i,x_j))\cdot
\nabla_\theta m_\theta(x_i,x_j),
\qquad
w_\ell(x)=-\ell'(x).
\end{equation*}
The CE vector field is the full finite-sample gradient direction of the binary cross-entropy objective on the same samples. Direct pairwise AUC-surrogate curves repeatedly apply the corresponding pairwise surrogate vector field. Warm-up curves apply the CE field first and then switch to the chosen pairwise surrogate field.

\paragraph{Optimization settings.}
The AUC step size is
\begin{equation*}
\eta_{\mathrm{auc}}=0.15,
\end{equation*}
the CE step size is
\begin{equation*}
\eta_{\mathrm{ce}}=0.105,
\end{equation*}
the warm-up length is \(26\) steps, and the maximum trajectory length is \(360\) steps. The severe-misranking threshold used by the scan is \(M=0.2\), and the severe-badness exit threshold is \(0.0015\), used only to detect when a trajectory leaves the severe-start region. The trajectories are clamped to the numerical box \([-8,8]^2\) only as a stability guard; the displayed trajectories stay in the plotted window.

\paragraph{Panel initializations.}
All three panels use the same generated sample set, the same landscape, and the same optimization hyperparameters. They differ only in the initialization:
\begin{equation*}
\begin{array}{ccl}
\text{panel (a)} &:& (\theta_{1,0},\theta_{2,0})=(-0.450,-1.3775),\\
\text{panel (b)} &:& (\theta_{1,0},\theta_{2,0})=(-0.075,-1.0450),\\
\text{panel (c)} &:& (\theta_{1,0},\theta_{2,0})=(0.750,-0.5375).
\end{array}
\end{equation*}
The corresponding initial severe-misranking fractions are \(0.3993\), \(0.2752\), and \(0.1372\), respectively. Final empirical AUC values for the trajectories emphasized in the main figure are:
\begin{equation*}
\begin{array}{c|ccccc}
\text{panel} & \text{CE} & \text{hinge} & \text{squared-hinge} &
\text{warm}\to\text{hinge} & \text{warm}\to\text{squared-hinge}\\
\hline
\text{(a)} & 0.9106 & 0.6016 & 0.6207 & 0.9028 & 0.9766\\
\text{(b)} & 0.9106 & 0.9028 & 0.9766 & 0.9028 & 0.9766\\
\text{(c)} & 0.9106 & 0.9757 & 0.9766 & 0.9757 & 0.9766
\end{array}
\end{equation*}

\paragraph{Plotting choices.}
The plotted paths are raw simulated iterates. Stars mark initializations, large circles mark final iterates, and small dots mark every \(20\) update steps. The contour is evaluated once on the fixed grid above and reused across all panels. The displayed runs are fixed reproducible simulations corresponding to the three regimes discussed in the introduction: weak correction from a severe start, recovery from a milder start, and favorable-start behavior. All reported trajectories are raw iterates from the corresponding update rules.

\subsection{Details of the Noise-Scale Escape Simulation}
\label{app:fig2-noise-escape}

This subsection gives the configuration for panel (a) of Figure~\ref{fig:surrogate-weight-geometry}. The simulation serves as a controlled visualization of the escape-time mechanism in Corollary~\ref{cor:main-escape-time-scale}. It keeps the potential, initialization, drift, and time horizon fixed, and changes only the variance scale from the first-order \(B^{-1}\) regime to the collapsed \(B^{-2}\) regime.

\paragraph{Potential and landscape.}
The one-dimensional state is denoted by \(z\). The unnormalized potential is
\begin{equation*}
\begin{aligned}
U_{\mathrm{raw}}(z)
={}&
0.105\cdot(z+1.45)^2\cdot(z-1.34)^2\\
&+0.18\cdot\exp\!\left(-\left(\frac{z-0.05}{0.37}\right)^2\right)
-0.075\cdot z\\
&+0.018\cdot\sin(2.6\cdot z+0.2).
\end{aligned}
\end{equation*}
For plotting and simulation, this potential is rescaled on the grid \(z\in[-2.10,2.02]\) by
\begin{equation*}
U(z)=0.04+
0.90\cdot
\frac{U_{\mathrm{raw}}(z)-\min_{\tilde z}U_{\mathrm{raw}}(\tilde z)}
{\max_{\tilde z}U_{\mathrm{raw}}(\tilde z)-\min_{\tilde z}U_{\mathrm{raw}}(\tilde z)}.
\end{equation*}
The grid has \(3000\) points. The displayed state window is \(z\in[-1.92,1.82]\). The left well minimum, barrier, right well minimum, and escape threshold are
\begin{equation*}
z_{\mathrm{left}}=-1.3691,\qquad
z_{\mathrm{barrier}}=0.0170,\qquad
z_{\mathrm{right}}=1.4032,\qquad
z_{\mathrm{esc}}=0.5370.
\end{equation*}
The initialization is \(z_0=-1.42\).

\paragraph{Stochastic update.}
The simulation uses an Euler discretization of overdamped motion in the potential:
\begin{equation*}
z_{t+1}
=
z_t-\Delta t\cdot U'(z_t)+\sigma_B\cdot\xi_t,
\qquad
\xi_t\sim\mathcal N(0,1),
\end{equation*}
with independent \(\xi_t\)'s across time and replicates. The derivative \(U'\) is computed by finite differences on the same \(3000\)-point grid and linearly interpolated at the current state. States are clipped to \([-2.05,1.95]\) as a numerical guard.

\paragraph{Noise regimes and hyperparameters.}
Both regimes use
\begin{equation*}
B=8,\qquad
\Delta t=0.025,\qquad
\texttt{noise\_base}=4.5,\qquad
T=180
\end{equation*}
iterations. The first-order regime uses
\begin{equation*}
\sigma_B=\texttt{noise\_base}\cdot\sqrt{\Delta t/B},
\end{equation*}
while the collapsed regime uses
\begin{equation*}
\sigma_B=\texttt{noise\_base}\cdot\sqrt{\Delta t}/B.
\end{equation*}
The empirical density in Figure~\ref{fig:surrogate-weight-geometry}(a) is computed from \(12000\) independent replicates per regime and sampled every \(2\) iterations. The random seeds are \(20260524\) for the \(B^{-1}\) regime and \(20260512\) for the \(B^{-2}\) regime. A replicate is counted as escaped once \(z_t\ge z_{\mathrm{esc}}\) for some \(t\le T\). With the above configuration, the final escape fractions are \(0.76275\) for the first-order regime and \(0.0265\) for the collapsed regime, shown in the figure as \(76\%\) and \(3\%\).

\paragraph{Rendering choices.}
The left curve in panel (a) is the same potential \(U\), rotated so that the vertical axis is the state variable \(z\). The two right subpanels show smoothed empirical state-density heatmaps over iteration and state. Horizontal guide lines mark the initialization, barrier, and target well. The density smoothing is used only for rendering; the escape fractions are computed from the raw simulated paths.

\subsection{Main Experimental Protocol}
\label{app:real-experiment-protocol}

This subsection fixes the reporting protocol behind Table~\ref{tab:main-surrogate-family-results} and the expanded tables below. Each experiment is a binary ranking task evaluated by held-out AUC. For every model--task pair, the protocol compares CE with five pairwise surrogate families: square, hinge, exponential, logistic, and squared-hinge. Pairwise AUC training is run either directly from the beginning or after a CE warm-up phase.

\paragraph{Reported quantity.}
Every number in the result tables is the held-out test AUC of the validation-best checkpoint, averaged over matching runs. The main table and the expanded appendix tables report mean held-out AUC values with sample standard deviations. When a dataset does not provide fixed train, validation, and test splits, the same seed convention is used to create stratified splits before training.

\paragraph{Direct and warm-started rows.}
For CE, the main table reports the CE run. For each pairwise surrogate family, we run both direct pairwise AUC-surrogate training and CE-warm-started pairwise AUC-surrogate training when both are part of the grid. Table~\ref{tab:main-surrogate-family-results} reports these two modes as separate direct and warm-up rows for each task, with bold and underline marking the best and second-best results within each task. The expanded appendix tables keep the same direct/warm-started separation for the remaining model-family blocks.

\paragraph{Expanded results corresponding to Table 1.}
Table~\ref{tab:main-surrogate-family-results} displays DenseNet121 and Qwen2.5-1.5B-Instruct with direct and CE-warm-started rows for all five pairwise surrogate families. Table~\ref{tab:app-full-resnet50-distilbert} gives the corresponding ResNet50 and DistilBERT-base results: each task is split into a direct row and a CE-warm-started row, and each cell reports the mean held-out AUC together with the corresponding sample standard deviation.

\newcommand{\aucstd}[2]{#1$_{\scriptstyle \pm #2}$}
\newcommand{\appbestcell}[3]{\cellcolor{#1}{\bfseries #2}$_{\scriptstyle \pm #3}$}
\newcommand{\appsecondcell}[3]{\cellcolor{#1}\underline{#2}$_{\scriptstyle \pm #3}$}
\makeatletter
\newcommand{\appinplacetablecaption}[2]{%
  \refstepcounter{table}%
  \@makecaption{\tablename~\thetable}{#2}%
  \label{#1}%
}
\makeatother

\par\Needspace{0.30\textheight}
\begingroup
\centering
\appinplacetablecaption{tab:app-full-resnet50-distilbert}{\textbf{Expanded direct and CE-warm-started results (AUC \%) for ResNet50 and DistilBERT-base.} This table reports the model-family blocks not shown in Table~\ref{tab:main-surrogate-family-results}. Each task has a direct row and a warm-up row. Values are mean held-out AUC (\%), with the \(\pm\) sample standard deviation shown as a subscript. Within each task, the best and second-best configurations are bolded with darker shading and underlined with lighter shading, respectively.}
\scriptsize
\setlength{\tabcolsep}{1.25pt}
\renewcommand{\arraystretch}{1.10}
\begin{tabularx}{\textwidth}{@{}>{\raggedright\arraybackslash}p{0.205\textwidth}>{\centering\arraybackslash}p{0.083\textwidth}*{6}{>{\centering\arraybackslash}X}@{}}
\toprule
Task & Mode & CE & Square & Hinge & Exp. & Logistic & Sq.-hinge \\
\midrule
\rowcolor{resnethead}\multicolumn{8}{c}{{\bfseries ResNet50}}\\
\midrule
ImageNet-LT Dogs & Direct & \aucstd{91.34}{1.00} & \aucstd{93.56}{0.38} & \aucstd{79.50}{3.75} & \aucstd{91.35}{0.86} & \aucstd{93.67}{0.47} & \appbestcell{resnetbest}{95.83}{0.14} \\
 & Warm-up & -- & \aucstd{93.84}{0.34} & \aucstd{82.33}{2.75} & \aucstd{91.00}{1.48} & \aucstd{93.51}{0.54} & \appsecondcell{resnetsecond}{94.42}{0.41} \\
\addlinespace[1.5pt]
Dog Breeds & Direct & \aucstd{86.06}{1.09} & \aucstd{85.09}{1.08} & \aucstd{84.63}{1.28} & \aucstd{84.68}{0.84} & \aucstd{84.82}{1.28} & \aucstd{85.65}{1.04} \\
 & Warm-up & -- & \appsecondcell{resnetsecond}{86.52}{1.03} & \aucstd{83.72}{1.74} & \aucstd{85.87}{0.91} & \aucstd{85.30}{1.79} & \appbestcell{resnetbest}{87.00}{1.44} \\
\addlinespace[1.5pt]
CRJ700/CRJ900 & Direct & \aucstd{53.57}{7.75} & \aucstd{54.99}{8.32} & \aucstd{54.77}{8.25} & \aucstd{54.93}{8.27} & \aucstd{55.54}{7.40} & \appsecondcell{resnetsecond}{55.57}{8.12} \\
 & Warm-up & -- & \aucstd{54.66}{7.30} & \aucstd{54.77}{8.25} & \aucstd{54.73}{7.27} & \aucstd{54.51}{7.30} & \appbestcell{resnetbest}{56.62}{6.45} \\
\addlinespace[1.5pt]
APTOS Sev./Mod. & Direct & \aucstd{81.02}{1.47} & \appsecondcell{resnetsecond}{82.70}{0.63} & \aucstd{81.54}{0.84} & \aucstd{82.01}{0.34} & \aucstd{81.79}{1.04} & \appbestcell{resnetbest}{83.92}{1.09} \\
 & Warm-up & -- & \aucstd{80.49}{0.95} & \aucstd{80.62}{1.61} & \aucstd{80.70}{1.22} & \aucstd{80.38}{0.78} & \aucstd{81.26}{1.34} \\
\addlinespace[1.5pt]
DeepWeeds P/S & Direct & \aucstd{96.39}{1.37} & \aucstd{96.98}{2.35} & \aucstd{96.99}{2.36} & \aucstd{96.98}{2.37} & \aucstd{96.90}{2.41} & \appsecondcell{resnetsecond}{97.04}{2.38} \\
 & Warm-up & -- & \aucstd{96.98}{2.35} & \aucstd{96.15}{2.76} & \aucstd{96.98}{2.37} & \aucstd{96.90}{2.41} & \appbestcell{resnetbest}{97.12}{2.61} \\
\addlinespace[2pt]\midrule\addlinespace[1pt]
\rowcolor{distilhead}\multicolumn{8}{c}{{\bfseries DistilBERT-base}}\\
\midrule
GoEmotions Annoy. & Direct & \aucstd{83.97}{0.35} & \aucstd{82.73}{1.61} & \aucstd{80.39}{3.96} & \aucstd{79.91}{4.16} & \aucstd{79.20}{4.16} & \aucstd{83.87}{0.50} \\
 & Warm-up & -- & \aucstd{83.97}{0.35} & \aucstd{82.75}{1.16} & \appsecondcell{distilsecond}{83.98}{0.35} & \aucstd{83.97}{0.35} & \appbestcell{distilbest}{84.31}{0.38} \\
\addlinespace[1.5pt]
GoEmotions Fear & Direct & \aucstd{94.74}{0.58} & \aucstd{94.05}{1.65} & \aucstd{93.50}{0.73} & \aucstd{93.07}{0.80} & \aucstd{94.67}{1.48} & \appsecondcell{distilsecond}{95.62}{0.82} \\
 & Warm-up & -- & \aucstd{93.89}{2.03} & \aucstd{93.84}{1.74} & \aucstd{94.25}{1.22} & \aucstd{93.26}{0.60} & \appbestcell{distilbest}{96.48}{0.31} \\
\addlinespace[1.5pt]
ANLI R1 E/N & Direct & \appbestcell{distilbest}{70.21}{1.21} & \aucstd{66.37}{2.07} & \aucstd{61.51}{5.60} & \aucstd{64.44}{1.84} & \aucstd{63.32}{2.28} & \aucstd{67.84}{1.63} \\
 & Warm-up & -- & \aucstd{69.07}{1.52} & \appsecondcell{distilsecond}{69.86}{1.40} & \aucstd{69.09}{1.10} & \aucstd{68.48}{2.16} & \aucstd{69.77}{1.35} \\
\addlinespace[1.5pt]
ANLI R3 C/N & Direct & \appsecondcell{distilsecond}{62.09}{0.84} & \aucstd{61.48}{0.63} & \aucstd{57.40}{2.26} & \appbestcell{distilbest}{62.47}{2.17} & \aucstd{61.76}{1.62} & \aucstd{61.22}{1.30} \\
 & Warm-up & -- & \aucstd{61.88}{0.64} & \aucstd{61.09}{0.77} & \aucstd{59.64}{1.35} & \aucstd{61.31}{0.98} & \aucstd{61.81}{0.69} \\
\addlinespace[1.5pt]
ANLI R3 E/N & Direct & \appsecondcell{distilsecond}{63.49}{0.62} & \aucstd{62.40}{1.23} & \aucstd{58.85}{0.79} & \aucstd{58.49}{2.04} & \aucstd{60.11}{0.85} & \aucstd{62.24}{0.86} \\
 & Warm-up & -- & \aucstd{62.59}{1.60} & \aucstd{62.86}{0.59} & \aucstd{63.00}{0.67} & \aucstd{63.14}{0.57} & \appbestcell{distilbest}{63.87}{0.51} \\
\addlinespace[1.5pt]
SNLI E/N & Direct & \appsecondcell{distilsecond}{92.56}{0.34} & \aucstd{89.07}{2.98} & \aucstd{86.21}{1.25} & \aucstd{86.69}{1.06} & \aucstd{84.42}{3.97} & \aucstd{92.52}{0.43} \\
 & Warm-up & -- & \aucstd{92.21}{0.44} & \aucstd{92.21}{0.49} & \aucstd{92.37}{0.20} & \aucstd{92.21}{0.44} & \appbestcell{distilbest}{92.98}{0.22} \\
\addlinespace[1.5pt]
WANLI C/N & Direct & \appsecondcell{distilsecond}{82.39}{0.49} & \aucstd{80.58}{0.35} & \aucstd{75.51}{3.21} & \aucstd{79.54}{0.61} & \aucstd{77.89}{0.61} & \aucstd{82.08}{0.36} \\
 & Warm-up & -- & \aucstd{81.74}{0.43} & \aucstd{81.79}{0.48} & \aucstd{81.90}{0.34} & \aucstd{81.84}{0.40} & \appbestcell{distilbest}{82.41}{0.32} \\
\addlinespace[1.5pt]
WANLI E/N & Direct & \appbestcell{distilbest}{75.44}{0.50} & \aucstd{74.43}{0.27} & \aucstd{65.37}{4.09} & \aucstd{73.42}{0.35} & \aucstd{71.60}{1.88} & \aucstd{74.76}{0.67} \\
 & Warm-up & -- & \aucstd{74.72}{0.27} & \aucstd{73.99}{0.68} & \aucstd{74.63}{0.36} & \aucstd{74.24}{0.56} & \appsecondcell{distilsecond}{75.41}{0.29} \\
\bottomrule
\end{tabularx}
\par\endgroup
\par\addvspace{8pt}

\subsection{Data Sources and Task Construction}
\label{app:datasets-tasks}

The main experiments use binary ranking tasks. Appendix~\ref{app:extended-auc-experiments} further reports long-tailed multiclass, multilabel, and large-scale multiclass settings. For the binary tasks, each row defines a positive class and a negative class or negative class set, and all models are evaluated by held-out AUC on the resulting binary task. When a source dataset provides official train, validation, and test splits, those splits are used. Otherwise, stratified splits follow the reporting protocol. Text fields are converted into a single model input for NLP tasks, and labels are binarized before training.

\paragraph{CV tasks.}
The CV tasks are constructed from public image datasets and trained with ImageNet-normalized image pipelines. The CV rows use the following sources and binary definitions.
\begin{table}[t]
\caption{CV task construction. Each row gives the source dataset and the binary ranking task used in the experiments.}
\label{tab:app-cv-task-construction}
\centering
\small
\begin{tabularx}{\textwidth}{p{0.25\textwidth}p{0.26\textwidth}Y}
\toprule
\textbf{Task} & \textbf{Source} & \textbf{Binary construction}\\
\midrule
ImageNet-LT Dogs & ImageNet-LT split files with ImageNet images & Tail dog classes versus sibling dog classes.\\
Dog Breeds & Stanford Dogs & Spitz, husky, and malamute versus Eskimo dog and Samoyed.\\
CRJ700/CRJ900 & FGVC-Aircraft & Bombardier CRJ700 versus Bombardier CRJ900.\\
APTOS Sev./Mod. & APTOS diabetic-retinopathy fundus images & Severe-or-worse retinopathy versus moderate retinopathy.\\
DeepWeeds P/S & DeepWeeds & Parthenium weed versus Siam weed, with low-shot training construction.\\
Apple Disease & D-KAP apple leaf disease archive & Apple scab versus apple rot, with low-shot training construction.\\
Rice Disease & Rice Leaf Diseases & Brown spot versus leaf smut, with low-shot training construction.\\
HAM10000 M/N & HAM10000 & Melanoma versus melanocytic nevus.\\
\bottomrule
\end{tabularx}
\end{table}
The public CV sources are ImageNet-LT split files derived from ImageNet~\citep{liu2019large,deng2009imagenet}, Stanford Dogs~\citep{KhoslaYaoJayadevaprakashFeiFei_FGVC2011}, FGVC-Aircraft~\citep{maji2013fine}, APTOS 2019~\citep{aptos2019}, DeepWeeds~\citep{olsen2019deepweeds}, D-KAP~\citep{sharma2022d}, Rice Leaf Diseases~\citep{prajapati2017detection}, and HAM10000~\citep{tschandl2018ham10000}. Public releases are available for \href{https://www.kaggle.com/datasets/miljan/stanford-dogs-dataset-traintest}{Stanford Dogs}, \href{https://www.kaggle.com/datasets/seryouxblaster764/fgvc-aircraft}{FGVC-Aircraft}, \href{https://www.kaggle.com/datasets/sovitrath/diabetic-retinopathy-224x224-2019-data}{APTOS 2019}, \href{https://www.kaggle.com/datasets/imsparsh/deepweeds}{DeepWeeds}, \href{https://www.kaggle.com/datasets/hsmcaju/d-kap}{D-KAP}, \href{https://www.kaggle.com/datasets/vbookshelf/rice-leaf-diseases}{Rice Leaf Diseases}, and \href{https://www.kaggle.com/datasets/kmader/skin-cancer-mnist-ham10000}{HAM10000}. The image preprocessing uses random resized crop to \(224\), random horizontal flip, RandAugment~\citep{cubuk2020randaugment} with \(2\) operations and magnitude \(9\), conversion to tensor, and ImageNet normalization for training; validation and test use resize, center crop, tensor conversion, and ImageNet normalization.

The extension experiments additionally use CIFAR-100~\citep{krizhevsky2009learning}, iNaturalist 2018~\citep{van2018inaturalist}, and CheXpert~\citep{irvin2019chexpert}. The CIFAR-100 experiments construct long-tailed medium-mammals and vehicles tasks from the official classes; iNaturalist 2018 uses the official training and validation sets for an \(8{,}142\)-class long-tailed task; and CheXpert uses patient-disjoint training and internal-validation data, with the official validation set reserved for final testing. The task compositions, split sizes, and evaluation metrics are specified in Appendix~\ref{app:extended-auc-experiments}.

\paragraph{NLP tasks.}
The DistilBERT and Qwen experiments use the same NLP task definitions. GoEmotions tasks are one-versus-rest emotion detection tasks; the official test split is used and a stratified validation split is carved from the official training split. ANLI, SNLI, and WANLI rows are prepared binary natural-language-inference contrasts with train, validation, and test splits fixed before model training.
\begin{table}[t]
\caption{NLP task construction. DistilBERT and Qwen use the same binary task definitions.}
\label{tab:app-nlp-task-construction}
\centering
\small
\begin{tabularx}{\textwidth}{p{0.255\textwidth}p{0.305\textwidth}Y}
\toprule
\textbf{Task} & \textbf{Source} & \textbf{Binary construction}\\
\midrule
GoEmotions Annoyance & GoEmotions, simplified & Annoyance label versus all other examples.\\
GoEmotions Fear & GoEmotions, simplified & Fear label versus all other examples.\\
ANLI R1 Ent./Neu. & ANLI Round 1 prepared split & Entailment versus neutral.\\
ANLI R3 Con./Neu. & ANLI Round 3 prepared split & Contradiction versus neutral.\\
ANLI R3 Ent./Neu. & ANLI Round 3 prepared split & Entailment versus neutral.\\
SNLI High-Overlap & SNLI prepared high-overlap split & Entailment versus neutral on a high-overlap slice.\\
WANLI Con./Neu. & WANLI prepared split & Contradiction versus neutral.\\
WANLI Ent./Neu. & WANLI prepared split & Entailment versus neutral.\\
\bottomrule
\end{tabularx}
\end{table}
The NLP sources are Hugging Face \texttt{go\_emotions} with the simplified configuration~\citep{demszky2020goemotions}, \texttt{anli}~\citep{nie2020adversarial}, \texttt{snli}~\citep{bowman2015large}, and \texttt{alisawuffles/WANLI}~\citep{liu2022wanli}. GoEmotions uses the official text field. NLI examples are formatted as premise--hypothesis inputs. The Qwen2.5-1.5B-Instruct experiments use the same prepared binary NLP tasks as the DistilBERT experiments, so the comparison changes the pretrained model family, tokenizer, scale, and resulting LoRA trainable-subspace geometry rather than the task distribution.

\paragraph{Split sizes.}
The following tables report the exact sample counts used by the result tables. Counts are written as total examples with positive/negative counts in parentheses.
\begin{table}[t]
\caption{CV split sizes. In the split columns, \(N\) is the total number of examples and pos./neg. gives the positive/negative counts.}
\label{tab:app-cv-split-sizes}
\centering
\small
\begin{tabularx}{\textwidth}{p{0.20\textwidth}p{0.245\textwidth}p{0.255\textwidth}Y}
\toprule
\textbf{CV task} & \textbf{Train (\(N\); pos./neg.)} & \textbf{Val. (\(N\); pos./neg.)} & \textbf{Test (\(N\); pos./neg.)}\\
\midrule
ImageNet-LT Dogs & \(6788\) (\(202/6586\)) & \(1080\) (\(360/720\)) & \(2700\) (\(900/1800\))\\
Dog Breeds & \(517\) (\(259/258\)) & \(111\) (\(56/55\)) & \(110\) (\(55/55\))\\
CRJ700/CRJ900 & \(50\) (\(25/25\)) & \(30\) (\(15/15\)) & \(66\) (\(33/33\))\\
APTOS Sev./Mod. & \(1041\) (\(342/699\)) & \(223\) (\(73/150\)) & \(223\) (\(73/150\))\\
DeepWeeds P/S & \(50\) (\(25/25\)) & \(30\) (\(15/15\)) & \(80\) (\(40/40\))\\
Apple Disease & \(50\) (\(25/25\)) & \(30\) (\(15/15\)) & \(80\) (\(40/40\))\\
Rice Disease & \(36\) (\(18/18\)) & \(16\) (\(8/8\)) & \(28\) (\(14/14\))\\
HAM10000 M/N & \(5473\) (\(779/4694\)) & \(1173\) (\(167/1006\)) & \(1172\) (\(167/1005\))\\
\bottomrule
\end{tabularx}
\end{table}
\begin{table}[t]
\caption{NLP split sizes. In the split columns, \(N\) is the total number of examples and pos./neg. gives the positive/negative counts.}
\label{tab:app-nlp-split-sizes}
\centering
\small
\begin{tabularx}{\textwidth}{p{0.20\textwidth}p{0.255\textwidth}p{0.255\textwidth}Y}
\toprule
\textbf{NLP task} & \textbf{Train (\(N\); pos./neg.)} & \textbf{Val. (\(N\); pos./neg.)} & \textbf{Test (\(N\); pos./neg.)}\\
\midrule
GoEmotions Annoy. & \(39069\) (\(2223/36846\)) & \(4341\) (\(247/4094\)) & \(5427\) (\(320/5107\))\\
GoEmotions Fear & \(39069\) (\(536/38533\)) & \(4341\) (\(60/4281\)) & \(5427\) (\(78/5349\))\\
ANLI R1 E/N & \(12423\) (\(5371/7052\)) & \(667\) (\(334/333\)) & \(667\) (\(334/333\))\\
ANLI R3 C/N & \(20000\) (\(10000/10000\)) & \(798\) (\(396/402\)) & \(798\) (\(396/402\))\\
ANLI R3 E/N & \(20000\) (\(10000/10000\)) & \(804\) (\(402/402\)) & \(804\) (\(402/402\))\\
SNLI High-Overlap & \(20000\) (\(10000/10000\)) & \(4000\) (\(2000/2000\)) & \(4000\) (\(2000/2000\))\\
WANLI C/N & \(20000\) (\(10000/10000\)) & \(3540\) (\(1540/2000\)) & \(3142\) (\(745/2397\))\\
WANLI E/N & \(20000\) (\(10000/10000\)) & \(4000\) (\(2000/2000\)) & \(3858\) (\(1858/2000\))\\
\bottomrule
\end{tabularx}
\end{table}

\subsection{Models, Optimization, and Reporting Protocol}
\label{app:training-details}

All deep-learning experiments were run on a server with four NVIDIA RTX 3090 GPUs, each with 24GB of memory. Each training run uses one GPU; runs are scheduled independently across available devices rather than by distributed data-parallel training. Device assignment is not used as a statistical unit. Based on the available experiment logs, all experiments involved in this paper required approximately \(1.8\times 10^3\) GPU-hours in total.

\paragraph{Models.}
The CV experiments use ImageNet-pretrained ResNet50~\citep{he2016deep} and DenseNet121~\citep{huang2017densely} backbones with a binary classification head. The NLP experiments use DistilBERT-base-uncased~\citep{sanh2019distilbert} and Qwen2.5-1.5B-Instruct~\citep{qwen2025qwen25technicalreport} sequence-classification models. The Qwen experiments use LoRA adapters~\citep{hu2021lora} with rank \(16\), LoRA alpha \(32\), dropout \(0.05\), and bfloat16 computation. The DistilBERT runs use the same binary sequence-classification pipeline with LoRA rank \(16\), LoRA alpha \(32\), dropout \(0.05\), and float32 computation.

\paragraph{Losses.}
The pairwise losses are square, hinge, exponential, logistic, and squared-hinge, using the positive margin convention
\begin{equation*}
m_\theta(x^+,x^-)=f_\theta(x^+)-f_\theta(x^-).
\end{equation*}
Direct pairwise AUC-surrogate training starts immediately with the specified pairwise loss. CE-warm-started AUC training uses CE for an initial fraction of the update budget and then switches to the specified pairwise loss.

\paragraph{Base training settings.}
The following settings define the training configurations used in the final tables. ImageNet-LT uses a separate ResNet50 configuration because the training split is much larger than the low-shot and medium-scale CV tasks.
APTOS Sev./Mod. is the only CV task that appears under both ResNet50 and DenseNet121; it uses the same binary construction and split protocol in both rows.
\begin{table}[t]
\caption{Base training settings for the experimental regimes.}
\label{tab:app-base-training-settings}
\centering
\small
\begin{tabularx}{\textwidth}{p{0.19\textwidth}p{0.30\textwidth}Y}
\toprule
\textbf{Regime} & \textbf{Models/tasks} & \textbf{Optimization settings}\\
\midrule
CV, ImageNet-LT & ResNet50 on ImageNet-LT Dogs & \(20\) epochs, batch/eval batch \(128/256\), SGD with momentum \(0.9\), learning rate \(0.05\), weight decay \(10^{-4}\), cosine annealing, positive batch fraction \(0.25\), CE-to-AUC switch LR scale \(1.0\).\\
CV, other ResNet50 & Dog Breeds, CRJ700/CRJ900, APTOS Sev./Mod., DeepWeeds P/S & \(10\) epochs, batch/eval batch \(64/128\), AdamW~\citep{loshchilov2018decoupled}, learning rate \(10^{-4}\), weight decay \(10^{-4}\), cosine annealing, positive batch fraction \(0.25\), CE-to-AUC switch LR scale \(0.5\).\\
CV, DenseNet121 & Apple Disease, Rice Disease, APTOS Sev./Mod., HAM10000 M/N & \(10\) epochs, batch/eval batch \(48/128\), AdamW, learning rate \(10^{-4}\), weight decay \(10^{-4}\), cosine annealing, positive batch fraction \(0.25\), CE-to-AUC switch LR scale \(0.5\).\\
DistilBERT NLP & Eight NLP tasks in Appendix~\ref{app:datasets-tasks} & \(4\) epochs, batch/eval batch \(32/64\), gradient accumulation \(1\), AdamW, learning rate \(2\cdot10^{-4}\), weight decay \(0.01\), cosine schedule with warm-up steps equal to \(5\%\) of total optimizer updates, maximum length \(128\), positive batch fraction \(0.25\), max gradient norm \(1.0\), CE-to-AUC switch LR scale \(0.5\).\\
Qwen NLP & Same eight NLP tasks & \(4\) epochs, batch/eval batch \(24/48\), gradient accumulation \(1\), AdamW, learning rate \(8\cdot10^{-5}\), weight decay \(0.01\), cosine schedule with warm-up steps equal to \(5\%\) of total optimizer updates, maximum length \(256\), positive batch fraction \(0.25\), max gradient norm \(1.0\), CE-to-AUC switch LR scale \(0.5\).\\
\bottomrule
\end{tabularx}
\end{table}
The positive batch fraction fixes the class-batch ratio used in the experiments. Thus the implementation uses unequal class counts \(B_+\) and \(B_-\), while the \(B\)-positive and \(B\)-negative notation in the main text gives the same \(B^{-1}\) versus \(B^{-2}\) Hoeffding scaling up to constants.

\paragraph{Warm-up and margin grids.}
The result grids tune the CE warm-up ratio and the pairwise margin. Direct pairwise AUC-surrogate training is the zero-warm-up case.
\begin{table}[t]
\caption{Warm-up-ratio and pairwise-margin grids used for pairwise AUC-surrogate runs.}
\label{tab:app-warmup-margin-grids}
\centering
\small
\begin{tabularx}{\textwidth}{p{0.16\textwidth}Y}
\toprule
\textbf{Regime} & \textbf{Grid}\\
\midrule
CV & Warm-up ratios \(0,0.01,0.025,0.03,0.05,0.075,0.10,0.15,0.20,0.25,0.30\). Margins \(0.1,0.2,0.3,0.4,0.5,0.6,0.7,0.8,0.9\).\\
DistilBERT & Warm-up ratios \(0,0.01,0.025,0.05,0.075,0.10,0.15,0.20,0.25,0.30\). Margins \(0.1,0.2,0.4,0.6,0.7\).\\
Qwen & Warm-up ratios \(0,0.02,0.04,0.06,0.08,0.10\). Margins \(0.1,0.2,0.4,0.6\).\\
\bottomrule
\end{tabularx}
\end{table}

\paragraph{Configuration selection and statistics.}
Within every individual run, train, validation, and test AUC are evaluated once per epoch, so candidate checkpoints are the end-of-epoch states and no intra-epoch checkpoint is used for table selection. The test AUC reported for that run is taken from the epoch with the largest validation AUC. The table-generation step groups runs with identical model, task, loss, mode, warm-up ratio, margin, epoch count, batch size, and learning rate. Each configuration is summarized by the arithmetic mean and sample standard deviation of the validation-best test AUC values across matching runs. Hyperparameters are selected by mean validation AUC across seeds; the reported value is the corresponding held-out test AUC. For the expanded direct/warm-up tables, ties are broken by smaller sample standard deviation.

\paragraph{Matched initialization control.}
To isolate initialization from backbone architecture, we compare random and ImageNet-pretrained initialization over 45 matched settings, holding the backbone, task, surrogate, direct/warm-up configuration, and model-selection rule fixed. Warm-up gain is computed within each setting as the validation-selected test AUC after CE-warm-started pairwise training minus that after direct pairwise training. All matched settings are retained, and the mean and interquartile range are reported. The comparison is treated as descriptive because no paired inferential test is added.

\paragraph{Splits, seeds, and reproducibility.}
The reported seed set is \(\{123,124,125,126,127\}\) for all experiments. Means, sample standard deviations, and validation-selected hyperparameters are computed over the corresponding seed set for each model family. Seeds are used for model initialization, data-loader shuffling, stratified batch sampling, and any stratified split construction when the source task does not provide the exact split used by the pipeline. Thus the split construction is reproducible from the stated seed and task script, while the split-size tables in Appendix~\ref{app:datasets-tasks} give the resulting positive and negative counts. Prepared NLI splits are written once before model training; GoEmotions one-versus-rest runs use the official test split and a stratified validation split carved from the official training split by seed.

\paragraph{Implementation environment.}
The experiments are implemented in Python with PyTorch, torchvision, Hugging Face Transformers, Hugging Face Datasets, PEFT, NumPy, pandas, and scikit-learn ROC-AUC evaluation. The recorded Python environment uses Python \(3.12.10\), PyTorch \(2.6.0+\mathrm{cu124}\), torchvision \(0.21.0+\mathrm{cu124}\), Transformers \(4.51.3\), tokenizers \(0.21.1\), huggingface-hub \(0.30.2\), safetensors \(0.5.3\), NumPy \(2.1.2\), pandas \(2.3.3\), and scikit-learn \(1.7.2\). CV preprocessing uses torchvision image transforms, and NLP preprocessing uses Hugging Face tokenizers with dynamic padding.

\paragraph{Evaluation, checkpointing, and mechanism logs.}
\label{app:mechanism-diagnostics}
Validation and test AUC are computed with the standard ROC-AUC estimator. The same training pipeline also records the pairwise quantities used for the mechanism measurements in Figures~\ref{fig:experiment-trajectory-diagnostics} and~\ref{fig:experiment-aggregate-diagnostics}. These mechanism measurements are computed on logged training mini-batches rather than on the held-out test set. Scores are converted to sigmoid scores before pairwise margin statistics are computed. This logged sigmoid-margin scale is used for cross-model comparability; it only changes the empirical logging scale and does not change the score-margin convention or theoretical results. Thus the logged severe-misranking thresholds are fixed thresholds in sigmoid-score difference units. The primary unthresholded statistic is the logged-batch misranking ratio
\begin{equation*}
\Pbb\{m_\theta(X^+,X^-)<0\},
\end{equation*}
which measures the fraction of logged positive-negative training pairs whose score margin has the wrong sign. We additionally log thresholded severe-misranking ratios at thresholds \(-0.01\), \(-0.025\), \(-0.05\), \(-0.075\), \(-0.1\), \(-0.15\), and \(-0.2\). At the start of pairwise training, namely the first pairwise update for direct pairwise training and the first update after the CE-to-AUC switch for warm-started training, logged positive-negative pairs are binned by their current margin. Figure~\ref{fig:experiment-aggregate-diagnostics}(c,d) uses disjoint wrong-pair intervals \(I\) and reports the relative weight proportion
\[
P_I^{\rm rel}
=
\frac{
\sum_{(i,j):\,m_{ij}\in I} q_{ij}\big/
\sum_{i=1}^{B_+}\sum_{j=1}^{B_-} q_{ij}
}{
\#\{(i,j):m_{ij}\in I\}/(B_+B_-)
},
\qquad
q_{ij}=\bigl\|w_\ell(m_{ij})\cdot(\bm g_j^- - \bm g_i^+)\bigr\|_2 .
\]
Thus the numerator is the proportion of weighted pair-gradient contribution assigned to interval \(I\), and the denominator is the pair proportion in the same interval. This quantity is the empirical counterpart of the surrogate-weight calculations in Appendix~\ref{app:surrogate-weight-calculations}: a value above \(1\) means that the interval receives more gradient weight than its frequency alone would imply, and a value below \(1\) means that it receives less. These logs support the two state-level checks used in Section~\ref{sec:experiments}: favorable surrogate weights should place more relative gradient weight on severe wrong pairs, and CE warm-up should reduce the misranking ratio at the point where pairwise AUC optimization begins.

For checks involving the severe-misranking mass and correction-direction criterion in Appendix~\ref{app:pairwise-bad-region-dynamics}, the direction is fixed mechanically before the projected quantities are measured. On a selection batch, define the severe-pair set \(\mathcal S_M(\theta)=\{(i,j):m_{ij}\le -M\}\). The empirical severe-pair objective on this batch has first-order directional change proportional to
\[
\frac{1}{|\mathcal S_M(\theta)|}
\sum_{(i,j)\in\mathcal S_M(\theta)}
w_\ell(m_{ij})
\left\langle
\bm g_j^- - \bm g_i^+,\bm u
\right\rangle .
\]
Thus the vector below is the finite-batch severe-error normal: its negative gives the steepest local direction for reducing the surrogate-weighted severe-pair margin error among unit directions in the trainable parameter space. The sign convention is immaterial for the variance term and for the absolute projected update size reported in the mechanism measurements. When \(\mathcal S_M(\theta)\) is nonempty, compute
\[
\widehat{\bm u}_{\rm sev}
=
\frac{1}{|\mathcal S_M(\theta)|}
\sum_{(i,j)\in\mathcal S_M(\theta)}
w_\ell(m_{ij})\cdot(\bm g_j^- - \bm g_i^+).
\]
If \(\mathcal S_M(\theta)\) is empty or \(\|\widehat{\bm u}_{\rm sev}\|_2\) is negligible, the checkpoint is recorded as having no stable severe-pair correction direction. Otherwise, the empirical severe-error direction is
\begin{equation}
\widehat{\bm v}_{\rm sev}
=
-\frac{\widehat{\bm u}_{\rm sev}}{\|\widehat{\bm u}_{\rm sev}\|_2}.
\label{eq:app-empirical-severe-direction}
\end{equation}
This matches the steepest-decrease orientation in Appendix~\ref{app:pairwise-bad-region-dynamics}. The deterministic rule is applied to every method, task, seed, and checkpoint before the projected update and Hoeffding terms are computed. On an independent logging batch, write
\[
\widehat\psi_{ij}
=
\left\langle
w_\ell(m_{ij})\cdot(\bm g_j^- - \bm g_i^+),
\widehat{\bm v}_{\rm sev}
\right\rangle ,
\qquad
\overline\psi
=
\frac{1}{B_+B_-}\sum_{i=1}^{B_+}\sum_{j=1}^{B_-}\widehat\psi_{ij}.
\]
The projected update size is \(|\overline\psi|\). The first-order Hoeffding variance estimate uses the row and column means
\[
\widehat r_i=\frac{1}{B_-}\sum_{j=1}^{B_-}\widehat\psi_{ij},
\qquad
\widehat c_j=\frac{1}{B_+}\sum_{i=1}^{B_+}\widehat\psi_{ij},
\]
and records
\[
\widehat V_1
=
\widehat{\Var}_i(\widehat r_i-\overline\psi)
+
\widehat{\Var}_j(\widehat c_j-\overline\psi).
\]
The first-order variance term reported in Figure~\ref{fig:experiment-trajectory-diagnostics}(a,b) is the scaled quantity
\[
\widehat V_{1,\mathrm{scaled}}
=
\frac{\widehat{\Var}_i(\widehat r_i-\overline\psi)}{B_+}
+
\frac{\widehat{\Var}_j(\widehat c_j-\overline\psi)}{B_-}.
\]
This is the finite-batch counterpart of the order-\(B^{-1}\) term in Proposition~\ref{prop:main-hoeffding-collapse}. The residual pair-interaction estimate is obtained from
\[
\widehat V_2
=
\widehat{\Var}_{i,j}(\widehat\psi_{ij}-\widehat r_i-\widehat c_j+\overline\psi),
\]
which enters the same Hoeffding split at scale \((B_+B_-)^{-1}\), reducing to \(B^{-2}\) in the equal-size case. Proposition~\ref{prop:app-finite-pair-high-error-check} gives the finite-pair check for severe-misranking mass, and Proposition~\ref{prop:app-severe-compression-small-projections} states when severe-pair compression makes \(\widehat V_1\) small.

\paragraph{Practical diagnose--intervene--recheck protocol.}
First, use the logging-batch severe-misranking ratio as an error gate: if it decreases, continue direct pairwise training; if it remains high, proceed to the motion gate. Second, construct \(\widehat{\bm v}_{\rm sev}\) on a selection batch and estimate the projected drift \(|\overline\psi|\) and first-order variance \(\widehat V_{1,\mathrm{scaled}}\) on an independent logging batch. If either source of corrective motion remains substantial, continue direct pairwise training; if both are weak, apply CE warm-up and repeat both checks before switching to pairwise training. Among the five losses studied here, squared-hinge provides the intended severity-increasing and easy-pair-suppressing weight profile; other surrogates with the same profile can be used analogously. These diagnostics are interpreted along the observed training trajectory rather than through a universal numerical threshold.

\paragraph{Fixed-checkpoint batch-size variance scaling.}
\label{app:fixed-checkpoint-variance-scaling}
We measure fixed-checkpoint batch-size scaling on APTOS, SNLI, and ANLI-R1. At each checkpoint, the empirical severe-error direction \(\widehat{\bm v}_{\rm sev}\) is constructed by Eq.~\eqref{eq:app-empirical-severe-direction} from a fixed reference pool of \(64\) positive and \(64\) negative examples and is then held fixed across all batch-size measurements at that checkpoint. A disjoint logging pool contains \(128\) examples per class. Matched checkpoints use the same sample indices but compute their own checkpoint-specific directions.

Severe pairs satisfy \(f_\theta(x^+)-f_\theta(x^-)\le -0.05\) on the sigmoid-score difference scale used by the mechanism logs. For each per-class batch size \(B\in\{4,8,16,32,64\}\), we use \(5{,}000\) with-replacement draws from the logging pool and float64 aggregation to estimate the variance of the severe-truncated pairwise mini-batch update projected onto \(\widehat{\bm v}_{\rm sev}\). We fit the log empirical variance against \(\log B\). Each setting is repeated over three independent reference/logging-pool splits, and Table~\ref{tab:app-fixed-checkpoint-variance-scaling} reports the ranges of severe projected drift and fitted slope across the three splits. Slope stability is checked with \(1{,}000\) bootstrap repetitions. Good and bad regions are identified by severe projected drift before the slopes are fitted.

For equal positive and negative batch sizes, the measured projected variance uses the decomposition in Eq.~\eqref{eq:app-var-split}, written as
\[
\Var(\widehat g_B)
=
\frac{V_1}{B}
+
\frac{V_2}{B^2},
\]
where \(V_1\) is the sum of the first-order row and column Hoeffding projection variances and \(V_2\) is the residual pair-interaction variance.

\begin{table}[!htbp]
\centering
\caption{Fixed-checkpoint batch-size scaling of the severe-truncated projected variance. Good and bad regions are determined by severe projected drift. Entries report ranges over three independent reference/logging-pool splits.}
\label{tab:app-fixed-checkpoint-variance-scaling}
\small
\setlength{\tabcolsep}{4pt}
\begin{tabularx}{\linewidth}{>{\raggedright\arraybackslash}X l c c}
\toprule
Task / surrogate & Region & \shortstack{Severe projected\\drift} & \shortstack{Log-variance--log-\(B\)\\slope}\\
\midrule
APTOS / direct square & Good & \(+0.273\) to \(+1.043\) & \(-1.102\) to \(-1.070\)\\
APTOS / direct square & Bad & \(-0.187\) to \(+0.105\) & \(-1.250\) to \(-1.185\)\\
APTOS / direct exponential & Good & \(+0.574\) to \(+1.304\) & \(-1.099\) to \(-1.081\)\\
APTOS / direct exponential & Bad & \(-0.104\) to \(+0.164\) & \(-1.470\) to \(-1.351\)\\
APTOS / direct logistic & Good & \(+0.552\) to \(+0.687\) & \(-1.091\) to \(-1.060\)\\
APTOS / direct logistic & Bad & \(-0.003\) to \(+0.016\) & \(-1.562\) to \(-1.412\)\\
SNLI / direct squared-hinge & Good & \(+0.213\) to \(+0.384\) & \(-1.091\) to \(-1.032\)\\
SNLI / direct squared-hinge & Bad & \(-0.0004\) to \(+0.0011\) & \(-1.595\) to \(-1.492\)\\
ANLI-R1 / direct squared-hinge & Good & \(+0.107\) to \(+0.151\) & \(-1.062\) to \(-1.038\)\\
ANLI-R1 / direct squared-hinge & Bad & \(+0.011\) to \(+0.037\) & \(-1.201\) to \(-1.141\)\\
\bottomrule
\end{tabularx}
\end{table}

All variance calculations record the corresponding float64 numerical floor. The smallest measured severe projected variance remains \(1.39\times10^9\) times its corresponding numerical floor. As a separate numerical control, applying the same scaling procedure to an exact-degeneracy construction satisfying Corollary~\ref{cor:app-fully-severe-compressed-collapse} gives a log-variance--log-\(B\) slope of approximately \(-2.000\); this controlled result is not included in Table~\ref{tab:app-fixed-checkpoint-variance-scaling}.

\begin{figure}[!htbp]
\centering
\includegraphics[width=\linewidth]{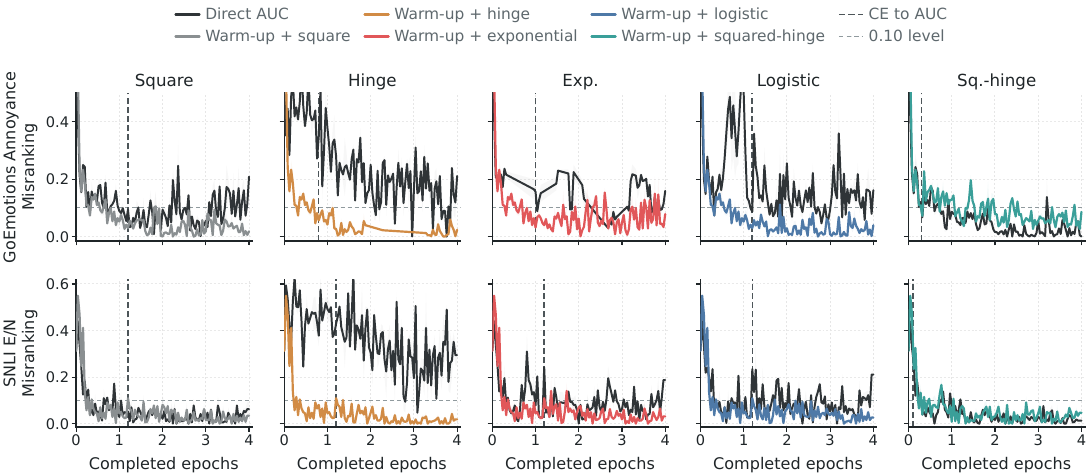}
\caption{Misranking trajectories for all five pairwise losses on DistilBERT GoEmotions Annoyance and SNLI E/N. Black curves are direct pairwise training; colored curves are CE warm-up followed by AUC training, with loss colors shown in the legend. Vertical dashed lines mark the CE-to-AUC switch.}
\label{fig:app-all-loss-trajectories}
\end{figure}

Figure~\ref{fig:app-all-loss-trajectories} extends the trajectory check in Figure~\ref{fig:experiment-trajectory-diagnostics}(c,d). On GoEmotions Annoyance and SNLI E/N, direct training remains visibly separated from warm-started training for square, hinge, exponential, and logistic losses, while squared-hinge shows much smaller separation. This supports the use of the main-text trajectory panels as representative examples rather than isolated cases.

\begin{figure}[!htbp]
\centering
\includegraphics[width=\linewidth]{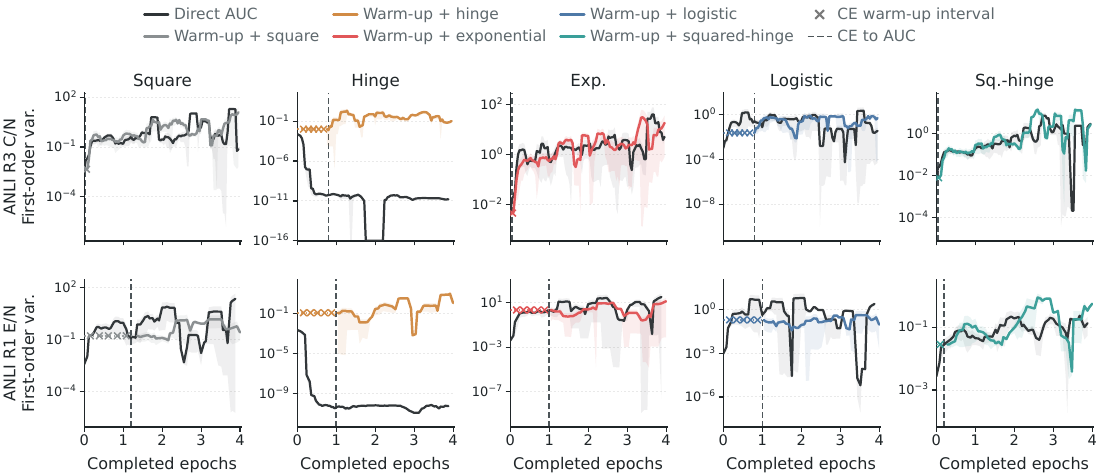}
\caption{Scaled first-order Hoeffding variance trajectories for all five pairwise losses on DistilBERT ANLI R3 C/N and ANLI R1 E/N. Black curves are direct pairwise training; colored curves are CE warm-up followed by AUC training, with loss colors shown in the legend. Shaded bands show one standard error across matching runs after the same trajectory smoothing used for the displayed mean. \(x\) marks denote the CE warm-up interval before the pairwise variance term is recorded, and vertical dashed lines mark the CE-to-AUC switch.}
\label{fig:app-first-order-variance-five-loss}
\end{figure}

Figure~\ref{fig:app-first-order-variance-five-loss} complements Figure~\ref{fig:app-all-loss-trajectories} by plotting the corresponding first-order Hoeffding variance term rather than the misranking ratio. The panels use the same correction-direction construction as Figure~\ref{fig:experiment-trajectory-diagnostics}(a,b), but show all five pairwise losses on two representative NLP tasks.

\FloatBarrier
\paragraph{Matched directional intervention.}
\label{app:matched-directional-intervention}
We run matched interventions from fixed direct-hinge checkpoints on ImageNet-LT dog-tail versus siblings and APTOS severe-plus versus moderate. At each checkpoint, the empirical severe-pair normal is computed from a fixed reference pool of \(64\) positive and \(64\) negative training examples with \(M=0\), and \(\widehat{\bm v}_{\rm sev}\) is obtained by Eq.~\eqref{eq:app-empirical-severe-direction} and then held fixed throughout the intervention. A disjoint pool of \(128\) examples per class is retained for mechanism logging. All arms use training/data seed \(123\) and pool seed \(20260725\), and share the same checkpoint, data order, augmentations, optimizer, and learning-rate schedule.

For ImageNet-LT dog-tail versus siblings, the intervention starts from update \(690\), uses learning rate \(0.025\) and per-update parameter \(L_2\) magnitude \(\sigma=0.05\), and continues for three epochs (\(207\) optimizer updates) over \(16\) noise seeds \(20262000\)--\(20262015\). For APTOS severe-plus versus moderate, it starts from update \(110\), uses learning rate \(5\times10^{-5}\) and \(\sigma=0.01\), and continues for three epochs (\(66\) optimizer updates) over \(24\) noise seeds \(20263000\)--\(20263023\). All \(40\) realizations are retained.

After each optimizer step, the directional arm adds \(\epsilon_t\sigma\widehat{\bm v}_{\rm sev}\), where the \(\epsilon_t\in\{-1,+1\}\) are independent Rademacher signs. The matched isotropic arm adds an independent normalized Gaussian direction with the same per-update \(L_2\) magnitude, refreshed at every update. One deterministic no-noise run provides the shared baseline for each task. Endpoint effects are measured at the end of the third intervention epoch on the test set. Let \(\Delta\mathrm{AUC}\) and \(\Delta S_{\rm sev}\) denote the changes in test AUC and test severe-misranking mass \(S_{\rm sev}=\mathbb{P}\{m_\theta(X^+,X^-)\le 0\}\) relative to this baseline. A realization is classified as rescue when \(\Delta\mathrm{AUC}\ge 0.5\) and \(\Delta S_{\rm sev}\le -0.5\) percentage points, as harm when \(\Delta\mathrm{AUC}\le -0.5\) and \(\Delta S_{\rm sev}\ge 0.5\) percentage points, and as no clear effect otherwise. The same rule is applied to the directional and isotropic arms.

\begin{table}[!ht]
\centering
\caption{Matched directional-intervention data at fixed direct-hinge checkpoints. R/H/N denotes rescue/harm/no clear effect. Endpoint changes are measured at the final intervention epoch relative to the shared no-noise baseline and are reported in percentage points; D/I denotes directional/isotropic.}
\label{tab:app-directional-intervention}
\small
\setlength{\tabcolsep}{4pt}
\begin{tabular}{lccccc}
\toprule
Task
& \(n\)
& \shortstack{Directional\\R/H/N}
& \shortstack{Isotropic\\R/H/N}
& \shortstack{Mean \(\Delta\)AUC\\D/I}
& \shortstack{Mean \(\Delta S_{\rm sev}\)\\D/I} \\
\midrule
ImageNet-LT dog-tail
& 16
& 2/6/8
& 0/0/16
& \(-2.541/-0.001\)
& \(+6.036/-0.046\) \\
APTOS
& 24
& 6/13/5
& 1/1/22
& \(-1.316/+0.035\)
& \(+1.317/-0.033\) \\
\midrule
Total
& 40
& 8/19/13
& 1/1/38
& \textemdash
& \textemdash \\
\bottomrule
\end{tabular}
\end{table}

\FloatBarrier
\subsection{Multiclass, Multilabel, and Deep-AUC Extensions}
\label{app:extended-auc-experiments}

\paragraph{Small-scale multiclass experiments.}
We evaluate CE, direct hinge, warm-up hinge, direct squared-hinge, and warm-up squared-hinge on three long-tailed multiclass tasks. Aircraft Airbus-6 LT contains the A318, A319, A320, A321, A330-200, and A330-300 classes, with \(93/200/200\) training/validation/test examples. CIFAR-100 Medium-Mammals-5 LT contains camel, cattle, chimpanzee, elephant, and kangaroo, with \(431/250/250\) examples. CIFAR-100 Vehicles-10 LT contains bicycle, bus, motorcycle, pickup truck, train, lawn mower, rocket, streetcar, tank, and tractor, with \(818/500/500\) examples.

The experiments use a randomly initialized ResNet18 trained for \(10\) epochs, with \(24\) batches per epoch, \(64\times64\) inputs, learning rate \(0.002\), and weight decay \(10^{-4}\). Pairwise training uses a macro one-vs-rest objective over class logits, and checkpoints are selected by validation Macro-OVO-AUC.

\begin{table}[!htbp]
\caption{Results on three long-tailed multiclass tasks (AUC\%).}
\label{tab:app-small-multiclass-results}
\centering
\small
\setlength{\tabcolsep}{4pt}
\begin{tabularx}{\textwidth}{Xccccc}
\toprule
{\bfseries Multiclass task} & {\bfseries CE} & \shortstack{{\bfseries Direct}\\{\bfseries hinge}} & \shortstack{{\bfseries Warm-up}\\{\bfseries hinge}} & \shortstack{{\bfseries Direct}\\{\bfseries squared-hinge}} & \shortstack{{\bfseries Warm-up}\\{\bfseries squared-hinge}}\\
\midrule
Aircraft Airbus-6 LT & 60.92 & 61.43 & 61.61 & {\bfseries 61.68} & 60.91\\
CIFAR-100 Medium-Mammals-5 LT & 81.08 & 81.13 & 82.91 & 83.62 & {\bfseries 84.44}\\
CIFAR-100 Vehicles-10 LT & 86.47 & 88.30 & 88.64 & 87.97 & {\bfseries 89.08}\\
\bottomrule
\end{tabularx}
\end{table}

\paragraph{Multilabel experiments.}
We also evaluate multilabel AUC optimization on CheXpert. Competition-5 uses Cardiomegaly, Edema, Consolidation, Atelectasis, and Pleural Effusion; Cardiac/Fluid-4 uses Cardiomegaly, Edema, Atelectasis, and Pleural Effusion; and Consolidation/Cardiac-2 uses Consolidation and Cardiomegaly. Each image may have multiple positive labels.

The experiments use an ImageNet-pretrained DenseNet121 with a randomly initialized multilabel head. Under the U-Zero rule, uncertain labels (\(-1\)) and missing labels are mapped to zero. Patient-disjoint training and internal-validation splits are constructed from \texttt{train.csv}, and the official \texttt{valid.csv} is used only as the test set. The training and internal-validation sets contain at most \(12{,}000\) and \(3{,}000\) examples, respectively. Models are trained for \(8\) epochs with batch size \(64\), \(224\times224\) inputs, learning rate \(10^{-4}\), and weight decay \(10^{-4}\).

The pairwise loss is computed separately between positive and negative examples for each label and macro-averaged over labels that contain both classes in the current batch. Checkpoints are selected by internal-validation macro ROC-AUC, and the final metric is the macro-average of label-wise test ROC-AUC.

\begin{table}[!htbp]
\caption{Results on three CheXpert multilabel tasks (AUC\%).}
\label{tab:app-multilabel-results}
\centering
\small
\setlength{\tabcolsep}{4pt}
\begin{tabularx}{\textwidth}{Xccccc}
\toprule
{\bfseries Multilabel task} & {\bfseries BCE} & \shortstack{{\bfseries Direct}\\{\bfseries hinge}} & \shortstack{{\bfseries Warm-up}\\{\bfseries hinge}} & \shortstack{{\bfseries Direct}\\{\bfseries squared-hinge}} & \shortstack{{\bfseries Warm-up}\\{\bfseries squared-hinge}}\\
\midrule
Competition-5 & 80.86 & 82.75 & {\bfseries 83.40} & 81.51 & 83.09\\
Cardiac/Fluid-4 & 80.76 & 81.14 & {\bfseries 82.13} & 79.96 & 81.44\\
Consolidation/Cardiac-2 & 51.52 & 57.28 & 56.09 & 57.08 & {\bfseries 59.69}\\
\bottomrule
\end{tabularx}
\end{table}

\paragraph{Large-scale multiclass experiments.}
We further evaluate large-scale multiclass learning on iNaturalist 2018 and ImageNet-LT, retaining their natural long-tailed training distributions without class resampling. iNaturalist 2018 contains \(437{,}513\) training examples and \(24{,}426\) official validation examples, for \(461{,}939\) samples over \(8{,}142\) classes. ImageNet-LT contains \(115{,}845\) valid training examples, \(20{,}000\) validation examples, and \(50{,}000\) test examples, for \(185{,}845\) samples over \(1{,}000\) classes.

The iNaturalist 2018 experiment uses an ImageNet-supervised pretrained ResNet50 with a randomly initialized \(8{,}142\)-class head. It is trained for \(3\) epochs with training and evaluation batch size \(128\) and \(128\times128\) inputs; the official validation set is evaluated once after the fixed training budget. The ImageNet-LT experiment uses a Places365-supervised pretrained ResNet50~\citep{7968387} with a randomly initialized classification head. It is trained for \(4\) epochs with training batch size \(128\), evaluation batch size \(256\), and \(128\times128\) inputs; checkpoints are selected by validation Macro-OVO-AUC, and the test set is evaluated only after the selected checkpoint is restored.

We compare CE, direct hinge, warm-up hinge, direct squared-hinge, and warm-up squared-hinge, and report Macro-OVO-AUC and Macro-OVR-AUC.

\begin{table}[!htbp]
\caption{Large-scale multiclass results on iNaturalist 2018 and ImageNet-LT (AUC\%).}
\label{tab:app-large-multiclass-results}
\centering
\small
\setlength{\tabcolsep}{4pt}
\begin{tabular}{llccccc}
\toprule
{\bfseries Dataset} & {\bfseries Metric} & {\bfseries CE} & \shortstack{{\bfseries Hinge}\\{\bfseries direct}} & \shortstack{{\bfseries Hinge}\\{\bfseries warm-up}} & \shortstack{{\bfseries Sq.-hinge}\\{\bfseries direct}} & \shortstack{{\bfseries Sq.-hinge}\\{\bfseries warm-up}}\\
\midrule
iNaturalist & Macro-OVO-AUC & 96.64 & 92.45 & 99.70 & 89.18 & {\bfseries 99.74}\\
ImageNet-LT & Macro-OVO-AUC & 98.50 & 97.87 & 98.45 & 97.69 & {\bfseries 98.91}\\
iNaturalist & Macro-OVR-AUC & 97.70 & 88.72 & 97.31 & 79.94 & {\bfseries 98.13}\\
ImageNet-LT & Macro-OVR-AUC & 94.80 & 93.21 & 94.81 & 92.95 & {\bfseries 96.08}\\
\bottomrule
\end{tabular}
\end{table}

\paragraph{Deep-AUC comparison.}
We compare the results in this paper with Deep-AUC as implemented by LibAUC~\citep{yuan2023libauc} using \texttt{AUCMLoss(v2)} and the PESG optimizer. \texttt{AUCMLoss(v2)} uses the squared AUC-margin objective with margin \(1.0\); PESG uses SGD mode, momentum \(0.9\), epoch decay \(0.002\), weight decay \(10^{-5}\), and positive fraction \(0.25\).

The squared AUC-margin objective can be reconstructed with auxiliary variables in a min--max form, avoiding explicit construction of all positive--negative pairs at each update. This reconstruction uses class-dependent first- and second-order statistics. In contrast, hinge, logistic, exponential, and squared-hinge assign weights that depend on each pair's margin and generally cannot be reconstructed from fixed class means and second moments alone.

More explicitly, for a target margin \(c\), the quadratic objective
\[
\mathbb{E}\!\left[(c-f_\theta(X^+)+f_\theta(X^-))^2\right]
\]
expands exactly into the positive- and negative-class score means, second moments, and their interaction. Auxiliary primal and dual variables represent these quantities as an instance-wise min--max problem, thereby removing explicit pair construction and the corresponding two-sample pairwise estimator for this quadratic objective. By contrast, the margin-dependent weights of hinge, logistic, exponential, and squared-hinge are not generally determined by fixed class means and second moments; reproducing them without explicit pairs would require a different loss-specific reconstruction.

\begin{table}[!htbp]
\caption{Deep-AUC and our results on \(17\) task--model settings (AUC\%).}
\label{tab:app-deep-auc-comparison}
\centering
\small
\setlength{\tabcolsep}{3pt}
\begin{tabularx}{\textwidth}{>{\raggedright\arraybackslash}X>{\centering\arraybackslash}p{0.16\textwidth}>{\centering\arraybackslash}p{0.16\textwidth}}
\toprule
{\bfseries Task--model setting} & {\bfseries Deep-AUC} & {\bfseries Ours}\\
\midrule
Apple Disease--DenseNet121 & 55.56 & {\bfseries 62.69}\\
Rice Disease--DenseNet121 & {\bfseries 77.76} & 75.71\\
APTOS Sev./Mod.--DenseNet121 & 82.12 & {\bfseries 84.20}\\
HAM10000 M/N--DenseNet121 & 97.16 & {\bfseries 97.70}\\
ImageNet-LT Dogs--ResNet50 & 94.56 & {\bfseries 95.83}\\
Dog Breeds--ResNet50 & 85.74 & {\bfseries 87.00}\\
CRJ700/CRJ900--ResNet50 & 53.52 & {\bfseries 56.62}\\
APTOS Sev./Mod.--ResNet50 & 82.40 & {\bfseries 83.92}\\
DeepWeeds P/S--ResNet50 & 98.43 & {\bfseries 98.72}\\
GoEmotions Annoyance--DistilBERT & {\bfseries 84.51} & 84.38\\
GoEmotions Fear--DistilBERT & 96.11 & {\bfseries 96.48}\\
ANLI R1 E/N--DistilBERT & {\bfseries 71.67} & 70.21\\
ANLI R3 C/N--DistilBERT & {\bfseries 62.63} & 62.47\\
ANLI R3 E/N--DistilBERT & 63.65 & {\bfseries 63.87}\\
SNLI E/N--DistilBERT & 92.94 & {\bfseries 92.98}\\
WANLI C/N--DistilBERT & 82.38 & {\bfseries 82.41}\\
WANLI E/N--DistilBERT & 75.40 & {\bfseries 75.44}\\
\bottomrule
\end{tabularx}
\end{table}

\FloatBarrier

Deep-AUC is stronger in \(4\) of the \(17\) task--model settings, while our validation-selected direct configuration is stronger in \(13\). The comparison identifies Deep-AUC as a strong scalable solution when its quadratic reconstruction is appropriate, while showing that direct pairwise optimization remains competitive and supports a broader range of surrogate dynamics.

\subsection{Assets, Licenses, and Terms}
\label{app:assets-licenses}

The experiments use existing public datasets, pretrained models, and standard software assets. Dataset sources and task constructions are listed in Appendix~\ref{app:datasets-tasks} and Appendix~\ref{app:extended-auc-experiments}. The CV experiments use ImageNet-LT split files derived from ImageNet, Stanford Dogs, FGVC-Aircraft, CIFAR-100, iNaturalist 2018, CheXpert, APTOS 2019, DeepWeeds, D-KAP, Rice Leaf Diseases, and HAM10000. The NLP experiments use the Hugging Face releases or prepared splits of GoEmotions, ANLI, SNLI, and WANLI. The model assets are randomly initialized ResNet18, ImageNet- or Places365-supervised pretrained ResNet50, ImageNet-pretrained DenseNet121, DistilBERT-base, and Qwen2.5-1.5B-Instruct; the DistilBERT and Qwen NLP runs both use LoRA as the parameter-efficient fine-tuning method. The Deep-AUC comparison uses LibAUC \texttt{AUCMLoss(v2)} and PESG.

The asset licenses and terms are as follows. ImageNet/ImageNet-LT and Stanford Dogs are used under ImageNet-style non-commercial research and educational terms; FGVC-Aircraft is non-commercial research only, with original photographer copyright retained. CIFAR-100 is distributed through its official project page and used subject to its source terms; iNaturalist 2018 is used subject to the competition data terms and the licenses attached to its images; and CheXpert is obtained through Stanford's official release channel and used under its data-use terms. APTOS 2019 \(224\times224\) is CC0 1.0; DeepWeeds is CC BY 4.0; D-KAP is listed as ``Database: Open Database, Contents: Database Contents''; the Rice Leaf Diseases release used here does not list a license; and HAM10000 is CC BY-NC-SA 4.0 in the Kaggle release used here. GoEmotions is Apache 2.0, ANLI is CC BY-NC 4.0, SNLI is CC BY-SA 4.0, and WANLI is CC BY 4.0. DistilBERT-base-uncased and Qwen2.5-1.5B-Instruct are Apache 2.0; TorchVision ImageNet-pretrained weights are used subject to TorchVision notices and ImageNet-derived terms, and Places365-pretrained weights are used subject to their release terms. The software stack uses standard open-source licenses, including Apache 2.0 for Hugging Face Transformers/Datasets/PEFT and BSD-style licenses for PyTorch/TorchVision, NumPy, pandas, and scikit-learn; LibAUC is used under its project license.

Source and use details for these datasets and models appear in Appendix~\ref{app:datasets-tasks}, Appendix~\ref{app:training-details}, and Appendix~\ref{app:extended-auc-experiments}. We used the assets for research experiments and did not redistribute the original datasets, pretrained checkpoints, or derived datasets as part of this manuscript. Users who reproduce the experiments should obtain each dataset and pretrained model from the official source or public release and follow the corresponding license, data-use terms, and model-use terms. Any anonymized code or reproduction package should include its own license file and instructions for separately downloading third-party assets from their official sources.

\subsection{Ethics, Broader Impacts, and Safeguards}
\label{app:ethics-impacts}

This work studies optimization dynamics for AUC training and does not introduce a deployed decision system. The potential positive impact is improved understanding of when pairwise AUC optimization is reliable under class imbalance and ranking-based evaluation. Potential negative impact could arise if AUC-optimized models are deployed in high-stakes domains without suitable validation, calibration, uncertainty assessment, subgroup analysis, or human oversight. The experiments include public medical-image and natural-language datasets, so any downstream use of the trained models should follow the dataset terms and the safeguards required by the target application.

We do not release a new high-risk model, scraped dataset, or human-subject dataset. The manuscript reports experiments on existing public datasets and pretrained models, and it does not involve new data collection, crowdsourcing, interaction with human subjects, or interventions on human participants. Accordingly, no new compensation protocol, participant instruction text, or institution-specific IRB information is part of this work. No additional release safeguard is introduced beyond respecting the original asset terms.

\section{Direct Pairwise SGD as a Stochastic Dynamics}
\label{app:pairwise-sgd-dynamics}

This section fixes the direct pairwise SGD process used in the paper and states its local stochastic-dynamics interface. When each step uses a fresh positive-negative mini-batch, the parameter update is Markov in the current parameter: its drift is the negative population pairwise gradient, and its covariance is the mini-batch covariance later decomposed by Hoeffding projections.

\paragraph{Pairwise setup.}
We use the positive-margin convention
\begin{equation*}
m_\theta(x^+,x^-)=f_\theta(x^+)-f_\theta(x^-).
\end{equation*}
For a pairwise surrogate \(\ell\), define
\begin{equation*}
R_\ell(\theta)=\E[\ell(m_\theta(X^+,X^-))],
\qquad
w_\ell(x)=-\ell'(x),
\end{equation*}
where \(X^+\sim P_+\), \(X^-\sim P_-\), and the two draws are independent. Given \(B_+\) positives and \(B_-\) negatives, let
\begin{equation*}
m_{ij}=f_\theta(x_i^+)-f_\theta(x_j^-),
\qquad
\bm g_i^+=\nabla_\theta f_\theta(x_i^+),
\qquad
\bm g_j^-=\nabla_\theta f_\theta(x_j^-).
\end{equation*}
The direct pairwise mini-batch gradient is
\begin{equation}
\widehat{\bm g}_{B_+,B_-}^\ell(\theta)
=
\frac{1}{B_+B_-}
\sum_{i=1}^{B_+}\sum_{j=1}^{B_-}
w_\ell(m_{ij})\cdot(\bm g_j^- - \bm g_i^+).
\label{eq:app-pairwise-gradient-general}
\end{equation}
When \(B_+=B_-=B\), we write \(\widehat{\bm g}_{B}^\ell(\theta)=\widehat{\bm g}_{B,B}^\ell(\theta)\), so the equal-size specialization is
\begin{equation}
\widehat{\bm g}_B^\ell(\theta)
=
\frac{1}{B^2}
\sum_{i=1}^B\sum_{j=1}^B
w_\ell(m_{ij})\cdot(\bm g_j^- - \bm g_i^+).
\label{eq:app-pairwise-gradient}
\end{equation}

For local uses, we follow the localization convention in \citet[Remark 11]{JMLR:v20:17-526}: when the smoothness and Lipschitz assumptions hold only on a neighborhood, the analysis is applied to the stopped process until it exits that neighborhood, or to a globally regularized objective that agrees with the original objective on the neighborhood. The drift and covariance identities below are unchanged before the exit time.

\begin{theorem}[Markov-state finite-time weak SGD diffusion bridge]
\label{thm:app-markov-state-weak-bridge}
Let \((\Gamma_k)_{k\ge0}\) be i.i.d. random variables, let \(F(\cdot,\Gamma_k)\) be random objectives on \(\R^q\), and let the Markov state \(\bm y_k\in\R^q\), initialized at \(\bm y_0=\bm y_{\mathrm{init}}\), evolve by constant-step SGD
\begin{equation*}
\bm y_{k+1}
=
\bm y_k-\eta\cdot\nabla F(\bm y_k,\Gamma_k),
\qquad
k=0,\ldots,N-1,
\quad
N=\lfloor T/\eta\rfloor .
\end{equation*}
Assume the random objective family satisfies the hypotheses of \citet[Corollary 10]{JMLR:v20:17-526} with \(f_\gamma=F(\cdot,\Gamma_0)\) and \(f=\bar F\), where
\begin{equation*}
\bar F(\bm y)=\E[F(\bm y,\Gamma_0)].
\end{equation*}
Define
\begin{equation*}
\bm A(\bm y)
=
\Cov(\nabla F(\bm y,\Gamma_0)).
\end{equation*}
Then the It\^o SDE
\begin{equation*}
d\widetilde{\bm y}_t
=
-\nabla \bar F(\widetilde{\bm y}_t)\,dt
+
(\eta\cdot\bm A(\widetilde{\bm y}_t))^{1/2}\,d\bm W_t,
\qquad
\widetilde{\bm y}_0=\bm y_{\mathrm{init}},
\end{equation*}
where \(\bm W_t\) is a standard Brownian motion in \(\R^q\), is an order-\(1\) finite-time weak approximation to the discrete dynamics: for every test function \(g\) in the \(G_2\) class of \citet[Corollary 10]{JMLR:v20:17-526}, there is a constant \(C_{T,g}\) independent of \(\eta\) such that
\begin{equation*}
\max_{0\le k\le N}
\left|
\E[g(\bm y_k)]
-
\E[g(\widetilde{\bm y}_{k\eta})]
\right|
\le
C_{T,g}\cdot\eta.
\end{equation*}
\end{theorem}

\begin{proof}
The state process is Markov by assumption: the conditional law of \(\bm y_{k+1}\) given the past is determined by the current state \(\bm y_k\) and a fresh noise variable \(\Gamma_k\). Thus the update has the constant-step stochastic-gradient form required by \citet[Corollary 10]{JMLR:v20:17-526}:
\begin{equation*}
\bm y_{k+1}
=
\bm y_k-\eta\cdot\nabla F(\bm y_k,\Gamma_k).
\end{equation*}
The mean objective is
\begin{equation*}
\bar F(\bm y)=\E[F(\bm y,\Gamma_0)].
\end{equation*}
Under the exchange-of-derivative and regularity conditions included in \citet[Corollary 10]{JMLR:v20:17-526},
\begin{equation*}
\E[\nabla F(\bm y,\Gamma_0)]
=
\nabla \bar F(\bm y).
\end{equation*}
Hence the conditional first moment of one step is
\begin{equation*}
\E[\bm y_{k+1}-\bm y_k\mid \bm y_k=\bm y]
=
-\eta\cdot\nabla \bar F(\bm y),
\end{equation*}
so the continuous-time drift is \(-\nabla\bar F(\bm y)\). The stochastic-gradient covariance is
\begin{equation*}
\Cov(\nabla F(\bm y,\Gamma_k)\mid \bm y_k=\bm y)
=
\Cov(\nabla F(\bm y,\Gamma_0))
=
\bm A(\bm y),
\end{equation*}
where the second equality uses the identical distribution of the noise variables. Therefore the finite-time weak modified equation has diffusion coefficient
\begin{equation*}
(\eta\cdot\bm A(\bm y))^{1/2}.
\end{equation*}
Applying \citet[Corollary 10]{JMLR:v20:17-526} with \(f_\gamma=F(\cdot,\Gamma_0)\), mean objective \(\bar F\), and covariance field \(\bm A\) gives the displayed SDE and the weak-error bound for every \(G_2\) test function \(g\).
\end{proof}

\begin{proposition}[Plain pairwise mini-batch gradients are unbiased]
\label{prop:app-pairwise-batch-unbiased}
Assume \(\nabla_\theta \ell(m_\theta(X^+,X^-))\) is integrable and differentiation can be exchanged with expectation. Then the direct pairwise mini-batch gradient in Eq.~\eqref{eq:app-pairwise-gradient-general} satisfies
\begin{equation*}
\E[\widehat{\bm g}_{B_+,B_-}^\ell(\theta)]
=
\nabla R_\ell(\theta).
\end{equation*}
The equal-size shorthand \(\widehat{\bm g}_{B}^\ell\) satisfies the same identity.
\end{proposition}

\begin{proof}
For one independent positive-negative pair,
\begin{equation*}
\nabla_\theta \ell(m_\theta(X^+,X^-))
=
\ell'(m_\theta(X^+,X^-))
\cdot
\bigl(\nabla f_\theta(X^+)-\nabla f_\theta(X^-)\bigr).
\end{equation*}
Since \(w_\ell=-\ell'\), the same vector can be written as
\begin{equation*}
w_\ell(m_\theta(X^+,X^-))
\cdot
\bigl(\nabla f_\theta(X^-)-\nabla f_\theta(X^+)\bigr).
\end{equation*}
The exchange assumption gives
\begin{equation*}
\nabla R_\ell(\theta)
=
\E[\nabla_\theta \ell(m_\theta(X^+,X^-))].
\end{equation*}
Every pair \((X_i^+,X_j^-)\) inside the mini-batch has the same product distribution as \((X^+,X^-)\), although the \(B_+B_-\) pair terms are not mutually independent. Therefore linearity of expectation yields
\begin{equation*}
\E[\widehat{\bm g}_{B_+,B_-}^\ell(\theta)]
=
\frac{1}{B_+B_-}
\sum_{i=1}^{B_+}\sum_{j=1}^{B_-}
\E[
w_\ell(m_\theta(X_i^+,X_j^-))
\cdot
(\nabla f_\theta(X_j^-)-\nabla f_\theta(X_i^+))
]
=
\nabla R_\ell(\theta).
\end{equation*}
\end{proof}

\begin{proposition}[One-step conditional moments for pairwise dynamics]
\label{prop:app-pairwise-one-step-moments}
Consider the parameter dynamics
\begin{equation*}
\theta_{k+1}
=
\theta_k-\eta_k\cdot\widehat{\bm g}_{B_+,B_-,k}^\ell(\theta_k),
\end{equation*}
where, conditionally on \(\theta_k\), \(\widehat{\bm g}_{B_+,B_-,k}^\ell(\theta_k)\) is computed from a fresh positive-negative mini-batch. Let
\begin{equation*}
\Delta_k=\theta_{k+1}-\theta_k.
\end{equation*}
Under Proposition~\ref{prop:app-pairwise-batch-unbiased},
\begin{equation*}
\E[\Delta_k\mid\theta_k]
=
-\eta_k\cdot\nabla R_\ell(\theta_k).
\end{equation*}
If the conditional second moment is finite, then
\begin{equation*}
\E[\Delta_k\cdot\Delta_k^\top\mid\theta_k]
=
\eta_k^2\cdot
\left\{
\Cov(\widehat{\bm g}_{B_+,B_-,k}^\ell(\theta_k)\mid\theta_k)
+
\nabla R_\ell(\theta_k)\cdot\nabla R_\ell(\theta_k)^\top
\right\}.
\end{equation*}
\end{proposition}

\begin{proof}
The increment is
\begin{equation*}
\Delta_k=-\eta_k\cdot\widehat{\bm g}_{B_+,B_-,k}^\ell(\theta_k).
\end{equation*}
Because the fresh mini-batch used at step \(k\) is independent of the past, the conditional law of \(\widehat{\bm g}_{B_+,B_-,k}^\ell(\theta_k)\) given the full history depends on the history only through \(\theta_k\). Proposition~\ref{prop:app-pairwise-batch-unbiased} therefore applies conditionally:
\begin{equation*}
\E[\widehat{\bm g}_{B_+,B_-,k}^\ell(\theta_k)\mid\theta_k]
=
\nabla R_\ell(\theta_k).
\end{equation*}
Multiplying by \(-\eta_k\) gives the displayed drift identity.

For the second moment,
\begin{equation*}
\E[\Delta_k\cdot\Delta_k^\top\mid\theta_k]
=
\eta_k^2\cdot
\E[
\widehat{\bm g}_{B_+,B_-,k}^\ell(\theta_k)
\cdot
\widehat{\bm g}_{B_+,B_-,k}^\ell(\theta_k)^\top
\mid\theta_k].
\end{equation*}
For any square-integrable vector \(Y\),
\begin{equation*}
\E[Y\cdot Y^\top]
=
\Cov(Y)+\E[Y]\cdot\E[Y]^\top.
\end{equation*}
Applying this identity conditionally with
\begin{equation*}
Y=\widehat{\bm g}_{B_+,B_-,k}^\ell(\theta_k)
\end{equation*}
and substituting the conditional unbiasedness identity proves the formula.
\end{proof}

\begin{corollary}[Mean-ODE generator for decreasing steps]
\label{cor:app-pairwise-mean-ode}
In Proposition~\ref{prop:app-pairwise-one-step-moments}, suppose \(\eta_k\downarrow0\). For the one-step kernel
\begin{equation*}
(K_k\varphi)(\theta)
=
\E[\varphi(\theta-\eta_k\cdot\widehat{\bm g}_{B_+,B_-}^\ell(\theta))],
\end{equation*}
assume that for every compact \(C\) in a local open set \(U\),
\begin{equation*}
\sup_{\theta\in C}
\E[\|\widehat{\bm g}_{B_+,B_-}^\ell(\theta)\|_2^3]
<\infty.
\end{equation*}
Then for every \(\varphi\in C_b^3(\R^d)\) and every compact \(C\subset U\),
\begin{equation*}
\sup_{\theta\in C}
\left|
\frac{(K_k\varphi)(\theta)-\varphi(\theta)}{\eta_k}
+
\langle\nabla\varphi(\theta),\nabla R_\ell(\theta)\rangle
\right|
\to0.
\end{equation*}
\end{corollary}

\begin{proof}
Taylor expand around \(\theta\). For \(G=\widehat{\bm g}_{B_+,B_-}^\ell(\theta)\),
\begin{equation*}
\varphi(\theta-\eta_k\cdot G)
=
\varphi(\theta)
-\eta_k\cdot\langle\nabla\varphi(\theta),G\rangle
+\frac{\eta_k^2}{2}\cdot
G^\top\cdot\nabla^2\varphi(\theta)\cdot G
+\mathcal R_k(\theta,G),
\end{equation*}
where, uniformly over \(\theta\in C\),
\begin{equation*}
|\mathcal R_k(\theta,G)|
\le
C_\varphi\cdot\eta_k^3\cdot\|G\|_2^3
\end{equation*}
because \(\varphi\in C_b^3(\R^d)\). Taking expectation and dividing by \(\eta_k\) gives
\begin{align*}
\frac{(K_k\varphi)(\theta)-\varphi(\theta)}{\eta_k}
&=
-\langle\nabla\varphi(\theta),\E[\widehat{\bm g}_{B_+,B_-}^\ell(\theta)]\rangle \\
&\quad+
\frac{\eta_k}{2}\cdot
\E[
\widehat{\bm g}_{B_+,B_-}^\ell(\theta)^\top
\cdot
\nabla^2\varphi(\theta)
\cdot
\widehat{\bm g}_{B_+,B_-}^\ell(\theta)
] \\
&\quad+
O\!\left(
\eta_k^2\cdot
\E[\|\widehat{\bm g}_{B_+,B_-}^\ell(\theta)\|_2^3]
\right).
\end{align*}
The first term equals
\begin{equation*}
-\langle\nabla\varphi(\theta),\nabla R_\ell(\theta)\rangle
\end{equation*}
by Proposition~\ref{prop:app-pairwise-batch-unbiased}. On the compact set \(C\),
\begin{equation*}
\left|
\frac{\eta_k}{2}\cdot
\E[
\widehat{\bm g}_{B_+,B_-}^\ell(\theta)^\top
\cdot
\nabla^2\varphi(\theta)
\cdot
\widehat{\bm g}_{B_+,B_-}^\ell(\theta)
]
\right|
\le
\frac{\eta_k}{2}\cdot
\|\nabla^2\varphi\|_{\infty,C}
\cdot
\E[\|\widehat{\bm g}_{B_+,B_-}^\ell(\theta)\|_2^2].
\end{equation*}
Since \(\|u\|_2^2\le1+\|u\|_2^3\), the uniform third-moment assumption makes this term converge to \(0\) uniformly on \(C\). The remainder term also converges to \(0\) uniformly on \(C\). This proves the displayed convergence.
\end{proof}

\begin{corollary}[Pairwise AUC SGD Langevin interface]
\label{cor:app-pairwise-langevin-interface}
Consider constant-step direct pairwise AUC SGD
\begin{equation*}
\theta_{k+1}
=
\theta_k-\eta\cdot\widehat{\bm g}_{B_+,B_-,k}^\ell(\theta_k),
\end{equation*}
with fresh pairwise mini-batches at each step. Define
\begin{equation*}
\bm\Sigma_{B_+,B_-,\ell}(\theta)
=
\Cov(\widehat{\bm g}_{B_+,B_-}^\ell(\theta)).
\end{equation*}
For a unit direction \(\bm v\), define the scalar kernel
\begin{equation*}
\psi_{\theta,\bm v}(x^+,x^-)
=
\left\langle
\ell'(m_\theta(x^+,x^-))
\cdot
\bigl(\nabla f_\theta(x^+)-\nabla f_\theta(x^-)\bigr),
\bm v
\right\rangle .
\end{equation*}
Let \(\psi_{\theta,\bm v}^{(1,0)}\), \(\psi_{\theta,\bm v}^{(0,1)}\), and \(\psi_{\theta,\bm v}^{(1,1)}\) denote the positive-class first-order projection, the negative-class first-order projection, and the residual pair-interaction component of this kernel, as defined explicitly in Theorem~\ref{thm:app-hoeffding}.
Under the usual local smoothness and moment assumptions for stochastic-approximation diffusion approximations~\citep{ethier2009markov,kushner2003stochastic,JMLR:v20:17-526,fontaine2021convergence}, the local Langevin interface for the pairwise dynamics is
\begin{equation*}
d\bm\Theta_t
=
-\nabla R_\ell(\bm\Theta_t)\,dt
+
(\eta\cdot\bm\Sigma_{B_+,B_-,\ell}(\bm\Theta_t))^{1/2}\,d\bm W_t.
\end{equation*}
Moreover, along every unit direction \(\bm v\),
\begin{equation*}
\bm v^\top\cdot\bm\Sigma_{B_+,B_-,\ell}(\theta)\cdot\bm v
=
\frac{
\Var(\psi_{\theta,\bm v}^{(1,0)}(X^+))
}{B_+}
+
\frac{
\Var(\psi_{\theta,\bm v}^{(0,1)}(X^-))
}{B_-}
+
\frac{
\Var(\psi_{\theta,\bm v}^{(1,1)}(X^+,X^-))
}{B_+B_-}.
\end{equation*}
In the equal-size case \(B_+=B_-=B\), this reduces to the \(B^{-1}\) versus \(B^{-2}\) split used below.
\end{corollary}

\begin{proof}
Let
\begin{equation*}
\Gamma
=
(X_1^+,\ldots,X_{B_+}^+,X_1^-,\ldots,X_{B_-}^-)
\end{equation*}
denote one full positive-negative mini-batch, and define the random batch objective
\begin{equation*}
F_{B_+,B_-}^\ell(\theta,\Gamma)
=
\frac{1}{B_+B_-}
\sum_{i=1}^{B_+}\sum_{j=1}^{B_-}
\ell(m_\theta(X_i^+,X_j^-)).
\end{equation*}
Its gradient is exactly the direct pairwise mini-batch gradient because
\begin{equation*}
\nabla_\theta \ell(m_\theta(X_i^+,X_j^-))
=
\ell'(m_{ij})\cdot(\bm g_i^+-\bm g_j^-)
=
w_\ell(m_{ij})\cdot(\bm g_j^- - \bm g_i^+).
\end{equation*}
Thus
\begin{equation*}
\nabla F_{B_+,B_-}^\ell(\theta,\Gamma)
=
\widehat{\bm g}_{B_+,B_-}^\ell(\theta).
\end{equation*}
The population objective associated with \(F_{B_+,B_-}^\ell\) is
\begin{equation*}
\E[F_{B_+,B_-}^\ell(\theta,\Gamma)]
=
R_\ell(\theta),
\end{equation*}
because each pair \((X_i^+,X_j^-)\) has the same product distribution as \((X^+,X^-)\). Theorem~\ref{thm:app-markov-state-weak-bridge}, applied with Markov state \(\bm y_k=\theta_k\), gives the local Langevin interface. Equivalently, the first two conditional moments in Proposition~\ref{prop:app-pairwise-one-step-moments} identify the local drift as
\begin{equation*}
-\nabla R_\ell(\theta)
\end{equation*}
and the local diffusion covariance as
\begin{equation*}
\eta\cdot\bm\Sigma_{B_+,B_-,\ell}(\theta),
\end{equation*}
under the stated local regularity and moment assumptions.

For the directional coefficient, use
\begin{equation*}
\bm v^\top\cdot\bm\Sigma_{B_+,B_-,\ell}(\theta)\cdot\bm v
=
\Var(\langle\widehat{\bm g}_{B_+,B_-}^\ell(\theta),\bm v\rangle).
\end{equation*}
The exact variance split in Corollary~\ref{cor:app-unequal-batch-hoeffding} gives
\begin{equation*}
\Var(\langle\widehat{\bm g}_{B_+,B_-}^\ell(\theta),\bm v\rangle)
=
\frac{
\Var(\psi_{\theta,\bm v}^{(1,0)}(X^+))
}{B_+}
+
\frac{
\Var(\psi_{\theta,\bm v}^{(0,1)}(X^-))
}{B_-}
+
\frac{
\Var(\psi_{\theta,\bm v}^{(1,1)}(X^+,X^-))
}{B_+B_-}.
\end{equation*}
Thus the first-order Hoeffding projections are the directional variance terms with scales \(B_+^{-1}\) and \(B_-^{-1}\) in the diffusion covariance of the local Langevin interface. When those projections vanish, the directional variance contribution is governed by the canonical term with scale \((B_+B_-)^{-1}\), which is \(B^{-2}\) in the equal-size case.
\end{proof}

\subsection{Extensions Beyond Binary AUC}
\label{app:extensions-beyond-binary-auc}

Multiclass one-vs-one AUC is a weighted combination of binary AUC comparisons between class pairs, while one-vs-rest AUC combines binary comparisons between each class and the remaining classes. Multilabel macro-AUC is an average of label-wise binary AUC objectives. The corresponding pairwise gradients can therefore be decomposed through the binary comparisons that constitute each metric.

In these settings, the same example may participate in several class-pair or label-wise objectives. All first-order contributions produced by that example must be aggregated before the variance of the first-order Hoeffding term \(V_1\) is computed. The disappearance of a first-order projection for one class pair or one label is not sufficient for the complete multiclass or multilabel first-order term to vanish. Exact \(B^{-2}\) scaling requires the aggregated first-order projection to vanish as a whole.

Partial AUC and fixed top-\(k\) metrics can likewise be written as restricted or weighted pairwise objectives, with weights selecting pairs in a specified false-positive-rate region or ranking position. If a partial-AUC threshold or top-\(k\) boundary is estimated from the current mini-batch, the sampling fluctuation of that threshold or boundary contributes an additional first-order random term. The corresponding degeneracy condition must therefore apply to the complete aggregated gradient, including the effect of boundary estimation. These decompositions provide a unified form for applying our first-order projection and variance-scaling analysis to multiclass, multilabel, and restricted ranking objectives, while making the corresponding aggregate degeneracy conditions explicit.

\section{Bad Regions and Hoeffding Projection Collapse}
\label{app:pairwise-bad-region-dynamics}

This section collects the technical statements used by Section~\ref{sec:bad-regions}. The statements isolate three ingredients: severe misrankings, loss of deterministic corrective motion, and collapse of the first-order stochastic term in the two-sample mini-batch gradient.

The section begins by defining the severe-error functional and its steepest local correction direction; the corrective-cone formulation is recorded only as an appendix-level extension. Proposition~\ref{prop:app-ustat} through Corollary~\ref{cor:app-first-second-order} then prove the U-statistic variance decomposition, and Proposition~\ref{prop:app-cone-bad-region} records the corresponding appendix-level cone extension. Proposition~\ref{prop:app-severe-compression-small-projections} through Corollary~\ref{cor:app-fully-severe-compressed-collapse} give conditional sufficient conditions under which severe-misranking regions make the first-order stochastic term small, increasingly small as severe-misranking mass grows, or exactly degenerate. Proposition~\ref{prop:app-explicit-bad-regions} through Corollary~\ref{cor:app-mean-badmass-seed} give elementary constructions and initialization-distribution lower bounds for severe-misranking mass. Proposition~\ref{prop:app-finite-pair-high-error-check} through Corollary~\ref{cor:app-init-distribution-badseed} give finite-sample verification statements. Proposition~\ref{prop:app-degenerate-near-stationary-bad-region} and Corollary~\ref{cor:app-degenerate-saddle-bad-region} combine high-error mass, weak drift, and projection collapse into the correction-direction bad-region sufficient condition used in the main text.

\paragraph{Severe-misranking mass and bad-region criterion.}
For \(M>0\), define
\begin{equation*}
\pi_M(\theta)
=
\Pbb(m_\theta(X^+,X^-)\le -M).
\end{equation*}
A region \(\mathcal U\) is high-error at level \((M,\rho)\) if
\begin{equation*}
\inf_{\theta\in\mathcal U}\pi_M(\theta)\ge\rho.
\end{equation*}
It is bad for SGD in a specified correction direction \(\bm v\) when high-error mass persists while both local correction components are weak: the projected drift \(|\langle\nabla R_\ell(\theta),\bm v\rangle|\) is small, and the first-order Hoeffding variance defined below is small relative to the residual pair-interaction variance. The next two paragraphs give the canonical steepest construction of this direction from the severe-error statistic and the corresponding cone version for comparable first-order corrections.

\paragraph{Severe-error functional and steepest correction direction.}
The directional conditions above are local statements about motion that can reduce the severe-misranking statistic. To make this explicit, fix a nonnegative severe-pair weight \(a_M(x)\). For the differentiable statements below, take \(a_M\) to be a smooth approximation of \(\mathbf 1\{x\le -M\}\); the hard-mask version is used only in the finite-batch logging rule in Appendix~\ref{app:mechanism-diagnostics}. Let \(A_{M,\ell}\) be a differentiable margin-error potential satisfying
\begin{equation}
\label{eq:app-severe-potential-derivative}
A_{M,\ell}'(x)
=
-a_M(x)\cdot w_\ell(x),
\qquad
w_\ell(x)=-\ell'(x).
\end{equation}
Define the local severe-error functional
\begin{equation}
\label{eq:app-severe-error-functional}
\Phi_{M,\ell}(\theta)
=
\E\!\left[
A_{M,\ell}(m_\theta(X^+,X^-))
\right].
\end{equation}
When differentiation under the expectation is justified, for example under the same local smoothness and integrability assumptions used for the pairwise gradient in Appendix~\ref{app:pairwise-sgd-dynamics}, its gradient is
\begin{equation}
\label{eq:app-severe-error-gradient}
\nabla\Phi_{M,\ell}(\theta)
=
\E\!\left[
a_M(m_\theta(X^+,X^-))\cdot w_\ell(m_\theta(X^+,X^-))
\cdot
\left(\nabla f_\theta(X^-)-\nabla f_\theta(X^+)\right)
\right].
\end{equation}
When \(\nabla \Phi_{M,\ell}(\theta)\ne \bm 0\), the unit direction
\begin{equation}
\label{eq:app-steepest-severe-direction}
\bm v_{\rm sev}(\theta)
=
-
\frac{\nabla \Phi_{M,\ell}(\theta)}
{\|\nabla \Phi_{M,\ell}(\theta)\|_2}
\end{equation}
is the steepest local direction for reducing \(\Phi_{M,\ell}\) under the Euclidean parameter norm. Indeed, by Cauchy--Schwarz, for any unit \(\bm u\),
\begin{equation}
\label{eq:app-steepest-ineq}
-\langle\nabla \Phi_{M,\ell}(\theta),\bm u\rangle
\le
\|\nabla \Phi_{M,\ell}(\theta)\|_2,
\end{equation}
with equality at \(\bm u=\bm v_{\rm sev}(\theta)\). Thus a mini-batch update \(\Delta\theta\) gives first-order improvement of the severe-error functional according to
\begin{equation}
\label{eq:app-first-order-severe-improvement}
\Phi_{M,\ell}(\theta+\Delta\theta)-\Phi_{M,\ell}(\theta)
=
\langle\nabla \Phi_{M,\ell}(\theta),\Delta\theta\rangle
+o(\|\Delta\theta\|_2).
\end{equation}
For this mechanism, stochastic motion is measured after projection onto the steepest direction that decreases \(\Phi_{M,\ell}\), with \(\bm v_{\rm sev}\) being the canonical direction.

\paragraph{Corrective cone.}
The projected statements used in the main text can be stated for a fixed correction direction or for a cone of comparable first-order corrections. For \(\alpha>0\), define
\begin{equation}
\label{eq:app-corrective-cone}
\mathcal C_\alpha(\theta)
=
\left\{
\bm u:\|\bm u\|_2=1,\ 
\left\langle\nabla \Phi_{M,\ell}(\theta),\bm u\right\rangle
\le
-\alpha
\right\}.
\end{equation}
Directions in \(\mathcal C_\alpha(\theta)\) all reduce the severe-error functional at first order by at least \(\alpha\). The single direction \(\bm v_{\rm sev}\) in Eq.~\eqref{eq:app-steepest-severe-direction} gives the steepest correction direction, while the cone formulation records the appendix-level comparison class. When \(\|\nabla\Phi_{M,\ell}(\theta)\|_2\ge\alpha\), this steepest direction belongs to \(\mathcal C_\alpha(\theta)\).

\begin{proposition}[Directional pairwise gradients are two-sample U-statistics]
\label{prop:app-ustat}
Fix \(\theta\) and a unit vector \(\bm v\). Define
\begin{equation*}
\psi_{\theta,\bm v}(x^+,x^-)
=
\left\langle
\ell'(m_\theta(x^+,x^-))
\cdot
\bigl(\nabla f_\theta(x^+)-\nabla f_\theta(x^-)\bigr),
\bm v
\right\rangle .
\end{equation*}
Then
\begin{equation*}
\langle \widehat{\bm g}_B^\ell(\theta),\bm v\rangle
=
\frac{1}{B^2}
\sum_{i=1}^B\sum_{j=1}^B
\psi_{\theta,\bm v}(X_i^+,X_j^-).
\end{equation*}
\end{proposition}

\begin{proof}
Taking the inner product of Eq.~\eqref{eq:app-pairwise-gradient} with \(\bm v\) and using \(w_\ell=-\ell'\) gives
\begin{equation*}
\langle \widehat{\bm g}_B^\ell(\theta),\bm v\rangle
=
\frac{1}{B^2}
\sum_{i=1}^B\sum_{j=1}^B
\left\langle
-\ell'(m_{ij})\cdot(\bm g_j^- - \bm g_i^+),
\bm v
\right\rangle .
\end{equation*}
This equals
\begin{equation*}
\frac{1}{B^2}
\sum_{i=1}^B\sum_{j=1}^B
\left\langle
\ell'(m_{ij})\cdot(\bm g_i^+ - \bm g_j^-),
\bm v
\right\rangle .
\end{equation*}
Substituting \(m_{ij}=m_\theta(X_i^+,X_j^-)\), \(\bm g_i^+=\nabla f_\theta(X_i^+)\), and \(\bm g_j^-=\nabla f_\theta(X_j^-)\) yields the displayed two-sample U-statistic representation.
\end{proof}

\begin{theorem}[Two-sample Hoeffding decomposition and variance split]
\label{thm:app-hoeffding}
Assume the positive batch is i.i.d., the negative batch is i.i.d., the two classes are independent, and
\begin{equation*}
\E[\psi_{\theta,\bm v}(X^+,X^-)^2]<\infty.
\end{equation*}
Let
\begin{equation*}
\mu_{\theta,\bm v}
=
\E[\psi_{\theta,\bm v}(X^+,X^-)],
\end{equation*}
\begin{equation*}
\psi_{\theta,\bm v}^{(1,0)}(x^+)
=
\E[\psi_{\theta,\bm v}(x^+,X^-)]-\mu_{\theta,\bm v},
\qquad
\psi_{\theta,\bm v}^{(0,1)}(x^-)
=
\E[\psi_{\theta,\bm v}(X^+,x^-)]-\mu_{\theta,\bm v},
\end{equation*}
and
\begin{equation*}
\psi_{\theta,\bm v}^{(1,1)}(x^+,x^-)
=
\psi_{\theta,\bm v}(x^+,x^-)
-\E[\psi_{\theta,\bm v}(x^+,X^-)]
-\E[\psi_{\theta,\bm v}(X^+,x^-)]
+\mu_{\theta,\bm v}.
\end{equation*}
Then
\begin{align}
\langle \widehat{\bm g}_B^\ell(\theta),\bm v\rangle-\mu_{\theta,\bm v}
&=
\frac{1}{B}\sum_{i=1}^B
\psi_{\theta,\bm v}^{(1,0)}(X_i^+)
+
\frac{1}{B}\sum_{j=1}^B
\psi_{\theta,\bm v}^{(0,1)}(X_j^-)
\notag\\
&\quad+
\frac{1}{B^2}\sum_{i=1}^B\sum_{j=1}^B
\psi_{\theta,\bm v}^{(1,1)}(X_i^+,X_j^-),
\label{eq:app-hoeffding-decomp}
\end{align}
and
\begin{equation}
\Var(\langle \widehat{\bm g}_B^\ell(\theta),\bm v\rangle)
=
\frac{
\Var(\psi_{\theta,\bm v}^{(1,0)}(X^+))
+
\Var(\psi_{\theta,\bm v}^{(0,1)}(X^-))
}{B}
+
\frac{
\Var(\psi_{\theta,\bm v}^{(1,1)}(X^+,X^-))
}{B^2}.
\label{eq:app-var-split}
\end{equation}
\end{theorem}

\begin{proof}
For every \((x^+,x^-)\),
\begin{equation*}
\psi_{\theta,\bm v}(x^+,x^-)-\mu_{\theta,\bm v}
=
\psi_{\theta,\bm v}^{(1,0)}(x^+)
+
\psi_{\theta,\bm v}^{(0,1)}(x^-)
+
\psi_{\theta,\bm v}^{(1,1)}(x^+,x^-),
\end{equation*}
which follows by expanding the definition of \(\psi_{\theta,\bm v}^{(1,1)}\). Averaging this identity over the \(B^2\) pairs proves Eq.~\eqref{eq:app-hoeffding-decomp}.

The two first-order projections are centered:
\begin{equation*}
\E[\psi_{\theta,\bm v}^{(1,0)}(X^+)]=0,
\qquad
\E[\psi_{\theta,\bm v}^{(0,1)}(X^-)]=0.
\end{equation*}
The canonical term satisfies
\begin{equation*}
\E[\psi_{\theta,\bm v}^{(1,1)}(X^+,X^-)\mid X^+]=0,
\qquad
\E[\psi_{\theta,\bm v}^{(1,1)}(X^+,X^-)\mid X^-]=0,
\end{equation*}
because the positive and negative classes are independent and the conditional expectations subtracted in its definition match the corresponding conditional means.

Let
\begin{equation*}
P_B^+
=
\frac{1}{B}\sum_{i=1}^B
\psi_{\theta,\bm v}^{(1,0)}(X_i^+),
\quad
P_B^-
=
\frac{1}{B}\sum_{j=1}^B
\psi_{\theta,\bm v}^{(0,1)}(X_j^-),
\end{equation*}
and
\begin{equation*}
H_B
=
\frac{1}{B^2}\sum_{i=1}^B\sum_{j=1}^B
\psi_{\theta,\bm v}^{(1,1)}(X_i^+,X_j^-).
\end{equation*}
The decomposition gives
\begin{equation*}
\langle \widehat{\bm g}_B^\ell(\theta),\bm v\rangle-\mu_{\theta,\bm v}
=
P_B^+ + P_B^- + H_B.
\end{equation*}
The positive and negative batches are independent, so
\begin{equation*}
\Cov(P_B^+,P_B^-)=0.
\end{equation*}
Conditioning on the whole positive batch gives
\begin{equation*}
\E[H_B\mid X_1^+,\ldots,X_B^+]=0,
\end{equation*}
and therefore \(\Cov(P_B^+,H_B)=0\). Conditioning on the whole negative batch similarly gives \(\Cov(P_B^-,H_B)=0\). Thus
\begin{equation*}
\Var(P_B^+ + P_B^- + H_B)
=
\Var(P_B^+)+\Var(P_B^-)+\Var(H_B).
\end{equation*}
Since the positive samples are i.i.d. and the first-order projection is centered,
\begin{equation*}
\Var(P_B^+)
=
\frac{1}{B}\Var(\psi_{\theta,\bm v}^{(1,0)}(X^+)),
\qquad
\Var(P_B^-)
=
\frac{1}{B}\Var(\psi_{\theta,\bm v}^{(0,1)}(X^-)).
\end{equation*}
For \(H_B\), every off-diagonal covariance is zero. If one index is shared, conditioning on the shared sample gives zero because the canonical term has zero conditional mean. If both indices differ, the two terms are independent and centered. Therefore only the \(B^2\) diagonal terms remain:
\begin{equation*}
\Var(H_B)
=
\frac{B^2}{B^4}
\cdot
\Var(\psi_{\theta,\bm v}^{(1,1)}(X^+,X^-))
=
\frac{\Var(\psi_{\theta,\bm v}^{(1,1)}(X^+,X^-))}{B^2}.
\end{equation*}
Combining the three variance identities proves Eq.~\eqref{eq:app-var-split}. The same decomposition is the standard two-sample Hoeffding decomposition~\citep{lee2019u,serfling2009approximation}.
\end{proof}

\begin{corollary}[Unequal positive and negative batch sizes]
\label{cor:app-unequal-batch-hoeffding}
Let the positive and negative mini-batch sizes be \(B_+\) and \(B_-\), and define
\begin{equation*}
\widehat{\bm g}_{B_+,B_-}^\ell(\theta)
=
\frac{1}{B_+B_-}
\sum_{i=1}^{B_+}\sum_{j=1}^{B_-}
w_\ell(m_{ij})\cdot(\bm g_j^- - \bm g_i^+).
\end{equation*}
Under the assumptions of Theorem~\ref{thm:app-hoeffding},
\begin{equation*}
\Var(\langle \widehat{\bm g}_{B_+,B_-}^\ell(\theta),\bm v\rangle)
=
\frac{\Var(\psi_{\theta,\bm v}^{(1,0)}(X^+))}{B_+}
+
\frac{\Var(\psi_{\theta,\bm v}^{(0,1)}(X^-))}{B_-}
+
\frac{\Var(\psi_{\theta,\bm v}^{(1,1)}(X^+,X^-))}{B_+B_-}.
\end{equation*}
Thus unequal class counts change only the constants attached to the two first-order class projections; when \(B_+\) and \(B_-\) grow proportionally, the first-order and residual terms retain the same \(B^{-1}\) versus \(B^{-2}\) scaling used in the main text.
\end{corollary}

\begin{proof}
The Hoeffding decomposition in Eq.~\eqref{eq:app-hoeffding-decomp} becomes
\begin{align*}
\langle \widehat{\bm g}_{B_+,B_-}^\ell(\theta),\bm v\rangle-\mu_{\theta,\bm v}
={}&
\frac{1}{B_+}\sum_{i=1}^{B_+}
\psi_{\theta,\bm v}^{(1,0)}(X_i^+)
+
\frac{1}{B_-}\sum_{j=1}^{B_-}
\psi_{\theta,\bm v}^{(0,1)}(X_j^-)\\
&+
\frac{1}{B_+B_-}\sum_{i=1}^{B_+}\sum_{j=1}^{B_-}
\psi_{\theta,\bm v}^{(1,1)}(X_i^+,X_j^-).
\end{align*}
The orthogonality arguments in the proof of Theorem~\ref{thm:app-hoeffding} are unchanged. The two first-order averages have variances \(B_+^{-1}\Var(\psi_{\theta,\bm v}^{(1,0)}(X^+))\) and \(B_-^{-1}\Var(\psi_{\theta,\bm v}^{(0,1)}(X^-))\). For the residual term, only the \(B_+B_-\) diagonal pair terms contribute, giving the factor \((B_+B_-)^{-1}\).
\end{proof}

\begin{proposition}[Pointwise projected-variance baseline]
\label{prop:app-pointwise-baseline}
Let \(Z_1,\ldots,Z_B\) be i.i.d. and define
\begin{equation*}
\bar H_B(\theta)
=
\frac{1}{B}\sum_{i=1}^B h_\theta(Z_i).
\end{equation*}
For any unit vector \(\bm v\) with
\begin{equation*}
\Var(\langle h_\theta(Z),\bm v\rangle)>0,
\end{equation*}
one has
\begin{equation*}
\Var(\langle \bar H_B(\theta),\bm v\rangle)
=
\frac{1}{B}\Var(\langle h_\theta(Z),\bm v\rangle).
\end{equation*}
\end{proposition}

\begin{proof}
Set \(Y_i=\langle h_\theta(Z_i),\bm v\rangle\). Then \(Y_1,\ldots,Y_B\) are i.i.d., and
\begin{equation*}
\langle \bar H_B(\theta),\bm v\rangle
=
\frac{1}{B}\sum_{i=1}^B Y_i.
\end{equation*}
Expanding the variance gives
\begin{equation*}
\Var\!\left(\frac{1}{B}\sum_{i=1}^B Y_i\right)
=
\frac{1}{B^2}\sum_{i=1}^B\Var(Y_i)
+
\frac{1}{B^2}\sum_{i\neq j}\Cov(Y_i,Y_j).
\end{equation*}
The cross-covariances vanish by independence, which gives the stated \(B^{-1}\) variance.
\end{proof}

\begin{corollary}[First-order versus second-order stochasticity]
\label{cor:app-first-second-order}
If
\begin{equation*}
\Var(\psi_{\theta,\bm v}^{(1,0)}(X^+))
+
\Var(\psi_{\theta,\bm v}^{(0,1)}(X^-))
>0,
\end{equation*}
then
\begin{equation*}
\Var(\langle \widehat{\bm g}_B^\ell(\theta),\bm v\rangle)
=
\Omega(B^{-1}).
\end{equation*}
If
\begin{equation*}
\Var(\psi_{\theta,\bm v}^{(1,0)}(X^+))
=
\Var(\psi_{\theta,\bm v}^{(0,1)}(X^-))
=0
\end{equation*}
and
\begin{equation*}
\Var(\psi_{\theta,\bm v}^{(1,1)}(X^+,X^-))>0,
\end{equation*}
then
\begin{equation*}
\Var(\langle \widehat{\bm g}_B^\ell(\theta),\bm v\rangle)
=
\Theta(B^{-2}).
\end{equation*}
\end{corollary}

\begin{proof}
The conclusion follows directly from Eq.~\eqref{eq:app-var-split}. A nonzero coefficient on \(B^{-1}\) gives the first claim. If both first-order variances vanish and the residual pair-interaction variance is positive, the variance equals
\begin{equation*}
\frac{
\Var(\psi_{\theta,\bm v}^{(1,1)}(X^+,X^-))
}{B^2}.
\end{equation*}
\end{proof}

\begin{proposition}[Cone-level sufficient bad-region criterion]
\label{prop:app-cone-bad-region}
Let \(\mathcal U\subset\R^d\) be a neighborhood, and let \(\mathcal C_\alpha(\theta)\) be defined by Eq.~\eqref{eq:app-corrective-cone}. Suppose that \(\mathcal U\) is high-error at level \((M,\rho)\) and that, for every \(\theta\in\mathcal U\) and every \(\bm u\in\mathcal C_\alpha(\theta)\),
\begin{equation*}
|\langle\nabla R_\ell(\theta),\bm u\rangle|
\le
\varepsilon_{\rm drift},
\end{equation*}
and
\begin{equation*}
\Var(\psi_{\theta,\bm u}^{(1,0)}(X^+))
+
\Var(\psi_{\theta,\bm u}^{(0,1)}(X^-))
\le
\varepsilon_1.
\end{equation*}
Then every first-order corrective direction in \(\mathcal C_\alpha(\theta)\) has projected deterministic correction bounded by \(\varepsilon_{\rm drift}\), and its order-\(B^{-1}\) Hoeffding contribution is at most \(\varepsilon_1/B\). In particular, if \(\bm v_{\rm sev}(\theta)\in\mathcal C_\alpha(\theta)\), the same bound applies to the steepest severe-error correction direction.
\end{proposition}

\begin{proof}
Fix \(\theta\in\mathcal U\) and \(\bm u\in\mathcal C_\alpha(\theta)\). Since \(\bm u\) is fixed after \(\theta\) is fixed, Proposition~\ref{prop:app-ustat} gives the scalar two-sample U-statistic representation of \(\langle \widehat{\bm g}_B^\ell(\theta),\bm u\rangle\), and Theorem~\ref{thm:app-hoeffding} gives
\begin{equation*}
\Var(\langle \widehat{\bm g}_B^\ell(\theta),\bm u\rangle)
=
\frac{
\Var(\psi_{\theta,\bm u}^{(1,0)}(X^+))
+
\Var(\psi_{\theta,\bm u}^{(0,1)}(X^-))
}{B}
+
\frac{
\Var(\psi_{\theta,\bm u}^{(1,1)}(X^+,X^-))
}{B^2}.
\end{equation*}
The order-\(B^{-1}\) coefficient is exactly
\begin{equation*}
\Var(\psi_{\theta,\bm u}^{(1,0)}(X^+))
+
\Var(\psi_{\theta,\bm u}^{(0,1)}(X^-)).
\end{equation*}
By assumption, this coefficient is at most \(\varepsilon_1\), so the first-order Hoeffding contribution is at most \(\varepsilon_1/B\). The drift assumption gives
\begin{equation*}
|\langle\nabla R_\ell(\theta),\bm u\rangle|
\le
\varepsilon_{\rm drift}.
\end{equation*}
Because the argument holds for an arbitrary \(\theta\in\mathcal U\) and an arbitrary \(\bm u\in\mathcal C_\alpha(\theta)\), the bounds are uniform over the stated cone. Finally, \(\bm v_{\rm sev}(\theta)\) is a unit direction by Eq.~\eqref{eq:app-steepest-severe-direction}; if it also satisfies the cone membership condition, substituting \(\bm u=\bm v_{\rm sev}(\theta)\) gives the last claim.
\end{proof}

\begin{proposition}[Severe-pair compression makes first-order projections small]
\label{prop:app-severe-compression-small-projections}
Fix \(\theta\), a unit direction \(\bm v\), and \(M>0\). Write
\begin{equation*}
S_M^\theta
=
\{(x^+,x^-):m_\theta(x^+,x^-)\le -M\},
\qquad
\delta_M(\theta)
=
\Pbb((X^+,X^-)\notin S_M^\theta),
\end{equation*}
and abbreviate \(\psi=\psi_{\theta,\bm v}\). Assume
\begin{equation*}
|\psi(X^+,X^-)|\le K
\qquad\text{a.s.}
\end{equation*}
for some \(K<\infty\). Define the severe-truncated row and column means
\begin{equation*}
m_{+,M}(x^+)
=
\E\!\left[
\psi(x^+,X^-)\cdot
\1_{\{(x^+,X^-)\in S_M^\theta\}}
\right],
\end{equation*}
\begin{equation*}
m_{-,M}(x^-)
=
\E\!\left[
\psi(X^+,x^-)\cdot
\1_{\{(X^+,x^-)\in S_M^\theta\}}
\right].
\end{equation*}
If there exist constants \(a_+\) and \(a_-\) such that
\begin{equation*}
\E[(m_{+,M}(X^+)-a_+)^2]\le \varepsilon_+^2,
\qquad
\E[(m_{-,M}(X^-)-a_-)^2]\le \varepsilon_-^2,
\end{equation*}
then the coefficient of the order-\(B^{-1}\) term in the Hoeffding variance split satisfies
\begin{equation*}
\Var(\psi_{\theta,\bm v}^{(1,0)}(X^+))
+
\Var(\psi_{\theta,\bm v}^{(0,1)}(X^-))
\le
2\cdot\varepsilon_+^2
+
2\cdot\varepsilon_-^2
+
4\cdot K^2\cdot\delta_M(\theta).
\end{equation*}
\end{proposition}

\begin{proof}
Let
\begin{equation*}
m_+(x^+)=\E[\psi(x^+,X^-)],
\qquad
m_-(x^-)=\E[\psi(X^+,x^-)].
\end{equation*}
By definition,
\begin{equation*}
\psi_{\theta,\bm v}^{(1,0)}(X^+)=m_+(X^+)-\mu_{\theta,\bm v},
\qquad
\psi_{\theta,\bm v}^{(0,1)}(X^-)=m_-(X^-)-\mu_{\theta,\bm v}.
\end{equation*}
Therefore, for any constant \(a_+\),
\begin{equation*}
\Var(\psi_{\theta,\bm v}^{(1,0)}(X^+))
=
\Var(m_+(X^+))
\le
\E[(m_+(X^+)-a_+)^2].
\end{equation*}
Decompose
\begin{equation*}
m_+(x^+)
=
m_{+,M}(x^+)+r_+(x^+),
\end{equation*}
where
\begin{equation*}
r_+(x^+)
=
\E\!\left[
\psi(x^+,X^-)\cdot
\1_{\{(x^+,X^-)\notin S_M^\theta\}}
\right].
\end{equation*}
The inequality \((u+v)^2\le 2\cdot u^2+2\cdot v^2\) gives
\begin{equation*}
\E[(m_+(X^+)-a_+)^2]
\le
2\cdot\E[(m_{+,M}(X^+)-a_+)^2]
+
2\cdot\E[r_+(X^+)^2].
\end{equation*}
Let
\begin{equation*}
q_+(x^+)=\Pbb((x^+,X^-)\notin S_M^\theta).
\end{equation*}
Since \(|\psi|\le K\),
\begin{equation*}
|r_+(x^+)|\le K\cdot q_+(x^+),
\end{equation*}
and hence
\begin{equation*}
\E[r_+(X^+)^2]
\le
K^2\cdot\E[q_+(X^+)^2]
\le
K^2\cdot\E[q_+(X^+)]
=
K^2\cdot\delta_M(\theta).
\end{equation*}
Thus
\begin{equation*}
\Var(\psi_{\theta,\bm v}^{(1,0)}(X^+))
\le
2\cdot\varepsilon_+^2
+
2\cdot K^2\cdot\delta_M(\theta).
\end{equation*}
The column projection is identical. Namely, with
\begin{equation*}
r_-(x^-)
=
\E\!\left[
\psi(X^+,x^-)\cdot
\1_{\{(X^+,x^-)\notin S_M^\theta\}}
\right],
\qquad
q_-(x^-)=\Pbb((X^+,x^-)\notin S_M^\theta),
\end{equation*}
we have
\begin{equation*}
\E[q_-(X^-)]=\delta_M(\theta)
\end{equation*}
and therefore
\begin{equation*}
\Var(\psi_{\theta,\bm v}^{(0,1)}(X^-))
\le
2\cdot\varepsilon_-^2
+
2\cdot K^2\cdot\delta_M(\theta).
\end{equation*}
Adding the two bounds proves the claim.
\end{proof}

\begin{corollary}[Conditional monotonicity in severe-misranking mass]
\label{cor:app-severe-mass-tightens-projection-bound}
Fix \(M>0\), a unit direction \(\bm v\), and \(K<\infty\). Consider any collection of parameters for which the assumptions of Proposition~\ref{prop:app-severe-compression-small-projections} hold with the same envelope \(K\). For each such \(\theta\), let \(\varepsilon_+(\theta)\) and \(\varepsilon_-(\theta)\) be admissible row and column compression errors in Proposition~\ref{prop:app-severe-compression-small-projections}. Then
\begin{equation*}
\Var(\psi_{\theta,\bm v}^{(1,0)}(X^+))
+
\Var(\psi_{\theta,\bm v}^{(0,1)}(X^-))
\le
2\cdot\varepsilon_+(\theta)^2
+
2\cdot\varepsilon_-(\theta)^2
+
4\cdot K^2\cdot(1-\pi_M(\theta)).
\end{equation*}
Consequently, if two parameters \(\theta_1,\theta_2\) have the same admissible compression errors and
\begin{equation*}
\pi_M(\theta_1)\ge \pi_M(\theta_2),
\end{equation*}
then the projection-collapse upper bound for \(\theta_1\) is no larger than that for \(\theta_2\). More generally, along any sequence \((\theta_n)_{n\ge1}\) for which
\begin{equation*}
\pi_M(\theta_n)\to1,
\qquad
\varepsilon_+(\theta_n)\to0,
\qquad
\varepsilon_-(\theta_n)\to0,
\end{equation*}
the order-\(B^{-1}\) coefficient in the Hoeffding variance split converges to \(0\). If, in addition,
\begin{equation*}
\liminf_{n\to\infty}
\Var(\psi_{\theta_n,\bm v}^{(1,1)}(X^+,X^-))>0,
\end{equation*}
then the leading single-class stochastic term disappears along the sequence while the canonical pair-interaction term remains nontrivial.
\end{corollary}

\begin{proof}
The displayed bound is Proposition~\ref{prop:app-severe-compression-small-projections} with
\begin{equation*}
\delta_M(\theta)=\Pbb((X^+,X^-)\notin S_M^\theta)=1-\pi_M(\theta).
\end{equation*}
If the admissible compression errors are the same for \(\theta_1\) and \(\theta_2\), the only term in the bound that depends on the severe-misranking mass is
\begin{equation*}
4\cdot K^2\cdot(1-\pi_M(\theta)).
\end{equation*}
This term is nonincreasing in \(\pi_M(\theta)\), proving the monotonic comparison. For the sequence statement, the right-hand side satisfies
\begin{equation*}
2\cdot\varepsilon_+(\theta_n)^2
+
2\cdot\varepsilon_-(\theta_n)^2
+
4\cdot K^2\cdot(1-\pi_M(\theta_n))
\to0.
\end{equation*}
Therefore
\begin{equation*}
\Var(\psi_{\theta_n,\bm v}^{(1,0)}(X^+))
+
\Var(\psi_{\theta_n,\bm v}^{(0,1)}(X^-))
\to0.
\end{equation*}
The final claim follows by combining this convergence with the assumed positive lower limit of the canonical variance in the variance split of Theorem~\ref{thm:app-hoeffding}.
\end{proof}

\begin{corollary}[Exact collapse in a fully severe compressed region]
\label{cor:app-fully-severe-compressed-collapse}
In the setting of Proposition~\ref{prop:app-severe-compression-small-projections}, suppose
\begin{equation*}
\Pbb((X^+,X^-)\in S_M^\theta)=1
\end{equation*}
and the row and column conditional means are constant:
\begin{equation*}
\E[\psi_{\theta,\bm v}(X^+,X^-)\mid X^+]=a_+
\quad\text{a.s.},
\qquad
\E[\psi_{\theta,\bm v}(X^+,X^-)\mid X^-]=a_-
\quad\text{a.s.}
\end{equation*}
Then \(a_+=a_-=\mu_{\theta,\bm v}\), both first-order Hoeffding projections vanish, and
\begin{equation*}
\Var(\langle \widehat{\bm g}_B^\ell(\theta),\bm v\rangle)
=
\frac{
\Var(\psi_{\theta,\bm v}^{(1,1)}(X^+,X^-))
}{B^2}.
\end{equation*}
\end{corollary}

\begin{proof}
Taking expectations in the row-mean identity gives
\begin{equation*}
a_+
=
\E[\E[\psi_{\theta,\bm v}(X^+,X^-)\mid X^+]]
=
\mu_{\theta,\bm v}.
\end{equation*}
Taking expectations in the column-mean identity gives \(a_-=\mu_{\theta,\bm v}\). Hence
\begin{equation*}
\psi_{\theta,\bm v}^{(1,0)}(X^+)
=
\E[\psi_{\theta,\bm v}(X^+,X^-)\mid X^+]-\mu_{\theta,\bm v}
=0
\quad\text{a.s.},
\end{equation*}
and similarly
\begin{equation*}
\psi_{\theta,\bm v}^{(0,1)}(X^-)=0
\quad\text{a.s.}
\end{equation*}
The variance identity follows from Theorem~\ref{thm:app-hoeffding}.
\end{proof}

Proposition~\ref{prop:app-severe-compression-small-projections} is deliberately conditional. A large value of \(\pi_M(\theta)\) alone does not force projection collapse. It makes collapse relevant when the severe pairs that dominate the local pair distribution have nearly constant row and column effects in the projected update. Corollary~\ref{cor:app-severe-mass-tightens-projection-bound} states the precise monotonic consequence: under such compression, increasing severe-misranking mass tightens the upper bound on the first-order projections. Corollary~\ref{cor:app-fully-severe-compressed-collapse} records the corresponding exact case.

\begin{proposition}[Elementary severe-misranking and projection-collapse constructions]
\label{prop:app-explicit-bad-regions}
The two ingredients used in the bad-region mechanism are nonempty.
\begin{enumerate}[leftmargin=5mm]
  \item Let the sample space contain one positive point \(x^+\) and one negative point \(x^-\), and define
  \begin{equation*}
  f_\theta(x^+)=-\theta,
  \qquad
  f_\theta(x^-)=\theta.
  \end{equation*}
  Then \(m_\theta(x^+,x^-)=-2\cdot\theta\), and for every \(M>0\), the ray \(\{\theta\ge M/2\}\) satisfies \(\pi_M(\theta)=1\).
  \item Let \(A\) and \(C\) be independent Rademacher random variables, let \(f_\theta(z)=\theta\cdot z\), and use \(\ell(x)=x^2\). For \(\bm v=1\), the scalar pairwise-gradient kernel is
  \begin{equation*}
  G_\theta(A,C)
  =
  2\cdot\theta\cdot(A-C)^2
  =
  4\cdot\theta-4\cdot\theta\cdot A\cdot C.
  \end{equation*}
  Both first-order Hoeffding projections vanish, and
  \begin{equation*}
  \Var(\langle \widehat{\bm g}_B^\ell(\theta),\bm v\rangle)
  =
  \frac{16\cdot\theta^2}{B^2}.
  \end{equation*}
\end{enumerate}
\end{proposition}

\begin{proof}
The first claim follows from
\begin{equation*}
m_\theta(x^+,x^-)=(-\theta)-\theta=-2\cdot\theta.
\end{equation*}
If \(\theta\ge M/2\), then \(m_\theta(x^+,x^-)\le -M\), so the severe-misranking event has probability \(1\).

For the second claim,
\begin{equation*}
m_\theta(A,C)=\theta\cdot(A-C).
\end{equation*}
Since \(w(x)=-\ell'(x)=-2\cdot x\), the scalar gradient kernel in direction \(1\) is
\begin{equation*}
G_\theta(A,C)
=
w(m_\theta(A,C))\cdot(C-A)
=
2\cdot\theta\cdot(A-C)^2.
\end{equation*}
Using \(A^2=C^2=1\) gives
\begin{equation*}
(A-C)^2=2-2\cdot A\cdot C,
\end{equation*}
and hence
\begin{equation*}
G_\theta(A,C)=4\cdot\theta-4\cdot\theta\cdot A\cdot C.
\end{equation*}
Its mean is \(4\cdot\theta\). Conditioning on \(A=a\) gives
\begin{equation*}
\E[G_\theta(a,C)]-4\cdot\theta
=
-4\cdot\theta\cdot a\cdot\E[C]
=0,
\end{equation*}
and conditioning on \(C=c\) gives the same result. Therefore both first-order projections vanish. The centered canonical term is
\begin{equation*}
-4\cdot\theta\cdot A\cdot C,
\end{equation*}
whose variance is \(16\cdot\theta^2\). Applying Eq.~\eqref{eq:app-var-split} gives the displayed \(B^{-2}\) variance.
\end{proof}

\begin{proposition}[High-error starts under random or label-blind initialization]
\label{prop:app-high-error-starts}
High-error starts can be phrased as probability statements.
\begin{enumerate}[leftmargin=5mm]
  \item In the severe-misranking ray of Proposition~\ref{prop:app-explicit-bad-regions}, if \(\Theta_0\sim N(0,\sigma^2)\), then
  \begin{equation*}
  \Pbb(\Theta_0\ge M/2)
  =
  1-\Phi\!\left(\frac{M}{2\cdot\sigma}\right).
  \end{equation*}
  \item Let \(\Theta_0\) be a random initialization. Conditional on \(\Theta_0\), define
  \begin{equation*}
  Y^+=f_{\Theta_0}(X^+),
  \qquad
  Y^-=f_{\Theta_0}(X^-).
  \end{equation*}
  If \(Y^+\) and \(Y^-\) are conditionally independent and identically distributed, then
  \begin{equation*}
  \pi_0(\Theta_0)\ge\frac{1}{2}
  \qquad\text{a.s.}
  \end{equation*}
  If additionally
  \begin{equation*}
  Y^+,Y^- \stackrel{\mathrm{i.i.d.}}{\sim} N(0,\sigma^2)
  \quad\text{conditional on }\Theta_0,
  \end{equation*}
  then for every \(M\ge0\),
  \begin{equation*}
  \pi_M(\Theta_0)
  =
  1-\Phi\!\left(\frac{M}{\sqrt2\cdot\sigma}\right).
  \end{equation*}
\end{enumerate}
\end{proposition}

\begin{proof}
The Gaussian ray formula follows by standardization. For the label-blind statement, conditional on \(\Theta_0\), the distribution of \(Y^+-Y^-\) is symmetric about \(0\) because \((Y^+,Y^-)\) and \((Y^-,Y^+)\) have the same conditional law. Hence
\begin{equation*}
\Pbb(Y^+-Y^-\le0\mid \Theta_0)
=
\frac{1+\Pbb(Y^+-Y^-=0\mid\Theta_0)}{2}
\ge
\frac{1}{2}.
\end{equation*}
Under the conditional Gaussian assumption,
\begin{equation*}
Y^+-Y^- \mid \Theta_0
\sim
N(0,2\cdot\sigma^2),
\end{equation*}
which gives the displayed formula for \(\pi_M(\Theta_0)\).
\end{proof}

\begin{proposition}[Random isotropic linear heads give severe-misranking lower bounds]
\label{prop:app-isotropic-head-badmass}
Let \(h:\mathcal X\to\R^d\) be measurable and consider
\begin{equation*}
f_{\bm w}(x)=\bm w^\top h(x).
\end{equation*}
Let
\begin{equation*}
\bm W_0\sim N(0,\alpha^2\cdot\bm I_d),
\qquad
\alpha>0,
\end{equation*}
independently of \((X^+,X^-)\). Define
\begin{equation*}
\bm D=h(X^+)-h(X^-),
\qquad
\pi_M(\bm w)=\Pbb(\bm w^\top\bm D\le -M).
\end{equation*}
Then
\begin{equation*}
\E[\pi_0(\bm W_0)]
=
\frac{1}{2}+\frac{1}{2}\cdot\Pbb(\bm D=\bm 0)
\ge
\frac{1}{2}.
\end{equation*}
For \(M>0\),
\begin{equation*}
\E[\pi_M(\bm W_0)]
=
\E[H_M(\|\bm D\|_2)],
\end{equation*}
where
\begin{equation*}
H_M(r)
=
\begin{cases}
1-\Phi\!\left(\dfrac{M}{\alpha\cdot r}\right), & r>0,\\[6pt]
0, & r=0.
\end{cases}
\end{equation*}
If \(R>0\) and \(\Pbb(\|\bm D\|_2\ge R)\ge q\), then
\begin{equation*}
\E[\pi_M(\bm W_0)]
\ge
q\cdot
\left\{
1-\Phi\!\left(\frac{M}{\alpha\cdot R}\right)
\right\}.
\end{equation*}
\end{proposition}

\begin{proof}
Condition on \(\bm D=\bm d\). If \(\bm d\neq\bm 0\), then
\begin{equation*}
\bm W_0^\top\bm d
\sim
N(0,\alpha^2\cdot\|\bm d\|_2^2).
\end{equation*}
If \(\bm d=\bm 0\), then \(\bm W_0^\top\bm d=0\) almost surely. This proves the \(M=0\) identity and the exact \(H_M\) formula. If \(R>0\) and \(\|\bm D\|_2\ge R\), then
\begin{equation*}
\frac{M}{\alpha\cdot\|\bm D\|_2}
\le
\frac{M}{\alpha\cdot R},
\end{equation*}
and \(1-\Phi(\cdot)\) is nonincreasing. Taking expectation over this event gives the lower bound.
\end{proof}

\begin{lemma}[One-sided second-moment lower bound]
\label{lem:app-one-sided-second-moment}
Let \(X\ge0\), \(0<\E[X]<\infty\), and \(\E[X^2]<\infty\). For every \(\lambda\in[0,1)\),
\begin{equation*}
\Pbb(X\ge \lambda\cdot\E[X])
\ge
\frac{(1-\lambda)^2\cdot\E[X]^2}{\E[X^2]}.
\end{equation*}
\end{lemma}

\begin{proof}
Let \(A_\lambda=\{X\ge\lambda\cdot\E[X]\}\). On \(A_\lambda^c\), \(X<\lambda\cdot\E[X]\), so
\begin{equation*}
\E[X\cdot\1_{A_\lambda}]
=
\E[X]-\E[X\cdot\1_{A_\lambda^c}]
\ge
(1-\lambda)\cdot\E[X].
\end{equation*}
By Cauchy-Schwarz,
\begin{equation*}
\E[X\cdot\1_{A_\lambda}]^2
\le
\E[X^2]\cdot\Pbb(A_\lambda).
\end{equation*}
Combining the two inequalities proves the claim.
\end{proof}

\begin{proposition}[High-dimensional scaled random heads retain nontrivial bad mass]
\label{prop:app-highdim-scaled-head}
For each \(d\ge1\), let \(h_d:\mathcal X\to\R^d\) and
\begin{equation*}
\bm d_d=h_d(X^+)-h_d(X^-).
\end{equation*}
Let
\begin{equation*}
\bm w_{0,d}\sim N\!\left(0,\frac{\sigma^2}{d}\cdot\bm I_d\right),
\qquad
\sigma>0,
\end{equation*}
independently of \((X^+,X^-)\). Assume there exist constants \(a,b>0\), independent of \(d\), such that
\begin{equation*}
\E[\|\bm d_d\|_2^2]\ge a^2\cdot d,
\qquad
\E[\|\bm d_d\|_2^4]\le b^4\cdot d^2.
\end{equation*}
Then for every \(M>0\),
\begin{equation*}
\E[\pi_{M,d}(\bm w_{0,d})]
\ge
\frac{a^4}{4\cdot b^4}
\left\{
1-\Phi\!\left(\frac{\sqrt2\cdot M}{\sigma\cdot a}\right)
\right\},
\end{equation*}
where \(\pi_{M,d}(\bm w)=\Pbb(\bm w^\top\bm d_d\le -M)\).
\end{proposition}

\begin{proof}
Let \(X_d=\|\bm d_d\|_2^2\). Lemma~\ref{lem:app-one-sided-second-moment} with \(\lambda=1/2\) gives
\begin{equation*}
\Pbb\!\left(X_d\ge\frac{1}{2}\cdot\E[X_d]\right)
\ge
\frac{\E[X_d]^2}{4\cdot\E[X_d^2]}
\ge
\frac{a^4}{4\cdot b^4}.
\end{equation*}
Since \(\E[X_d]\ge a^2\cdot d\), this event implies
\begin{equation*}
\|\bm d_d\|_2\ge a\cdot\sqrt{d/2}.
\end{equation*}
Apply Proposition~\ref{prop:app-isotropic-head-badmass} with
\begin{equation*}
\alpha=\frac{\sigma}{\sqrt d},
\qquad
R=a\cdot\sqrt{d/2},
\qquad
q=\frac{a^4}{4\cdot b^4}.
\end{equation*}
Because \(\alpha\cdot R=\sigma\cdot a/\sqrt2\), the result follows.
\end{proof}

\begin{corollary}[Average severe-misranking mass gives a nontrivial high-error initialization probability]
\label{cor:app-mean-badmass-seed}
Let \(W_0\) be any random initialization and \(Y_M=\pi_M(W_0)\in[0,1]\). If \(\E[Y_M]\ge\mu_M\) for \(\mu_M\in(0,1]\), then for every \(\tau\in[0,\mu_M)\),
\begin{equation*}
\Pbb(Y_M\ge\tau)
\ge
\frac{\mu_M-\tau}{1-\tau}.
\end{equation*}
In particular,
\begin{equation*}
\Pbb\!\left(Y_M\ge\frac{\mu_M}{2}\right)
\ge
\frac{\mu_M}{2-\mu_M}
\ge
\frac{\mu_M}{2}.
\end{equation*}
\end{corollary}

\begin{proof}
Since \(0\le Y_M\le1\),
\begin{equation*}
Y_M
\le
\tau+(1-\tau)\cdot\1_{\{Y_M\ge\tau\}}.
\end{equation*}
Taking expectations and using \(\E[Y_M]\ge\mu_M\) gives the first inequality. The final display follows by taking \(\tau=\mu_M/2\).
\end{proof}

\begin{proposition}[Finite-pair high-error check]
\label{prop:app-finite-pair-high-error-check}
Fix \(\theta\), \(M\ge0\), and
\begin{equation*}
\pi_M(\theta)=\Pbb(m_\theta(X^+,X^-)\le -M).
\end{equation*}
Let \((\widetilde X_k^+,\widetilde X_k^-)_{k=1}^K\) be independent positive-negative pairs and define
\begin{equation*}
\widehat\pi_{M,K}(\theta)
=
\frac{1}{K}\sum_{k=1}^K
\1_{\{m_\theta(\widetilde X_k^+,\widetilde X_k^-)\le -M\}}.
\end{equation*}
Then
\begin{equation*}
\E[\widehat\pi_{M,K}(\theta)]=\pi_M(\theta),
\qquad
\Var(\widehat\pi_{M,K}(\theta))
\le
\frac{1}{4\cdot K}.
\end{equation*}
Consequently, for every \(\varepsilon>0\),
\begin{equation*}
\Pbb\!\left(
|\widehat\pi_{M,K}(\theta)-\pi_M(\theta)|\ge\varepsilon
\right)
\le
\frac{1}{4\cdot K\cdot\varepsilon^2}.
\end{equation*}
Equivalently,
\begin{equation*}
\Pbb\!\left(
\pi_M(\theta)\ge \widehat\pi_{M,K}(\theta)-\varepsilon
\right)
\ge
1-\frac{1}{4\cdot K\cdot\varepsilon^2},
\end{equation*}
and
\begin{equation*}
\Pbb\!\left(
\pi_M(\theta)\le \widehat\pi_{M,K}(\theta)+\varepsilon
\right)
\ge
1-\frac{1}{4\cdot K\cdot\varepsilon^2}.
\end{equation*}
\end{proposition}

\begin{proof}
The indicators in the empirical average are i.i.d. Bernoulli random variables with mean \(\pi_M(\theta)\). Their average has expectation \(\pi_M(\theta)\) and variance
\begin{equation*}
\frac{\pi_M(\theta)\cdot(1-\pi_M(\theta))}{K}
\le
\frac{1}{4\cdot K}.
\end{equation*}
Chebyshev's inequality gives the deviation bound. The one-sided lower and upper statements follow because the events
\begin{equation*}
\left\{
\widehat\pi_{M,K}(\theta)-\pi_M(\theta)\ge\varepsilon
\right\},
\qquad
\left\{
\pi_M(\theta)-\widehat\pi_{M,K}(\theta)\ge\varepsilon
\right\}
\end{equation*}
are both contained in
\begin{equation*}
\left\{
|\widehat\pi_{M,K}(\theta)-\pi_M(\theta)|\ge\varepsilon
\right\}.
\end{equation*}
\end{proof}

\begin{proposition}[Finite-initialization lower bound for high-error starts]
\label{prop:app-bad-seed-frequency}
Fix \(M\ge0\) and a seed-level threshold \(\tau\in[0,1]\). Let
\begin{equation*}
\theta^{(1)},\ldots,\theta^{(S)}
\end{equation*}
be \(S\) parameter points. For each seed \(s\), construct an independent pairwise estimate
\begin{equation*}
\widehat\pi^{(s)}_{M,K_s}(\theta^{(s)})
\end{equation*}
as in Proposition~\ref{prop:app-finite-pair-high-error-check}, choose \(\varepsilon_s>0\), and define
\begin{equation*}
C_s
=
\1_{\{\widehat\pi^{(s)}_{M,K_s}(\theta^{(s)})-\varepsilon_s\ge\tau\}}.
\end{equation*}
Then
\begin{equation*}
\Pbb\!\left(
\frac{1}{S}\sum_{s=1}^S
\1_{\{\pi_M(\theta^{(s)})\ge\tau\}}
\ge
\frac{1}{S}\sum_{s=1}^S C_s
\right)
\ge
1-\sum_{s=1}^S
\frac{1}{4\cdot K_s\cdot\varepsilon_s^2}.
\end{equation*}
\end{proposition}

\begin{proof}
For each \(s\), define
\begin{equation*}
E_s
=
\left\{
\pi_M(\theta^{(s)})
\ge
\widehat\pi^{(s)}_{M,K_s}(\theta^{(s)})-\varepsilon_s
\right\}.
\end{equation*}
By Proposition~\ref{prop:app-finite-pair-high-error-check},
\begin{equation*}
\Pbb(E_s)
\ge
1-\frac{1}{4\cdot K_s\cdot\varepsilon_s^2}.
\end{equation*}
The union bound gives
\begin{equation*}
\Pbb\!\left(\bigcap_{s=1}^S E_s\right)
\ge
1-\sum_{s=1}^S
\frac{1}{4\cdot K_s\cdot\varepsilon_s^2}.
\end{equation*}
On \(\bigcap_{s=1}^S E_s\), if \(C_s=1\), then
\begin{equation*}
\widehat\pi^{(s)}_{M,K_s}(\theta^{(s)})-\varepsilon_s\ge\tau
\end{equation*}
and therefore
\begin{equation*}
\pi_M(\theta^{(s)})\ge\tau.
\end{equation*}
Thus
\begin{equation*}
C_s
\le
\1_{\{\pi_M(\theta^{(s)})\ge\tau\}}
\end{equation*}
for every \(s\). Averaging over \(s=1,\ldots,S\) proves the claim.
\end{proof}

\begin{corollary}[Initialization-distribution lower bound for high-error starts]
\label{cor:app-init-distribution-badseed}
Let \(W_0\) be a random initialization and define
\begin{equation*}
p_{M,\tau}
=
\Pbb(\pi_M(W_0)\ge\tau).
\end{equation*}
Let \(W_0^{(1)},\ldots,W_0^{(S)}\) be independent draws from the same initialization distribution. For each \(s\), build \(C_s\) as in Proposition~\ref{prop:app-bad-seed-frequency}, using pairwise samples independent of the initializations. Then for every \(\rho>0\),
\begin{equation*}
\Pbb\!\left(
p_{M,\tau}
\ge
\frac{1}{S}\sum_{s=1}^S C_s-\rho
\right)
\ge
1-\sum_{s=1}^S
\frac{1}{4\cdot K_s\cdot\varepsilon_s^2}
-\frac{1}{4\cdot S\cdot\rho^2}.
\end{equation*}
\end{corollary}

\begin{proof}
Define the true seed indicators
\begin{equation*}
B_s=\1_{\{\pi_M(W_0^{(s)})\ge\tau\}},
\qquad
s=1,\ldots,S.
\end{equation*}
Then \(B_1,\ldots,B_S\) are i.i.d. Bernoulli random variables with mean \(p_{M,\tau}\), and
\begin{equation*}
\Var\!\left(\frac{1}{S}\sum_{s=1}^S B_s\right)
=
\frac{p_{M,\tau}\cdot(1-p_{M,\tau})}{S}
\le
\frac{1}{4\cdot S}.
\end{equation*}
Chebyshev's inequality gives
\begin{equation*}
\Pbb\!\left(
p_{M,\tau}
\ge
\frac{1}{S}\sum_{s=1}^S B_s-\rho
\right)
\ge
1-\frac{1}{4\cdot S\cdot\rho^2}.
\end{equation*}
By Proposition~\ref{prop:app-bad-seed-frequency},
\begin{equation*}
\Pbb\!\left(
\frac{1}{S}\sum_{s=1}^S B_s
\ge
\frac{1}{S}\sum_{s=1}^S C_s
\right)
\ge
1-\sum_{s=1}^S
\frac{1}{4\cdot K_s\cdot\varepsilon_s^2}.
\end{equation*}
On the intersection of these two events,
\begin{equation*}
p_{M,\tau}
\ge
\frac{1}{S}\sum_{s=1}^S B_s-\rho
\ge
\frac{1}{S}\sum_{s=1}^S C_s-\rho.
\end{equation*}
Another union bound proves the result.
\end{proof}

\begin{proposition}[Severe-misranking neighborhoods persist under local score stability]
\label{prop:app-bad-neighborhood}
Fix \(M>0\) and \(\delta\in(0,1)\). Suppose there exist \(\bar\theta\), \(\rho>0\), and \(L_{\rm sc}>0\) such that
\begin{equation*}
\Pbb(m_{\bar\theta}(X^+,X^-)\le -(M+\rho))\ge1-\delta
\end{equation*}
and
\begin{equation*}
|m_\theta(x^+,x^-)-m_{\bar\theta}(x^+,x^-)|
\le
L_{\rm sc}\cdot\|\theta-\bar\theta\|_2
\end{equation*}
for all relevant pairs. Then every \(\theta\) with
\begin{equation*}
\|\theta-\bar\theta\|_2\le \rho/L_{\rm sc}
\end{equation*}
satisfies
\begin{equation*}
\pi_M(\theta)\ge1-\delta.
\end{equation*}
\end{proposition}

\begin{proof}
On the event \(m_{\bar\theta}(X^+,X^-)\le -(M+\rho)\), the score-stability bound gives
\begin{equation*}
m_\theta(X^+,X^-)
\le
-(M+\rho)+L_{\rm sc}\cdot\|\theta-\bar\theta\|_2
\le
-M.
\end{equation*}
Thus the severe event at \(\bar\theta\) is contained in the severe event at \(\theta\). Taking probabilities proves the claim.
\end{proof}

\begin{proposition}[Single-coordinate degenerate high-error near-stationary regions are bad]
\label{prop:app-degenerate-near-stationary-bad-region}
Let \(\mathcal U\subset\R^d\) be a neighborhood and fix \(M,\rho>0\), a unit direction \(\bm v\), and constants \(\varepsilon_{\rm drift},\varepsilon_1,c_2\ge0\) with \(c_2>0\). Assume
\begin{equation*}
\inf_{\theta\in\mathcal U}\pi_M(\theta)\ge\rho,
\qquad
\sup_{\theta\in\mathcal U}
|\langle\nabla R_\ell(\theta),\bm v\rangle|
\le
\varepsilon_{\rm drift},
\end{equation*}
\begin{equation*}
\sup_{\theta\in\mathcal U}
\left\{
\Var(\psi_{\theta,\bm v}^{(1,0)}(X^+))
+
\Var(\psi_{\theta,\bm v}^{(0,1)}(X^-))
\right\}
\le
\varepsilon_1,
\qquad
\inf_{\theta\in\mathcal U}
\Var(\psi_{\theta,\bm v}^{(1,1)}(X^+,X^-))
\ge c_2.
\end{equation*}
Then \(\mathcal U\) is bad for SGD in the specified correction direction \(\bm v\): severe misrankings persist, deterministic correction along \(\bm v\) is bounded by \(\varepsilon_{\rm drift}\), and the first-order Hoeffding contribution to the projected variance is at most \(\varepsilon_1/B\). This fixed-direction statement is the form used in the projected escape-time argument. When the steepest direction \(\bm v_{\rm sev}(\theta)\) varies with \(\theta\), Proposition~\ref{prop:app-cone-bad-region} gives the corresponding cone-level uniform version. If \(\varepsilon_1=0\), then
\begin{equation*}
\Var(\langle\widehat{\bm g}_B^\ell(\theta),\bm v\rangle)
=
\frac{
\Var(\psi_{\theta,\bm v}^{(1,1)}(X^+,X^-))
}{B^2}
\qquad
\text{for all }\theta\in\mathcal U.
\end{equation*}
\end{proposition}

\begin{proof}
The first assumption says exactly that \(\mathcal U\) is high-error at level \((M,\rho)\). The second assumption gives the stated projected-drift bound. The variance split in Theorem~\ref{thm:app-hoeffding} gives, for each \(\theta\in\mathcal U\),
\begin{equation*}
\Var(\langle\widehat{\bm g}_B^\ell(\theta),\bm v\rangle)
=
\frac{
\Var(\psi_{\theta,\bm v}^{(1,0)}(X^+))
+
\Var(\psi_{\theta,\bm v}^{(0,1)}(X^-))
}{B}
+
\frac{
\Var(\psi_{\theta,\bm v}^{(1,1)}(X^+,X^-))
}{B^2}.
\end{equation*}
The first-order contribution is therefore at most \(\varepsilon_1/B\), while the residual pair-interaction component remains nontrivial because its variance is at least \(c_2\). If \(\varepsilon_1=0\), the \(B^{-1}\) term vanishes identically and the displayed \(B^{-2}\) identity follows.
\end{proof}

\begin{corollary}[Degenerate stationary neighborhoods are bad in a correction direction]
\label{cor:app-degenerate-saddle-bad-region}
Let \(\theta_\star\) be a stationary point of \(R_\ell\), let \(\mathcal U=B(\theta_\star,r)\), and fix \(M,\rho>0\), a unit direction \(\bm v\), and constants \(\varepsilon_1,c_2\ge0\) with \(c_2>0\). Assume \(\nabla R_\ell\) is \(L_R\)-Lipschitz on \(\mathcal U\). If
\begin{equation*}
\inf_{\theta\in\mathcal U}\pi_M(\theta)\ge\rho,
\end{equation*}
\begin{equation*}
\sup_{\theta\in\mathcal U}
\left\{
\Var(\psi_{\theta,\bm v}^{(1,0)}(X^+))
+
\Var(\psi_{\theta,\bm v}^{(0,1)}(X^-))
\right\}
\le
\varepsilon_1,
\qquad
\inf_{\theta\in\mathcal U}
\Var(\psi_{\theta,\bm v}^{(1,1)}(X^+,X^-))
\ge c_2>0,
\end{equation*}
then \(\mathcal U\) is bad for SGD in the sense of Proposition~\ref{prop:app-degenerate-near-stationary-bad-region}, with
\begin{equation*}
|\langle\nabla R_\ell(\theta),\bm v\rangle|
\le
L_R\cdot r
\qquad
\text{for every }\theta\in\mathcal U.
\end{equation*}
\end{corollary}

\begin{proof}
Since \(\theta_\star\) is stationary, \(\nabla R_\ell(\theta_\star)=\bm 0\). For every \(\theta\in B(\theta_\star,r)\),
\begin{equation*}
\|\nabla R_\ell(\theta)\|_2
=
\|\nabla R_\ell(\theta)-\nabla R_\ell(\theta_\star)\|_2
\le
L_R\cdot\|\theta-\theta_\star\|_2
\le
L_R\cdot r.
\end{equation*}
Hence
\begin{equation*}
|\langle\nabla R_\ell(\theta),\bm v\rangle|
\le
L_R\cdot r.
\end{equation*}
Apply Proposition~\ref{prop:app-degenerate-near-stationary-bad-region} with \(\varepsilon_{\rm drift}=L_R\cdot r\).
\end{proof}

\section{Surrogate Weight Calculations}
\label{app:surrogate-weight-calculations}

This section records the surrogate-weight calculations used by Section~\ref{sec:surrogate-weights}. The comparison is local and weight-based. It asks how the pairwise weight
\begin{equation*}
w_\ell(x)=-\ell'(x)
\end{equation*}
changes as a positive-negative margin \(x\) moves from severe violation to easy separation. The five canonical profiles studied in the main text are logistic, exponential, square, hinge, and squared-hinge.

\subsection{Canonical Weights}
\label{app:canonical-surrogate-weights}

\begin{proposition}[Canonical pairwise weights]
\label{prop:app-canonical-weights}
For
\begin{equation*}
\ell_{\log}(x)=\log(1+e^{-x}),
\qquad
\ell_{\exp}(x)=e^{-x},
\qquad
\ell_{\mathrm{sq},m}(x)=(m-x)^2,
\end{equation*}
\begin{equation*}
\ell_{\mathrm h,m}(x)=(m-x)_+,
\qquad
\ell_{\mathrm{sqh},m}(x)=(m-x)_+^2,
\end{equation*}
the corresponding weights are
\begin{equation*}
w_{\log}(x)=\frac{1}{1+e^x},
\qquad
w_{\exp}(x)=e^{-x},
\qquad
w_{\mathrm{sq},m}(x)=2\cdot(m-x),
\end{equation*}
\begin{equation*}
w_{\mathrm h,m}(x)=\1_{\{x<m\}},
\qquad
w_{\mathrm{sqh},m}(x)=2\cdot(m-x)_+.
\end{equation*}
For hinge, the boundary subgradient at \(x=m\) is chosen to be zero.
\end{proposition}

\begin{proof}
For logistic,
\begin{equation*}
\ell_{\log}'(x)
=
\frac{-e^{-x}}{1+e^{-x}}
=
-\frac{1}{1+e^x},
\end{equation*}
and hence \(w_{\log}(x)=1/(1+e^x)\). For exponential,
\begin{equation*}
\ell_{\exp}'(x)=-e^{-x},
\end{equation*}
so \(w_{\exp}(x)=e^{-x}\). For square,
\begin{equation*}
\ell_{\mathrm{sq},m}'(x)
=
2\cdot(m-x)\cdot(-1)
=
-2\cdot(m-x),
\end{equation*}
so \(w_{\mathrm{sq},m}(x)=2\cdot(m-x)\).

For hinge, \((m-x)_+\) has derivative \(-1\) on \(x<m\) and derivative \(0\) on \(x>m\). With the stated boundary convention, \(w_{\mathrm h,m}(x)=\1_{\{x<m\}}\). For squared-hinge, the derivative is \(-2\cdot(m-x)\) on \(x<m\) and \(0\) on \(x\ge m\), which gives \(w_{\mathrm{sqh},m}(x)=2\cdot(m-x)_+\).
\end{proof}

\subsection{Active-Side Correction Versus Class-Mean Collapse}
\label{app:active-side-surrogate-geometry}

The next statements describe what happens on a batch where all pairs are active, \(m_{ij}<m\). Define
\begin{equation*}
\bar s_+=\frac{1}{B}\sum_{i=1}^B f_\theta(x_i^+),
\qquad
\bar s_-=\frac{1}{B}\sum_{j=1}^B f_\theta(x_j^-),
\end{equation*}
\begin{equation*}
\bar{\bm g}_+=\frac{1}{B}\sum_{i=1}^B\bm g_i^+,
\qquad
\bar{\bm g}_-=\frac{1}{B}\sum_{j=1}^B\bm g_j^-,
\qquad
\bm d_B=\bar{\bm g}_- - \bar{\bm g}_+,
\end{equation*}
and the centered score-gradient terms
\begin{equation*}
\bm c_+=\frac{1}{B}\sum_{i=1}^B
\bigl(f_\theta(x_i^+)-\bar s_+\bigr)\cdot\bm g_i^+,
\qquad
\bm c_-=\frac{1}{B}\sum_{j=1}^B
\bigl(f_\theta(x_j^-)-\bar s_-\bigr)\cdot\bm g_j^-,
\end{equation*}
\begin{equation*}
\bm c_B=\bm c_+ + \bm c_-.
\end{equation*}
The vector \(\bm d_B\) is the class-mean gradient direction, while \(\bm c_B\) records within-class score-gradient variation.

\begin{proposition}[Hinge exact decoupling on the active region]
\label{prop:app-hinge-decoupling}
If all current pairs satisfy \(m_{ij}<m\), then the hinge mini-batch gradient satisfies
\begin{equation*}
\widehat{\bm g}_B^{\mathrm h,m}(\theta)
=
\bar{\bm g}_- - \bar{\bm g}_+,
\qquad
\bar{\bm g}_+
=
\frac{1}{B}\sum_{i=1}^B\bm g_i^+,
\quad
\bar{\bm g}_-
=
\frac{1}{B}\sum_{j=1}^B\bm g_j^-.
\end{equation*}
\end{proposition}

\begin{proof}
If \(m_{ij}<m\) for every pair, then \(w_{\mathrm h,m}(m_{ij})=1\). Substituting this into Eq.~\eqref{eq:app-pairwise-gradient} gives
\begin{equation*}
\widehat{\bm g}_B^{\mathrm h,m}(\theta)
=
\frac{1}{B^2}\sum_{i=1}^B\sum_{j=1}^B
(\bm g_j^- - \bm g_i^+)
=
\bar{\bm g}_- - \bar{\bm g}_+.
\end{equation*}
\end{proof}

\begin{lemma}[Affine active-side decomposition]
\label{lem:app-affine-active-decomposition}
Suppose \(w_\ell(x)=a\cdot x+b\) on the current batch. Then
\begin{equation*}
\widehat{\bm g}_B^\ell(\theta)
=
\bigl(a\cdot(\bar s_+-\bar s_-)+b\bigr)\cdot \bm d_B
-a\cdot\bm c_B.
\end{equation*}
\end{lemma}

\begin{proof}
Split the weight into constant and affine parts. The constant part gives
\begin{equation*}
\frac{b}{B^2}\sum_{i=1}^B\sum_{j=1}^B(\bm g_j^- - \bm g_i^+)
=
b\cdot\bm d_B.
\end{equation*}
For the affine part, using \(m_{ij}=f_\theta(x_i^+)-f_\theta(x_j^-)\),
\begin{align*}
\frac{1}{B^2}\sum_{i,j}m_{ij}\cdot(\bm g_j^- - \bm g_i^+)
&=
\frac{1}{B^2}\sum_{i,j}f_\theta(x_i^+)\cdot\bm g_j^-
-\frac{1}{B^2}\sum_{i,j}f_\theta(x_i^+)\cdot\bm g_i^+ \\
&\quad
-\frac{1}{B^2}\sum_{i,j}f_\theta(x_j^-)\cdot\bm g_j^-
+\frac{1}{B^2}\sum_{i,j}f_\theta(x_j^-)\cdot\bm g_i^+ \\
&=
\bar s_+\cdot\bar{\bm g}_-
-\frac{1}{B}\sum_i f_\theta(x_i^+)\cdot\bm g_i^+
-\frac{1}{B}\sum_j f_\theta(x_j^-)\cdot\bm g_j^- \\
&\quad
+\bar s_-\cdot\bar{\bm g}_+.
\end{align*}
By the definitions of \(\bm c_+\) and \(\bm c_-\),
\begin{equation*}
\frac{1}{B}\sum_i f_\theta(x_i^+)\cdot\bm g_i^+
=
\bm c_+ + \bar s_+\cdot\bar{\bm g}_+,
\qquad
\frac{1}{B}\sum_j f_\theta(x_j^-)\cdot\bm g_j^-
=
\bm c_- + \bar s_-\cdot\bar{\bm g}_-.
\end{equation*}
Substitution gives
\begin{equation*}
\frac{1}{B^2}\sum_{i,j}m_{ij}\cdot(\bm g_j^- - \bm g_i^+)
=
(\bar s_+-\bar s_-)\cdot(\bar{\bm g}_- - \bar{\bm g}_+)
-(\bm c_+ + \bm c_-),
\end{equation*}
or equivalently
\begin{equation*}
\frac{1}{B^2}\sum_{i,j}m_{ij}\cdot(\bm g_j^- - \bm g_i^+)
=
(\bar s_+-\bar s_-)\cdot\bm d_B-\bm c_B.
\end{equation*}
Multiplying this identity by \(a\) and adding the constant part proves the claim.
\end{proof}

\begin{corollary}[Squared-hinge retains within-class correction]
\label{cor:app-sqh-active-correction}
If all current pairs satisfy \(m_{ij}<m\), then
\begin{equation*}
\widehat{\bm g}_B^{\mathrm{sqh},m}(\theta)
=
2\cdot(m-\bar s_+ + \bar s_-)\cdot\bm d_B
+2\cdot\bm c_B.
\end{equation*}
Consequently,
\begin{equation*}
\dist\!\left(
\widehat{\bm g}_B^{\mathrm{sqh},m}(\theta),
\operatorname{span}\{\bm d_B\}
\right)
=
2\cdot
\dist\!\left(
\bm c_B,
\operatorname{span}\{\bm d_B\}
\right).
\end{equation*}
\end{corollary}

\begin{proof}
On the active region, squared-hinge has
\begin{equation*}
w_{\mathrm{sqh},m}(x)=2\cdot(m-x)=-2\cdot x+2\cdot m.
\end{equation*}
Lemma~\ref{lem:app-affine-active-decomposition} applies with \(a=-2\) and \(b=2\cdot m\), giving the displayed decomposition.

Let \(L_\theta=\operatorname{span}\{\bm d_B\}\). The first term in the decomposition lies in \(L_\theta\), so the distance from \(\widehat{\bm g}_B^{\mathrm{sqh},m}(\theta)\) to \(L_\theta\) is exactly the distance from \(2\cdot\bm c_B\) to \(L_\theta\). This proves the distance identity, including the case \(\bm d_B=\bm 0\), where \(L_\theta=\{\bm 0\}\).
\end{proof}

\begin{proposition}[Logistic near-decoupling on severe tails]
\label{prop:app-logistic-decoupling}
Assume
\begin{equation*}
\|\bm g_i^+\|_2\le L_B(\theta),
\qquad
\|\bm g_j^-\|_2\le L_B(\theta)
\end{equation*}
for all \(i,j\), and define
\begin{equation*}
q_{B,M}(\theta)
=
\frac{1}{B^2}\sum_{i=1}^B\sum_{j=1}^B
\1_{\{m_{ij}(\theta)\le -M\}}.
\end{equation*}
Then
\begin{equation*}
\left\|
\widehat{\bm g}_B^{\log}(\theta)
-
(\bar{\bm g}_- - \bar{\bm g}_+)
\right\|_2
\le
2\cdot L_B(\theta)
\cdot
\{1-q_{B,M}(\theta)\cdot(1-e^{-M})\}.
\end{equation*}
\end{proposition}

\begin{proof}
Let \(\delta_{ij}=1-w_{\log}(m_{ij})\). Since
\begin{equation*}
w_{\log}(x)=\frac{1}{1+e^x},
\end{equation*}
we have
\begin{equation*}
\widehat{\bm g}_B^{\log}(\theta)-(\bar{\bm g}_- - \bar{\bm g}_+)
=
-\frac{1}{B^2}\sum_{i=1}^B\sum_{j=1}^B
\delta_{ij}\cdot(\bm g_j^- - \bm g_i^+).
\end{equation*}
Taking norms and using
\begin{equation*}
\|\bm g_j^- - \bm g_i^+\|_2\le2\cdot L_B(\theta)
\end{equation*}
gives
\begin{equation*}
\left\|
\widehat{\bm g}_B^{\log}(\theta)-(\bar{\bm g}_- - \bar{\bm g}_+)
\right\|_2
\le
\frac{2\cdot L_B(\theta)}{B^2}
\sum_{i=1}^B\sum_{j=1}^B\delta_{ij}.
\end{equation*}
If \(m_{ij}\le -M\), then
\begin{equation*}
\delta_{ij}
=
\frac{e^{m_{ij}}}{1+e^{m_{ij}}}
\le
e^{-M}.
\end{equation*}
If \(m_{ij}>-M\), then \(\delta_{ij}\le1\). Hence
\begin{equation*}
\frac{1}{B^2}\sum_{i,j}\delta_{ij}
\le
q_{B,M}(\theta)\cdot e^{-M}
+
(1-q_{B,M}(\theta))
=
1-q_{B,M}(\theta)\cdot(1-e^{-M}).
\end{equation*}
Substituting this bound proves the result.
\end{proof}

\subsection{Easy-Side Silence Versus Overshoot}
\label{app:easy-side-surrogate-geometry}

\begin{proposition}[Active-linear easy-side silent profile of squared-hinge within the studied family]
\label{prop:app-sqh-unique-profile}
Among logistic, exponential, square, hinge, and squared-hinge, squared-hinge has a weight that is affine with nonzero slope on the whole active region \(x<m\) and identically zero on the whole easy region \(x\ge m\).
\end{proposition}

\begin{proof}
Use the explicit weights in Proposition~\ref{prop:app-canonical-weights}. Logistic has
\begin{equation*}
w_{\log}(x)=\frac{1}{1+e^x},
\end{equation*}
which is not affine on \(x<m\) and is strictly positive for every \(x\). Exponential has
\begin{equation*}
w_{\exp}(x)=e^{-x},
\end{equation*}
which is also non-affine on \(x<m\) and strictly positive everywhere. Hinge has
\begin{equation*}
w_{\mathrm h,m}(x)=\1_{\{x<m\}},
\end{equation*}
so it is silent on \(x\ge m\), but it is constant on the active side and has zero active-side slope. Square has
\begin{equation*}
w_{\mathrm{sq},m}(x)=2\cdot(m-x),
\end{equation*}
which is affine with nonzero slope and remains nonzero on the easy side. Squared-hinge has
\begin{equation*}
w_{\mathrm{sqh},m}(x)=2\cdot(m-x)_+,
\end{equation*}
which equals \(2\cdot(m-x)\) on \(x<m\) and \(0\) on \(x\ge m\). It is therefore the only member of the five with both properties.
\end{proof}

\begin{proposition}[Exact square-versus-squared-hinge separation]
\label{prop:app-square-sqh-separation}
Square and squared-hinge satisfy
\begin{equation*}
\widehat{\bm g}_B^{\mathrm{sq},m}(\theta)
=
\widehat{\bm g}_B^{\mathrm{sqh},m}(\theta)
-\frac{2}{B^2}\sum_{i,j}(m_{ij}-m)_+\cdot(\bm g_j^- - \bm g_i^+).
\end{equation*}
Consequently, if
\begin{equation*}
\|\bm g_i^+\|_2\le L_B(\theta),
\qquad
\|\bm g_j^-\|_2\le L_B(\theta)
\end{equation*}
for all \(i,j\), then
\begin{equation*}
\left\|
\widehat{\bm g}_B^{\mathrm{sq},m}(\theta)
-
\widehat{\bm g}_B^{\mathrm{sqh},m}(\theta)
\right\|_2
\le
4\cdot L_B(\theta)\cdot
\frac{1}{B^2}\sum_{i,j}(m_{ij}-m)_+.
\end{equation*}
\end{proposition}

\begin{proof}
For every scalar \(x\),
\begin{equation*}
m-x=(m-x)_+-(x-m)_+.
\end{equation*}
Multiplying by \(2\) gives
\begin{equation*}
w_{\mathrm{sq},m}(x)
=
2\cdot(m-x)
=
2\cdot(m-x)_+-2\cdot(x-m)_+
=
w_{\mathrm{sqh},m}(x)-2\cdot(x-m)_+.
\end{equation*}
Substituting this identity into Eq.~\eqref{eq:app-pairwise-gradient} proves the exact separation formula.

For the norm bound, apply the triangle inequality:
\begin{equation*}
\left\|
\widehat{\bm g}_B^{\mathrm{sq},m}(\theta)
-
\widehat{\bm g}_B^{\mathrm{sqh},m}(\theta)
\right\|_2
\le
\frac{2}{B^2}\sum_{i,j}(m_{ij}-m)_+\cdot
\|\bm g_j^- - \bm g_i^+\|_2.
\end{equation*}
The assumed bounds imply
\begin{equation*}
\|\bm g_j^- - \bm g_i^+\|_2
\le
\|\bm g_j^-\|_2+\|\bm g_i^+\|_2
\le
2\cdot L_B(\theta),
\end{equation*}
which gives the displayed inequality.
\end{proof}

\subsection{Hard-Tail Alignment}
\label{app:hard-tail-surrogate-geometry}

\begin{proposition}[Hard-pair reweighting increments]
\label{prop:app-hard-tail-increments}
Fix \(t>0\). Then
\begin{equation*}
\lim_{x\to-\infty}
\bigl(w_{\log}(x-t)-w_{\log}(x)\bigr)
=0,
\end{equation*}
\begin{equation*}
w_{\mathrm{sqh},m}(x-t)-w_{\mathrm{sqh},m}(x)=2\cdot t
\qquad
\text{whenever }x<m-t,
\end{equation*}
and
\begin{equation*}
w_{\exp}(x-t)-w_{\exp}(x)
=
(e^t-1)\cdot e^{-x}
\to\infty
\qquad
\text{as }x\to-\infty.
\end{equation*}
\end{proposition}

\begin{proof}
For logistic, \(w_{\log}(x)=1/(1+e^x)\), and both \(w_{\log}(x)\) and \(w_{\log}(x-t)\) converge to \(1\) as \(x\to-\infty\). This proves the first limit.

For squared-hinge, if \(x<m-t\), then both \(x\) and \(x-t\) are active. Hence
\begin{equation*}
w_{\mathrm{sqh},m}(x-t)-w_{\mathrm{sqh},m}(x)
=
2\cdot(m-x+t)-2\cdot(m-x)
=
2\cdot t.
\end{equation*}
For exponential,
\begin{equation*}
w_{\exp}(x-t)-w_{\exp}(x)
=
e^{-(x-t)}-e^{-x}
=
(e^t-1)\cdot e^{-x},
\end{equation*}
which diverges because \(t>0\) and \(e^{-x}\to\infty\).
\end{proof}

\begin{proposition}[Loss-gradient coupling in the hard tail]
\label{prop:app-loss-gradient-coupling}
The hard-tail ratios satisfy
\begin{equation*}
\frac{w_{\log}(x)^2}{\ell_{\log}(x)}\to0
\qquad
\text{as }x\to-\infty,
\end{equation*}
\begin{equation*}
\frac{w_{\mathrm{sqh},m}(x)^2}{\ell_{\mathrm{sqh},m}(x)}=4
\qquad
\text{for every }x<m,
\end{equation*}
and
\begin{equation*}
\frac{w_{\exp}(x)^2}{\ell_{\exp}(x)}
=e^{-x}\to\infty
\qquad
\text{as }x\to-\infty.
\end{equation*}
\end{proposition}

\begin{proof}
For logistic, \(w_{\log}(x)\to1\) as \(x\to-\infty\), while
\begin{equation*}
\ell_{\log}(x)=\log(1+e^{-x})\sim -x\to\infty.
\end{equation*}
Thus \(w_{\log}(x)^2/\ell_{\log}(x)\to0\). For squared-hinge and \(x<m\),
\begin{equation*}
w_{\mathrm{sqh},m}(x)=2\cdot(m-x),
\qquad
\ell_{\mathrm{sqh},m}(x)=(m-x)^2,
\end{equation*}
so the ratio is \(4\). For exponential,
\begin{equation*}
\frac{w_{\exp}(x)^2}{\ell_{\exp}(x)}
=
\frac{e^{-2x}}{e^{-x}}
=
e^{-x},
\end{equation*}
which diverges as \(x\to-\infty\).
\end{proof}

\subsection{Projected Stochastic Separation}
\label{app:projected-stochastic-surrogate-separation}

\begin{proposition}[Variance separation example: hinge versus squared-hinge]
\label{prop:app-hinge-sqh-variance-separation}
Fix \(\theta\) and a unit vector \(\bm v\). Assume all sampled pairs satisfy \(m_{ij}<m\) almost surely, and there exist constants \(c_+\) and \(c_-\) such that
\begin{equation*}
\langle\nabla_\theta f_\theta(X^+),\bm v\rangle=c_+
\quad\text{a.s.},
\qquad
\langle\nabla_\theta f_\theta(X^-),\bm v\rangle=c_-
\quad\text{a.s.}
\end{equation*}
Then
\begin{equation*}
\Var\!\left(
\langle\widehat{\bm g}_B^{\mathrm h,m}(\theta),\bm v\rangle
\right)
=0,
\end{equation*}
whereas
\begin{equation*}
\Var\!\left(
\langle\widehat{\bm g}_B^{\mathrm{sqh},m}(\theta),\bm v\rangle
\right)
=
\frac{
4\cdot(c_- - c_+)^2
}{B}
\cdot
\left\{
\Var(f_\theta(X^+))+\Var(f_\theta(X^-))
\right\}.
\end{equation*}
\end{proposition}

\begin{proof}
By Proposition~\ref{prop:app-hinge-decoupling},
\begin{equation*}
\widehat{\bm g}_B^{\mathrm h,m}(\theta)=\bar{\bm g}_- - \bar{\bm g}_+.
\end{equation*}
Projecting onto \(\bm v\) gives
\begin{equation*}
\langle\widehat{\bm g}_B^{\mathrm h,m}(\theta),\bm v\rangle
=
c_- - c_+,
\end{equation*}
which is deterministic, so the hinge variance is zero.

For squared-hinge, Corollary~\ref{cor:app-sqh-active-correction} gives
\begin{equation*}
\widehat{\bm g}_B^{\mathrm{sqh},m}(\theta)
=
2\cdot(m-\bar s_+ + \bar s_-)\cdot\bm d_B
+2\cdot\bm c_B.
\end{equation*}
The constant projected-gradient assumptions imply
\begin{equation*}
\langle\bm d_B,\bm v\rangle=c_- - c_+.
\end{equation*}
They also imply
\begin{equation*}
\langle\bm c_B,\bm v\rangle
=
\frac{1}{B}\sum_i
\bigl(f_\theta(X_i^+)-\bar s_+\bigr)\cdot c_+
+
\frac{1}{B}\sum_j
\bigl(f_\theta(X_j^-)-\bar s_-\bigr)\cdot c_-
=0,
\end{equation*}
because both centered sums vanish. Hence
\begin{equation*}
\langle\widehat{\bm g}_B^{\mathrm{sqh},m}(\theta),\bm v\rangle
=
2\cdot(c_- - c_+)\cdot(m-\bar s_+ + \bar s_-).
\end{equation*}
The constant \(m\) does not affect variance, and the positive and negative batches are independent. Therefore
\begin{equation*}
\Var(m-\bar s_+ + \bar s_-)
=
\Var(\bar s_+)+\Var(\bar s_-)
=
\frac{\Var(f_\theta(X^+))+\Var(f_\theta(X^-))}{B}.
\end{equation*}
Multiplying by \(4\cdot(c_- - c_+)^2\) proves the displayed formula.
\end{proof}

\section{Changing Where Pairwise AUC Optimization Begins}
\label{app:starting-region}

This section gives the formal statements behind Section~\ref{sec:warmup-pretraining}. The results have a common structure. CE warm-up can change the starting point of the AUC phase by reducing severe positive-negative misrankings. Parameter-efficient fine-tuning can reduce the need for a separate warm-up stage when the trainable subspace starts in, and remains inside, a favorable neighborhood where severe misranking is controlled and first-order projected stochastic motion persists over the finite training horizon.

\subsection{CE Warm-Up and Severe-Misranking Mass}
\label{app:ce-warmup-start}

\begin{proposition}[Pointwise CE risk controls severe-misranking mass]
\label{prop:app-ce-risk-severe-mass}
Define the pointwise logistic risk
\begin{equation*}
R_{\mathrm{CE}}(\theta)
=
\E[\log(1+e^{-f_\theta(X^+)})]
+
\E[\log(1+e^{f_\theta(X^-)})].
\end{equation*}
Then, for every \(M>0\),
\begin{equation*}
\pi_M(\theta)
\le
\frac{R_{\mathrm{CE}}(\theta)}{\log(1+e^{M/2})}.
\end{equation*}
\end{proposition}

\begin{proof}
The severe-misranking event is
\begin{equation*}
A
=
\{m_\theta(X^+,X^-)\le -M\}
=
\{f_\theta(X^-)-f_\theta(X^+)\ge M\}.
\end{equation*}
Let
\begin{equation*}
B_+=\{f_\theta(X^+)\le -M/2\},
\qquad
B_-=\{f_\theta(X^-)\ge M/2\}.
\end{equation*}
If \(A\) occurs and neither \(B_+\) nor \(B_-\) occurs, then
\begin{equation*}
f_\theta(X^+)>-M/2,
\qquad
f_\theta(X^-)<M/2,
\end{equation*}
so \(f_\theta(X^-)-f_\theta(X^+)<M\), contradicting \(A\). Therefore \(A\subseteq B_+\cup B_-\), and
\begin{equation*}
\pi_M(\theta)\le \Pbb(B_+)+\Pbb(B_-).
\end{equation*}
On \(B_+\),
\begin{equation*}
\log(1+e^{-f_\theta(X^+)})\ge \log(1+e^{M/2}),
\end{equation*}
and Markov's inequality gives
\begin{equation*}
\Pbb(B_+)
\le
\frac{\E[\log(1+e^{-f_\theta(X^+)})]}{\log(1+e^{M/2})}.
\end{equation*}
The same argument gives
\begin{equation*}
\Pbb(B_-)
\le
\frac{\E[\log(1+e^{f_\theta(X^-)})]}{\log(1+e^{M/2})}.
\end{equation*}
Adding the two inequalities proves the claim.
\end{proof}

\begin{corollary}[Small CE risk implies low severe-misranking mass]
\label{cor:app-ce-small-risk-low-severe-mass}
If a warm-up endpoint \(\theta_{\rm wu}\) satisfies
\begin{equation*}
R_{\mathrm{CE}}(\theta_{\rm wu})
\le
\delta\cdot\log(1+e^{M/2}),
\end{equation*}
then
\begin{equation*}
\pi_M(\theta_{\rm wu})\le \delta.
\end{equation*}
\end{corollary}

\begin{proof}
Apply Proposition~\ref{prop:app-ce-risk-severe-mass} at \(\theta_{\rm wu}\) and substitute the assumed upper bound on \(R_{\mathrm{CE}}(\theta_{\rm wu})\).
\end{proof}

\begin{proposition}[Finite-step CE warm-up under local PL and smoothness]
\label{prop:app-ce-finite-step-warmup}
Assume \(L_{\mathrm{CE}}>0\), \(\mu_{\mathrm{CE}}>0\), \(R_{\mathrm{CE}}\) has \(L_{\mathrm{CE}}\)-Lipschitz gradient on a region \(\mathcal U\), and the Polyak-Lojasiewicz condition
\begin{equation*}
\frac{1}{2}\cdot\|\nabla R_{\mathrm{CE}}(\theta)\|^2
\ge
\mu_{\mathrm{CE}}\cdot
\bigl(R_{\mathrm{CE}}(\theta)-R_{\mathrm{CE}}^\star\bigr)
\qquad
(\theta\in\mathcal U),
\end{equation*}
where \(R_{\mathrm{CE}}^\star=\inf_{\theta\in\mathcal U}R_{\mathrm{CE}}(\theta)\). Consider the fixed-step CE dynamics
\begin{equation*}
\theta_{t+1}
=
\theta_t-\alpha\cdot\nabla R_{\mathrm{CE}}(\theta_t),
\qquad
0<\alpha\le \min\{1/L_{\mathrm{CE}},1/\mu_{\mathrm{CE}}\},
\end{equation*}
and suppose \(\theta_0,\ldots,\theta_T\in\mathcal U\). Then
\begin{equation*}
R_{\mathrm{CE}}(\theta_t)-R_{\mathrm{CE}}^\star
\le
(1-\alpha\cdot\mu_{\mathrm{CE}})^t
\cdot
\bigl(R_{\mathrm{CE}}(\theta_0)-R_{\mathrm{CE}}^\star\bigr)
\qquad
(0\le t\le T).
\end{equation*}
Consequently, if
\begin{equation*}
R_{\mathrm{CE}}^\star
+
(1-\alpha\cdot\mu_{\mathrm{CE}})^T
\cdot
\bigl(R_{\mathrm{CE}}(\theta_0)-R_{\mathrm{CE}}^\star\bigr)
\le
\delta\cdot\log(1+e^{M/2}),
\end{equation*}
then \(\pi_M(\theta_T)\le\delta\).
\end{proposition}

\begin{proof}
For \(t<T\), smoothness gives
\begin{equation*}
R_{\mathrm{CE}}(\theta_{t+1})
\le
R_{\mathrm{CE}}(\theta_t)
+
\langle\nabla R_{\mathrm{CE}}(\theta_t),\theta_{t+1}-\theta_t\rangle
+
\frac{L_{\mathrm{CE}}}{2}\cdot\|\theta_{t+1}-\theta_t\|^2.
\end{equation*}
Substituting the CE gradient step,
\begin{equation*}
\theta_{t+1}-\theta_t=-\alpha\cdot\nabla R_{\mathrm{CE}}(\theta_t),
\end{equation*}
gives
\begin{equation*}
R_{\mathrm{CE}}(\theta_{t+1})
\le
R_{\mathrm{CE}}(\theta_t)
-
\left(\alpha-\frac{L_{\mathrm{CE}}\cdot\alpha^2}{2}\right)
\cdot
\|\nabla R_{\mathrm{CE}}(\theta_t)\|^2.
\end{equation*}
Because \(\alpha\le1/L_{\mathrm{CE}}\),
\begin{equation*}
\alpha-\frac{L_{\mathrm{CE}}\cdot\alpha^2}{2}\ge \frac{\alpha}{2}.
\end{equation*}
The PL inequality therefore gives
\begin{equation*}
R_{\mathrm{CE}}(\theta_{t+1})-R_{\mathrm{CE}}^\star
\le
(1-\alpha\cdot\mu_{\mathrm{CE}})
\cdot
\bigl(R_{\mathrm{CE}}(\theta_t)-R_{\mathrm{CE}}^\star\bigr).
\end{equation*}
Iterating the one-step contraction proves the displayed risk bound. The final severe-misranking bound follows by evaluating the risk bound at \(T\) and applying Corollary~\ref{cor:app-ce-small-risk-low-severe-mass}.
\end{proof}

\begin{proposition}[Warm-up transfer to a low-error endpoint]
\label{prop:app-warmup-transfer-low-error}
Fix \(M'>0\) and \(r\ge0\).
Assume pairwise margins are locally \(L_{\rm sc}\)-Lipschitz on a neighborhood containing \(\bar\theta\) and \(\theta_{\rm wu}\):
\begin{equation*}
|m_\theta(x^+,x^-)-m_{\theta'}(x^+,x^-)|
\le
L_{\rm sc}\cdot\|\theta-\theta'\|
\end{equation*}
for all relevant \(\theta,\theta'\) and all pairs \((x^+,x^-)\). If
\begin{equation*}
\|\theta_{\rm wu}-\bar\theta\|\le r
\qquad\text{and}\qquad
\pi_{M'}(\bar\theta)\le\delta,
\end{equation*}
then
\begin{equation*}
\pi_{M'+L_{\rm sc}\cdot r}(\theta_{\rm wu})\le\delta.
\end{equation*}
\end{proposition}

\begin{proof}
On the event
\begin{equation*}
\{m_{\theta_{\rm wu}}(X^+,X^-)\le -(M'+L_{\rm sc}\cdot r)\},
\end{equation*}
the Lipschitz condition and \(\|\theta_{\rm wu}-\bar\theta\|\le r\) imply
\begin{equation*}
m_{\bar\theta}(X^+,X^-)
\le
m_{\theta_{\rm wu}}(X^+,X^-)
+
L_{\rm sc}\cdot\|\theta_{\rm wu}-\bar\theta\|
\le
-M'.
\end{equation*}
Thus
\begin{equation*}
\{m_{\theta_{\rm wu}}(X^+,X^-)\le -(M'+L_{\rm sc}\cdot r)\}
\subseteq
\{m_{\bar\theta}(X^+,X^-)\le -M'\}.
\end{equation*}
Taking probabilities proves the claim.
\end{proof}

\begin{proposition}[Post-warm-up low severe-misranking neighborhood]
\label{prop:app-post-warmup-neighborhood}
Fix \(M>0\), \(r\ge0\), and \(L_{\rm sc}\ge0\) such that \(M-L_{\rm sc}\cdot r>0\).
Assume that the pairwise margin map is locally \(L_{\rm sc}\)-Lipschitz around a warm-up endpoint \(\theta_{\rm wu}\):
\begin{equation*}
|m_\theta(x^+,x^-)-m_{\theta_{\rm wu}}(x^+,x^-)|
\le
L_{\rm sc}\cdot\|\theta-\theta_{\rm wu}\|
\end{equation*}
for all \(\theta\) with \(\|\theta-\theta_{\rm wu}\|\le r\) and all sample pairs. If
\begin{equation*}
\pi_{M-L_{\rm sc}\cdot r}(\theta_{\rm wu})\le\delta,
\end{equation*}
then every \(\theta\) satisfying \(\|\theta-\theta_{\rm wu}\|\le r\) also satisfies
\begin{equation*}
\pi_M(\theta)\le\delta.
\end{equation*}
\end{proposition}

\begin{proof}
If \(m_\theta(X^+,X^-)\le -M\) and \(\|\theta-\theta_{\rm wu}\|\le r\), then
\begin{equation*}
m_{\theta_{\rm wu}}(X^+,X^-)
\le
m_\theta(X^+,X^-)+L_{\rm sc}\cdot r
\le
-(M-L_{\rm sc}\cdot r).
\end{equation*}
Therefore
\begin{equation*}
\{m_\theta(X^+,X^-)\le -M\}
\subseteq
\{m_{\theta_{\rm wu}}(X^+,X^-)\le -(M-L_{\rm sc}\cdot r)\}.
\end{equation*}
Taking probabilities gives \(\pi_M(\theta)\le\delta\).
\end{proof}

\begin{corollary}[Finite-horizon post-warm-up neighborhood check]
\label{cor:app-post-warmup-finite-horizon}
Under the assumptions of Proposition~\ref{prop:app-post-warmup-neighborhood}, let
\begin{equation*}
\widehat\pi_{M-L_{\rm sc}\cdot r,K}(\theta_{\rm wu})
\end{equation*}
be an independent finite-pair estimator as in Proposition~\ref{prop:app-finite-pair-high-error-check}. Consider a second-stage trajectory
\begin{equation*}
\theta_0,\theta_1,\ldots,\theta_T,
\qquad
\theta_0=\theta_{\rm wu},
\end{equation*}
such that \(\|\theta_t-\theta_{\rm wu}\|\le r\) for all \(0\le t\le T\). Then, for every \(\varepsilon>0\), with probability at least
\begin{equation*}
1-\frac{1}{4\cdot K\cdot\varepsilon^2},
\end{equation*}
one has the implication
\begin{equation*}
\widehat\pi_{M-L_{\rm sc}\cdot r,K}(\theta_{\rm wu})+\varepsilon\le\delta
\quad\Longrightarrow\quad
\pi_M(\theta_t)\le\delta
\quad\text{for all }0\le t\le T.
\end{equation*}
\end{corollary}

\begin{proof}
By Proposition~\ref{prop:app-finite-pair-high-error-check}, with probability at least \(1-1/(4\cdot K\cdot\varepsilon^2)\),
\begin{equation*}
\pi_{M-L_{\rm sc}\cdot r}(\theta_{\rm wu})
\le
\widehat\pi_{M-L_{\rm sc}\cdot r,K}(\theta_{\rm wu})+\varepsilon.
\end{equation*}
On this event, the premise of the implication gives
\begin{equation*}
\pi_{M-L_{\rm sc}\cdot r}(\theta_{\rm wu})\le\delta.
\end{equation*}
Since every \(\theta_t\) lies within radius \(r\) of \(\theta_{\rm wu}\), Proposition~\ref{prop:app-post-warmup-neighborhood} yields \(\pi_M(\theta_t)\le\delta\) for all \(0\le t\le T\).
\end{proof}

\subsection{Parameter-Efficient Fine-Tuning in a Stable Trainable Subspace}
\label{app:peft-stable-subspace}

We now consider a pretrained model with a frozen base and a low-dimensional trainable subspace, such as a head, adapter, or LoRA module. Let \(\bm\xi\) denote the trainable coordinates and absorb the frozen base into the score map \(f_{\bm\xi}\). The results below state sufficient conditions under which direct pairwise AUC-surrogate training starts in a favorable neighborhood and remains there for the finite horizon of interest.

\begin{proposition}[Trainable-subspace transfer of pairwise identities]
\label{prop:app-trainable-subspace-transfer}
Let the trainable parameter be \(\bm\xi\), and let the frozen pretrained base be absorbed into \(f_{\bm\xi}\). If \(f_{\bm\xi}\) and \(\nabla_{\bm\xi}f_{\bm\xi}\) are locally well defined, then the pairwise gradient identity, Hoeffding variance split, and surrogate-weight identities from Appendices~\ref{app:pairwise-bad-region-dynamics} and~\ref{app:surrogate-weight-calculations} remain valid after replacing \(\theta\) by \(\bm\xi\).
\end{proposition}

\begin{proof}
Each identity used in those appendices depends only on the score map, its gradient with respect to the trainable variable, the positive margin
\begin{equation*}
m_{\bm\xi}(x^+,x^-)=f_{\bm\xi}(x^+)-f_{\bm\xi}(x^-),
\end{equation*}
and algebraic operations such as pair averaging, conditioning, variance decomposition, and applying \(w_\ell(x)=-\ell'(x)\). If the pretrained base is frozen, the same objects exist with \(\bm\xi\) as the trainable variable:
\begin{equation*}
f_\theta\mapsto f_{\bm\xi},
\qquad
\nabla_\theta f_\theta\mapsto \nabla_{\bm\xi}f_{\bm\xi}.
\end{equation*}
No derivation requires the frozen coordinates to be trainable. Therefore the same identities hold verbatim in the trainable subspace.
\end{proof}

\begin{proposition}[Score perturbation in a trainable subspace]
\label{prop:app-peft-score-perturbation}
Fix a trainable-subspace initialization \(\bm\xi_0\) and radius \(\rho>0\). Suppose
\begin{equation*}
\sup_{\bm\zeta\in B(\bm\xi_0,\rho)}\sup_x
\|\nabla_{\bm\xi} f_{\bm\zeta}(x)\|
\le
L_f.
\end{equation*}
Then, for any \(\bm\xi,\bm\xi'\in B(\bm\xi_0,\rho)\) and every input \(x\),
\begin{equation*}
|f_{\bm\xi}(x)-f_{\bm\xi'}(x)|
\le
L_f\cdot\|\bm\xi-\bm\xi'\|,
\end{equation*}
and for every pair,
\begin{equation*}
|m_{\bm\xi}(x^+,x^-)-m_{\bm\xi'}(x^+,x^-)|
\le
2\cdot L_f\cdot\|\bm\xi-\bm\xi'\|.
\end{equation*}
\end{proposition}

\begin{proof}
For a fixed \(x\), define
\begin{equation*}
\phi(t)=f_{\bm\xi'+t\cdot(\bm\xi-\bm\xi')}(x),
\qquad
t\in[0,1].
\end{equation*}
The line segment between \(\bm\xi'\) and \(\bm\xi\) lies in \(B(\bm\xi_0,\rho)\). By the fundamental theorem of calculus and Cauchy's inequality,
\begin{equation*}
\begin{aligned}
|f_{\bm\xi}(x)-f_{\bm\xi'}(x)|
&=
\left|
\int_0^1
\left\langle
\nabla_{\bm\xi}f_{\bm\xi'+t\cdot(\bm\xi-\bm\xi')}(x),
\bm\xi-\bm\xi'
\right\rangle
dt
\right| \\
&\le
\int_0^1 L_f\cdot\|\bm\xi-\bm\xi'\|\,dt
=
L_f\cdot\|\bm\xi-\bm\xi'\|.
\end{aligned}
\end{equation*}
The margin bound follows from
\begin{equation*}
m_{\bm\xi}(x^+,x^-)-m_{\bm\xi'}(x^+,x^-)
=
\bigl(f_{\bm\xi}(x^+)-f_{\bm\xi'}(x^+)\bigr)
-
\bigl(f_{\bm\xi}(x^-)-f_{\bm\xi'}(x^-)\bigr)
\end{equation*}
and the triangle inequality.
\end{proof}

\begin{proposition}[Uniform severe-misranking control in a stable trainable subspace]
\label{prop:app-peft-uniform-severe-control}
Under the assumptions of Proposition~\ref{prop:app-peft-score-perturbation}, suppose \(M>4\cdot L_f\cdot\rho\) and a reference point \(\bar{\bm\xi}\in B(\bm\xi_0,\rho)\) satisfies
\begin{equation*}
\pi_{M-4\cdot L_f\cdot\rho}(\bar{\bm\xi})\le \delta_M.
\end{equation*}
Then every \(\bm\xi\in B(\bm\xi_0,\rho)\) satisfies
\begin{equation*}
\pi_M(\bm\xi)\le\delta_M.
\end{equation*}
\end{proposition}

\begin{proof}
Fix \(\bm\xi\in B(\bm\xi_0,\rho)\). On the event \(m_{\bm\xi}(X^+,X^-)\le -M\), Proposition~\ref{prop:app-peft-score-perturbation} gives
\begin{equation*}
m_{\bar{\bm\xi}}(X^+,X^-)
\le
-M+2\cdot L_f\cdot\|\bm\xi-\bar{\bm\xi}\|.
\end{equation*}
Since both \(\bm\xi\) and \(\bar{\bm\xi}\) lie in \(B(\bm\xi_0,\rho)\),
\begin{equation*}
\|\bm\xi-\bar{\bm\xi}\|
\le
\|\bm\xi-\bm\xi_0\|+\|\bar{\bm\xi}-\bm\xi_0\|
\le
2\cdot\rho.
\end{equation*}
Therefore
\begin{equation*}
m_{\bar{\bm\xi}}(X^+,X^-)
\le
-(M-4\cdot L_f\cdot\rho).
\end{equation*}
Thus
\begin{equation*}
\{m_{\bm\xi}(X^+,X^-)\le -M\}
\subseteq
\{m_{\bar{\bm\xi}}(X^+,X^-)\le -(M-4\cdot L_f\cdot\rho)\}.
\end{equation*}
Taking probabilities proves the claim.
\end{proof}

\begin{proposition}[Persistence of first-order projected variance]
\label{prop:app-peft-first-order-persistence}
Fix a unit vector \(\bm v\). Suppose that, on a neighborhood \(\mathcal U\) of a trainable-subspace initialization, the maps
\begin{equation*}
\bm\xi\mapsto \psi_{\bm\xi,\bm v}^{(1,0)}(X^+),
\qquad
\bm\xi\mapsto \psi_{\bm\xi,\bm v}^{(0,1)}(X^-)
\end{equation*}
are continuous in \(L_2\). If, for some reference point \(\bar{\bm\xi}\in\mathcal U\),
\begin{equation*}
\Var(\psi_{\bar{\bm\xi},\bm v}^{(1,0)}(X^+))
+
\Var(\psi_{\bar{\bm\xi},\bm v}^{(0,1)}(X^-))
>0,
\end{equation*}
then there is a smaller neighborhood \(\mathcal U'\subseteq\mathcal U\) and a constant \(\sigma_{1,\bm v}^2>0\) such that
\begin{equation*}
\Var(\psi_{\bm\xi,\bm v}^{(1,0)}(X^+))
+
\Var(\psi_{\bm\xi,\bm v}^{(0,1)}(X^-))
\ge
\sigma_{1,\bm v}^2
\qquad
(\bm\xi\in\mathcal U').
\end{equation*}
\end{proposition}

\begin{proof}
Let
\begin{equation*}
V(\bm\xi)
=
\Var(\psi_{\bm\xi,\bm v}^{(1,0)}(X^+))
+
\Var(\psi_{\bm\xi,\bm v}^{(0,1)}(X^-)).
\end{equation*}
\(L_2\)-continuity of each projection implies continuity of its second moment and its mean, hence continuity of its variance. Therefore \(V\) is continuous. Since \(V(\bar{\bm\xi})>0\), there exists a neighborhood \(\mathcal U'\subseteq\mathcal U\) such that
\begin{equation*}
|V(\bm\xi)-V(\bar{\bm\xi})|
\le
\frac{V(\bar{\bm\xi})}{2}
\qquad
(\bm\xi\in\mathcal U').
\end{equation*}
Setting \(\sigma_{1,\bm v}^2=V(\bar{\bm\xi})/2\) proves the result.
\end{proof}

\begin{proposition}[Finite-horizon confinement from a bounded update budget]
\label{prop:app-peft-confinement}
Consider discrete dynamics
\begin{equation*}
\bm\xi_{t+1}
=
\bm\xi_t-\eta_t\cdot\widehat{\bm g}_t,
\qquad
t=0,\ldots,T-1.
\end{equation*}
Assume there are \(\rho>0\) and \(g_\ast<\infty\) such that \(B(\bm\xi_0,\rho)\subseteq\mathcal U\),
\begin{equation*}
\|\widehat{\bm g}_t\|\le g_\ast
\qquad
\text{whenever }\bm\xi_t\in\mathcal U,
\end{equation*}
and
\begin{equation*}
\sum_{t=0}^{T-1}\eta_t\cdot g_\ast\le\rho.
\end{equation*}
Then \(\bm\xi_t\in B(\bm\xi_0,\rho)\subseteq\mathcal U\) for every \(0\le t\le T\).
\end{proposition}

\begin{proof}
We prove by induction that \(\bm\xi_t\in B(\bm\xi_0,\rho)\). The claim is trivial at \(t=0\). If it holds up to time \(t\), then \(\bm\xi_s\in\mathcal U\) for \(0\le s\le t\), and hence
\begin{equation*}
\bm\xi_{t+1}-\bm\xi_0
=
-\sum_{s=0}^{t}\eta_s\cdot\widehat{\bm g}_s.
\end{equation*}
Taking norms gives
\begin{equation*}
\|\bm\xi_{t+1}-\bm\xi_0\|
\le
\sum_{s=0}^{t}\eta_s\cdot\|\widehat{\bm g}_s\|
\le
\sum_{s=0}^{T-1}\eta_s\cdot g_\ast
\le
\rho.
\end{equation*}
Thus \(\bm\xi_{t+1}\in B(\bm\xi_0,\rho)\), closing the induction.
\end{proof}

\begin{proposition}[Finite-horizon reduced need for warm-up]
\label{prop:app-peft-finite-horizon-reduced-warmup}
Suppose a neighborhood \(\mathcal U\) of the trainable-subspace initialization satisfies
\begin{equation*}
\sup_{\bm\xi\in\mathcal U}\pi_M(\bm\xi)\le\delta_M
\end{equation*}
and, for some unit vector \(\bm v\),
\begin{equation*}
\Var(\psi_{\bm\xi,\bm v}^{(1,0)}(X^+))
+
\Var(\psi_{\bm\xi,\bm v}^{(0,1)}(X^-))
\ge
\sigma_{1,\bm v}^2>0
\qquad
(\bm\xi\in\mathcal U).
\end{equation*}
If a finite-horizon trajectory remains in \(\mathcal U\), then every iterate on that horizon satisfies
\begin{equation*}
\pi_M(\bm\xi_t)\le\delta_M,
\qquad
\Var(\langle \widehat{\bm g}_B^\ell(\bm\xi_t),\bm v\rangle)
\ge
\frac{\sigma_{1,\bm v}^2}{B}.
\end{equation*}
\end{proposition}

\begin{proof}
The severe-misranking bound follows directly from the uniform hypothesis. For the variance bound, Proposition~\ref{prop:app-trainable-subspace-transfer} transfers the variance split in Eq.~\eqref{eq:app-var-split} to the trainable coordinate \(\bm\xi_t\):
\begin{equation*}
\Var(\langle \widehat{\bm g}_B^\ell(\bm\xi_t),\bm v\rangle)
=
\frac{
\Var(\psi_{\bm\xi_t,\bm v}^{(1,0)}(X^+))
+
\Var(\psi_{\bm\xi_t,\bm v}^{(0,1)}(X^-))
}{B}
+
\frac{
\Var(\psi_{\bm\xi_t,\bm v}^{(1,1)}(X^+,X^-))
}{B^2}.
\end{equation*}
The residual variance is nonnegative, and the first-order term is bounded below by \(\sigma_{1,\bm v}^2\) throughout \(\mathcal U\). The displayed lower bound follows.
\end{proof}

\begin{proposition}[Three checks for finite-horizon direct pairwise AUC-surrogate training from a pretrained start]
\label{prop:app-peft-three-checks}
A finite-horizon conclusion of the form in Proposition~\ref{prop:app-peft-finite-horizon-reduced-warmup} follows from three checks:
\begin{enumerate}[leftmargin=5mm]
\item \textbf{Favorable initialization.} Verify a neighborhood \(\mathcal U\) of the pretrained trainable-subspace initialization where severe-misranking mass is uniformly controlled and the first-order projected variance has a uniform positive lower bound.
\item \textbf{Finite-horizon confinement.} Verify that the direct pairwise AUC-surrogate trajectory remains in \(\mathcal U\), for example by Proposition~\ref{prop:app-peft-confinement}.
\item \textbf{Pointwise transfer.} Once each iterate lies in \(\mathcal U\), the severe-misranking and first-order variance bounds transfer to every iterate by Proposition~\ref{prop:app-peft-finite-horizon-reduced-warmup}.
\end{enumerate}
\end{proposition}

\begin{proof}
The first check supplies the two uniform hypotheses of Proposition~\ref{prop:app-peft-finite-horizon-reduced-warmup}. The second check supplies its trajectory containment assumption. Substituting each iterate into the uniform neighborhood bounds gives the pointwise severe-misranking control and first-order variance lower bound on the whole finite horizon.
\end{proof}

\subsection{Local Pairwise AUC-Surrogate Descent in a Favorable Subspace}
\label{app:peft-favorable-subspace}

\begin{proposition}[Local pairwise AUC-surrogate descent in a favorable trainable subspace]
\label{prop:app-peft-local-auc-descent}
Let
\begin{equation*}
R_\ell(\bm\xi)=\E[\ell(m_{\bm\xi}(X^+,X^-))]
\end{equation*}
be the pairwise AUC surrogate risk in the trainable subspace. Suppose \(L_\ell>0\) and \(R_\ell\) has \(L_\ell\)-Lipschitz gradient on a neighborhood \(\mathcal U\). Consider dynamics
\begin{equation*}
\bm\xi_{t+1}
=
\bm\xi_t-\eta_t\cdot\widehat{\bm g}_t,
\qquad
t=0,\ldots,T-1,
\end{equation*}
adapted to a filtration \((\mathcal F_t)_{t\ge0}\), with \(\bm\xi_0,\ldots,\bm\xi_T\in\mathcal U\) almost surely. Assume
\begin{equation*}
\E[\widehat{\bm g}_t\mid\mathcal F_t]
=
\nabla R_\ell(\bm\xi_t),
\qquad
\E[\|\widehat{\bm g}_t-\nabla R_\ell(\bm\xi_t)\|^2\mid\mathcal F_t]
\le
\sigma^2,
\end{equation*}
and \(0<\eta_t\le1/L_\ell\). Then
\begin{equation*}
\E[R_\ell(\bm\xi_{t+1})\mid\mathcal F_t]
\le
R_\ell(\bm\xi_t)
-
\frac{\eta_t}{2}\cdot\|\nabla R_\ell(\bm\xi_t)\|^2
+
\frac{L_\ell\cdot\eta_t^2}{2}\cdot\sigma^2.
\end{equation*}
Consequently, if \(R_{\ell,\mathcal U}^\star=\inf_{\bm\xi\in\mathcal U}R_\ell(\bm\xi)\), then
\begin{equation*}
\sum_{t=0}^{T-1}
\frac{\eta_t}{2}\cdot
\E[\|\nabla R_\ell(\bm\xi_t)\|^2]
\le
R_\ell(\bm\xi_0)-R_{\ell,\mathcal U}^\star
+
\frac{L_\ell\cdot\sigma^2}{2}
\sum_{t=0}^{T-1}\eta_t^2.
\end{equation*}
\end{proposition}

\begin{proof}
Smoothness gives
\begin{equation*}
R_\ell(\bm\xi_{t+1})
\le
R_\ell(\bm\xi_t)
+
\langle\nabla R_\ell(\bm\xi_t),\bm\xi_{t+1}-\bm\xi_t\rangle
+
\frac{L_\ell}{2}\cdot\|\bm\xi_{t+1}-\bm\xi_t\|^2.
\end{equation*}
Substitute \(\bm\xi_{t+1}-\bm\xi_t=-\eta_t\cdot\widehat{\bm g}_t\) and condition on \(\mathcal F_t\):
\begin{equation*}
\begin{aligned}
\E[R_\ell(\bm\xi_{t+1})\mid\mathcal F_t]
&\le
R_\ell(\bm\xi_t)
-
\eta_t\cdot
\left\langle
\nabla R_\ell(\bm\xi_t),
\E[\widehat{\bm g}_t\mid\mathcal F_t]
\right\rangle \\
&\quad
+
\frac{L_\ell\cdot\eta_t^2}{2}
\cdot
\E[\|\widehat{\bm g}_t\|^2\mid\mathcal F_t].
\end{aligned}
\end{equation*}
The conditional unbiasedness assumption gives the inner product \(\|\nabla R_\ell(\bm\xi_t)\|^2\). Also,
\begin{equation*}
\E[\|\widehat{\bm g}_t\|^2\mid\mathcal F_t]
\le
\|\nabla R_\ell(\bm\xi_t)\|^2+\sigma^2,
\end{equation*}
because the cross term in
\(\widehat{\bm g}_t=(\widehat{\bm g}_t-\nabla R_\ell(\bm\xi_t))+\nabla R_\ell(\bm\xi_t)\)
has conditional expectation zero. Thus
\begin{equation*}
\E[R_\ell(\bm\xi_{t+1})\mid\mathcal F_t]
\le
R_\ell(\bm\xi_t)
-
\eta_t\cdot
\left(1-\frac{L_\ell\cdot\eta_t}{2}\right)
\cdot
\|\nabla R_\ell(\bm\xi_t)\|^2
+
\frac{L_\ell\cdot\eta_t^2}{2}\cdot\sigma^2.
\end{equation*}
Since \(\eta_t\le1/L_\ell\), \(1-L_\ell\cdot\eta_t/2\ge1/2\), proving the one-step bound. Taking expectations, summing over \(t\), and using \(\E[R_\ell(\bm\xi_T)]\ge R_{\ell,\mathcal U}^\star\) proves the finite-horizon inequality.
\end{proof}

\begin{corollary}[Local PL contraction in a favorable trainable subspace]
\label{cor:app-peft-local-pl-contraction}
In the setting of Proposition~\ref{prop:app-peft-local-auc-descent}, assume additionally that \(\mu>0\) and \(R_\ell\) satisfies the local PL inequality on \(\mathcal U\):
\begin{equation*}
\frac{1}{2}\cdot\|\nabla R_\ell(\bm\xi)\|^2
\ge
\mu\cdot
\bigl(R_\ell(\bm\xi)-R_{\ell,\mathcal U}^\star\bigr)
\qquad
(\bm\xi\in\mathcal U).
\end{equation*}
For a constant step size \(0<\eta\le\min\{1/L_\ell,1/\mu\}\), if \(\bm\xi_0,\ldots,\bm\xi_T\in\mathcal U\) almost surely, then
\begin{equation*}
\E[R_\ell(\bm\xi_t)-R_{\ell,\mathcal U}^\star]
\le
(1-\eta\cdot\mu)^t
\cdot
\bigl(R_\ell(\bm\xi_0)-R_{\ell,\mathcal U}^\star\bigr)
+
\frac{L_\ell\cdot\eta\cdot\sigma^2}{2\cdot\mu}
\cdot
\bigl(1-(1-\eta\cdot\mu)^t\bigr)
\end{equation*}
for every \(0\le t\le T\).
\end{corollary}

\begin{proof}
With \(\eta_t=\eta\), Proposition~\ref{prop:app-peft-local-auc-descent} gives
\begin{equation*}
\E[R_\ell(\bm\xi_{t+1})\mid\mathcal F_t]
\le
R_\ell(\bm\xi_t)
-
\frac{\eta}{2}\cdot\|\nabla R_\ell(\bm\xi_t)\|^2
+
\frac{L_\ell\cdot\eta^2}{2}\cdot\sigma^2.
\end{equation*}
The PL inequality gives
\begin{equation*}
\frac{1}{2}\cdot\|\nabla R_\ell(\bm\xi_t)\|^2
\ge
\mu\cdot
\bigl(R_\ell(\bm\xi_t)-R_{\ell,\mathcal U}^\star\bigr).
\end{equation*}
After subtracting \(R_{\ell,\mathcal U}^\star\), taking expectations, and iterating the resulting scalar inequality,
\begin{equation*}
a_{t+1}
\le
(1-\eta\cdot\mu)\cdot a_t
+
\frac{L_\ell\cdot\eta^2}{2}\cdot\sigma^2,
\qquad
a_t=\E[R_\ell(\bm\xi_t)-R_{\ell,\mathcal U}^\star],
\end{equation*}
we obtain
\begin{equation*}
a_t
\le
(1-\eta\cdot\mu)^t a_0
+
\frac{L_\ell\cdot\eta^2}{2}\cdot\sigma^2
\sum_{s=0}^{t-1}(1-\eta\cdot\mu)^s.
\end{equation*}
The geometric sum equals
\begin{equation*}
\frac{1-(1-\eta\cdot\mu)^t}{\eta\cdot\mu},
\end{equation*}
which proves the claim.
\end{proof}

\begin{proposition}[Direct pairwise AUC-surrogate training in a favorable trainable subspace]
\label{prop:app-peft-good-subspace-direct-auc}
Suppose a trainable-subspace neighborhood \(\mathcal U\) satisfies
\begin{equation*}
\sup_{\bm\xi\in\mathcal U}\pi_M(\bm\xi)\le\delta_M
\end{equation*}
and, for some unit vector \(\bm v\),
\begin{equation*}
\Var(\psi_{\bm\xi,\bm v}^{(1,0)}(X^+))
+
\Var(\psi_{\bm\xi,\bm v}^{(0,1)}(X^-))
\ge
\sigma_{1,\bm v}^2>0
\qquad
(\bm\xi\in\mathcal U).
\end{equation*}
Assume the direct pairwise AUC-surrogate dynamics in the trainable subspace are conditionally unbiased, satisfy the moment bound of Proposition~\ref{prop:app-peft-local-auc-descent}, and stay in \(\mathcal U\) on the finite horizon. Then, on that horizon,
\begin{equation*}
\pi_M(\bm\xi_t)\le\delta_M,
\qquad
\Var(\langle \widehat{\bm g}_B^\ell(\bm\xi_t),\bm v\rangle)
\ge
\frac{\sigma_{1,\bm v}^2}{B},
\end{equation*}
and the finite-horizon descent bound in Proposition~\ref{prop:app-peft-local-auc-descent} holds. If the local PL condition also holds, the contraction bound in Corollary~\ref{cor:app-peft-local-pl-contraction} holds as well.
\end{proposition}

\begin{proof}
The severe-misranking and first-order variance statements follow from Proposition~\ref{prop:app-peft-finite-horizon-reduced-warmup}. The conditional unbiasedness and moment assumptions place the same trajectory in the setting of Proposition~\ref{prop:app-peft-local-auc-descent}, so the descent bound holds. If the local PL condition is added, Corollary~\ref{cor:app-peft-local-pl-contraction} applies.
\end{proof}

\section{Model-Class Realizations of Projected Bad Components}
\label{app:transformer-degenerate-region}

This section supports the model-class side of Section~\ref{sec:bad-regions}. It constructs Transformer-style scorers in which the projected bad-component conditions from Appendix~\ref{app:pairwise-bad-region-dynamics} hold inside the model parameterization, beyond the abstract linear example. The attention map is genuinely data-dependent: the query, key, value, and output maps are nonzero, and the classification token assigns different attention weights to the data token for the two possible inputs. The results are existence statements. They produce neighborhoods in which a specified trainable direction has weak drift, collapsed first-order stochastic variation, and nonzero residual pair-interaction variance. Degeneracy is established for the constructed projected direction.

\subsection{Data-Dependent One-Layer Construction}
\label{app:transformer-one-layer}

\begin{proposition}[Data-dependent one-layer Transformer region with high error and first-order variance collapse]
\label{prop:one-layer-transformer-degenerate-region}
Fix \(M>0\), a local radius \(r>0\), a deterministic-drift tolerance \(\delta>0\), and any \(\lambda>0\). There is a one-head, one-layer residual self-attention scorer with nonzero query, key, value, and output maps, a single trainable scalar \(\theta\), and a scale \(\gamma>0\) such that the classification-token attention weight depends on the input and, on the nonempty interval
\begin{equation*}
\mathcal U_{\gamma,M,r}
:=
\left[\frac{M}{\gamma},\frac{M}{\gamma}+r\right],
\end{equation*}
the following properties hold for the pairwise square loss \(\ell(x)=x^2\).
For this loss, let
\begin{equation*}
G_\theta(A,C)
:=
\frac{d}{d\theta}\ell(m_\theta(A,C))
\end{equation*}
denote the scalar gradient kernel in the trainable direction.
\begin{enumerate}[leftmargin=5mm]
  \item The attention weight placed by the classification token on the data token is
  \begin{equation*}
  \alpha_z=\frac{\exp(\lambda\cdot z)}{1+\exp(\lambda\cdot z)},
  \qquad z\in\{-1,1\},
  \end{equation*}
  so \(\alpha_{+1}\ne \alpha_{-1}\).
  \item The severe-misranking mass at level \(M\) is
  \begin{equation*}
  \pi_M(\theta)
  :=
  \Pbb(m_\theta(X^+,X^-)\le -M)
  =
  \frac{1}{4},
  \qquad
  \theta\in\mathcal U_{\gamma,M,r}.
  \end{equation*}
  \item The expected scalar gradient in the trainable direction satisfies
  \begin{equation*}
  \sup_{\theta\in\mathcal U_{\gamma,M,r}}
  \left|\E[G_\theta(X^+,X^-)]\right|
  \le \delta
  \end{equation*}
  whenever \(\gamma>0\) is chosen sufficiently small.
  \item Both first-order Hoeffding projections of the scalar pairwise-gradient kernel vanish for every \(\theta\):
  \begin{equation*}
  \E[G_\theta(A,C)\mid A]-\E[G_\theta(A,C)]
  =
  0,
  \qquad
  \E[G_\theta(A,C)\mid C]-\E[G_\theta(A,C)]
  =
  0.
  \end{equation*}
  Consequently, for a mini-batch with \(B\) positives and \(B\) negatives,
  \begin{equation*}
  \Var\!\left(
  \frac{1}{B^2}
  \sum_{i=1}^B\sum_{j=1}^B
  G_\theta(A_i,C_j)
  \right)
  =
  \frac{\gamma^4\cdot \theta^2}{B^2}.
  \end{equation*}
\end{enumerate}
\end{proposition}

\begin{proof}
Let \(A\) and \(C\) be independent Rademacher random variables. A positive example is the two-token sequence \(X^+=s_A\), and a negative example is the two-token sequence \(X^-=s_C\). For \(z\in\{-1,1\}\), the sequence \(s_z\) consists of a classification token and a data token with embeddings
\begin{equation*}
\bm e_{\mathrm{cls}}=(1,0,0),
\qquad
\bm e_z=(0,z,0)
\qquad\text{in }\R^3.
\end{equation*}
There are no positional embeddings. Consider a single-head residual self-attention block without layer normalization. For a two-token hidden-state matrix \(\bm H\in\R^{2\times 3}\), the attention output is
\begin{equation*}
\mathrm{Attn}(\bm H)
=
\mathrm{softmax}\!\left(
(\bm H\cdot\bm W_Q)\cdot(\bm H\cdot\bm W_K)^\top
\right)
\cdot(\bm H\cdot\bm W_V)\cdot \bm W_O.
\end{equation*}
Choose
\begin{equation*}
\bm W_Q=
\begin{pmatrix}
1\\0\\0
\end{pmatrix},
\qquad
\bm W_K=
\begin{pmatrix}
0\\ \lambda\\0
\end{pmatrix},
\qquad
\bm W_V=
\begin{pmatrix}
0\\1\\0
\end{pmatrix},
\qquad
\bm W_O=
\begin{pmatrix}
0 & 1 & 0
\end{pmatrix}.
\end{equation*}
All four maps are nonzero. For the classification token, the query is \(1\), the classification-token key is \(0\), and the data-token key is \(\lambda\cdot z\). Hence the classification-token attention weight on the data token is
\begin{equation*}
\alpha_z
=
\frac{\exp(\lambda\cdot z)}
{1+\exp(\lambda\cdot z)}.
\end{equation*}
Since \(\lambda>0\), \(\alpha_{+1}\ne\alpha_{-1}\). The value at the classification token is \(0\), while the value at the data token is \(z\):
\begin{equation*}
\bm e_{\mathrm{cls}}\cdot\bm W_V=0,
\qquad
\bm e_z\cdot\bm W_V=z.
\end{equation*}
Thus the attention output added to the classification token is
\begin{equation*}
(0,\alpha_z\cdot z,0).
\end{equation*}
The residual addition leaves the classification-token representation equal to
\begin{equation*}
\bm r_z=(1,\alpha_z\cdot z,0).
\end{equation*}
Let the feed-forward map be identically zero, so the residual feed-forward sublayer leaves \(\bm r_z\) unchanged.

Now attach the trainable scalar head
\begin{equation*}
\bm a_\theta=(0,\gamma\cdot \theta,0)^\top,
\qquad
\gamma>0,
\end{equation*}
and define the score by
\begin{equation*}
f_\theta(s_z)=\bm a_\theta^\top\cdot \bm r_z.
\end{equation*}
Substituting \(\bm r_z=(1,\alpha_z\cdot z,0)\) gives
\begin{equation*}
f_\theta(s_z)
=
(0,\gamma\cdot \theta,0)\cdot(1,\alpha_z\cdot z,0)
=
\gamma\cdot\theta\cdot\alpha_z\cdot z.
\end{equation*}
Thus the one-layer Transformer realizes a one-dimensional scorer through a nonconstant attention weight. Define
\begin{equation*}
h_z=\alpha_z\cdot z.
\end{equation*}
Because \(\alpha_{+1}+\alpha_{-1}=1\), we have
\begin{equation*}
h_{+1}-h_{-1}=1.
\end{equation*}

For a positive value \(A\) and a negative value \(C\), the positive margin is
\begin{equation*}
m_\theta(A,C)
=
f_\theta(s_A)-f_\theta(s_C)
=
\gamma\cdot\theta\cdot(h_A-h_C).
\end{equation*}
If \(\theta\in\mathcal U_{\gamma,M,r}\), then \(\theta\ge M/\gamma\). On the event \(\{A=-1,C=1\}\), which has probability \(1/4\),
\begin{equation*}
m_\theta(A,C)
=
-\gamma\cdot \theta
\le
-M.
\end{equation*}
On the other three Rademacher outcomes, \(h_A-h_C\in\{0,1\}\), so \(m_\theta(A,C)\ge0\). Hence
\begin{equation*}
\pi_M(\theta)=\Pbb(m_\theta(A,C)\le -M)=\frac{1}{4}.
\end{equation*}

For \(\ell(x)=x^2\), the scalar gradient kernel in the trainable direction is
\begin{equation*}
G_\theta(A,C)
=
\frac{d}{d\theta}\ell(m_\theta(A,C))
=
2\cdot m_\theta(A,C)\cdot
\frac{d}{d\theta}m_\theta(A,C).
\end{equation*}
Since
\begin{equation*}
\frac{d}{d\theta}m_\theta(A,C)
=
\gamma\cdot(h_A-h_C),
\end{equation*}
we obtain
\begin{equation*}
G_\theta(A,C)
=
2\cdot\gamma^2\cdot\theta\cdot(h_A-h_C)^2.
\end{equation*}
Because \(h_A-h_C\) equals \(0\) when \(A=C\), equals \(1\) when \((A,C)=(1,-1)\), and equals \(-1\) when \((A,C)=(-1,1)\),
\begin{equation*}
(h_A-h_C)^2
=
\frac{1-A\cdot C}{2},
\end{equation*}
and therefore
\begin{equation*}
G_\theta(A,C)
=
\gamma^2\cdot\theta\cdot(1-A\cdot C).
\end{equation*}
Its mean is
\begin{equation*}
\E[G_\theta(A,C)]
=
\gamma^2\cdot\theta,
\end{equation*}
because \(\E[A\cdot C]=\E[A]\cdot\E[C]=0\). For \(\theta\in\mathcal U_{\gamma,M,r}\),
\begin{equation*}
\left|\E[G_\theta(A,C)]\right|
\le
\gamma^2
\left(\frac{M}{\gamma}+r\right)
=
\gamma\cdot M+\gamma^2\cdot r.
\end{equation*}
Choosing \(\gamma>0\) small enough that
\begin{equation*}
\gamma\cdot M+\gamma^2\cdot r\le\delta
\end{equation*}
gives the stated drift bound.

It remains to verify the Hoeffding projections and the mini-batch variance. Conditioning on \(A=a\) gives
\begin{equation*}
\E[G_\theta(a,C)]
=
\gamma^2\cdot\theta
-\gamma^2\cdot\theta\cdot a\cdot\E[C]
=
\gamma^2\cdot\theta
=
\E[G_\theta(A,C)].
\end{equation*}
The same calculation conditioning on \(C=c\) gives
\begin{equation*}
\E[G_\theta(A,c)]
=
\gamma^2\cdot\theta
-\gamma^2\cdot\theta\cdot\E[A]\cdot c
=
\E[G_\theta(A,C)].
\end{equation*}
Thus both first-order Hoeffding projections vanish identically. The centered canonical term is
\begin{equation*}
G_\theta(A,C)-\E[G_\theta(A,C)]
=
-\gamma^2\cdot\theta\cdot A\cdot C.
\end{equation*}
Since \(A\cdot C\) is Rademacher,
\begin{equation*}
\Var(G_\theta(A,C)-\E[G_\theta(A,C)])
=
\gamma^4\cdot\theta^2.
\end{equation*}
For independent positive samples \(A_1,\ldots,A_B\) and negative samples \(C_1,\ldots,C_B\), the first-order projections are zero, so only the canonical second-order term remains in the two-sample Hoeffding variance split. Equivalently, the only nonzero covariances among the centered terms occur when both indices match. Hence
\begin{equation*}
\Var\!\left(
\frac{1}{B^2}
\sum_{i=1}^B\sum_{j=1}^B
G_\theta(A_i,C_j)
\right)
=
\frac{B^2}{B^4}
\cdot
\gamma^4\cdot\theta^2
=
\frac{\gamma^4\cdot\theta^2}{B^2}.
\end{equation*}

\end{proof}

\subsection{Projected Bad Components in Trainable Transformer Neighborhoods}
\label{app:transformer-trainable-neighborhoods}

The next two corollaries move from a single trainable scalar to neighborhoods in the trainable Transformer parameter space. In both statements, \(\Psi_{\bm\omega,\bm v}(A,C)\) denotes the scalar projected pairwise-gradient kernel
\begin{equation*}
\Psi_{\bm\omega,\bm v}(A,C)
=
\left\langle
\nabla_{\bm\omega}\ell(m_{\bm\omega}(A,C)),
\bm v
\right\rangle .
\end{equation*}
Its components \(\Psi_{\bm\omega,\bm v}^{(1,0)}\), \(\Psi_{\bm\omega,\bm v}^{(0,1)}\), and \(\Psi_{\bm\omega,\bm v}^{(1,1)}\) are defined by the same two-sample Hoeffding projection formulas as in Theorem~\ref{thm:app-hoeffding}.

\begin{corollary}[Full-parameter one-layer Transformer neighborhood with a projected bad component]
\label{cor:full-param-one-layer-projected-bad}
Fix \(M>0\), a slack \(s>0\), a line radius \(r>0\), and tolerances
\begin{equation*}
\delta_{\rm drift}>0,
\qquad
\varepsilon_1>0.
\end{equation*}
In the one-layer architecture of Proposition~\ref{prop:one-layer-transformer-degenerate-region}, let the trainable parameter vector \(\bm\omega\) collect all query, key, value, output, and feed-forward parameters, together with the scaled scalar head coordinate \(\theta\) in the parameterization
\begin{equation*}
\bm a_\theta=(0,\gamma\cdot\theta,0)^\top.
\end{equation*}
Let \(\bm v_\theta\) be the unit direction corresponding to this scalar coordinate.
Then there are \(\gamma>0\), a compact line segment
\begin{equation*}
\mathcal L_{\gamma,M,s,r}
=
\left\{
\bm\omega(\theta):
\theta\in
\left[
\frac{M+s}{\gamma},
\frac{M+s}{\gamma}+r
\right]
\right\},
\end{equation*}
and an open neighborhood \(\mathcal U_{\rm tr}\) of this segment in the corresponding trainable parameter space such that the following statements hold for the pairwise square loss.
\begin{enumerate}[leftmargin=5mm]
  \item Severe misrankings persist:
  \begin{equation*}
  \inf_{\bm\omega\in\mathcal U_{\rm tr}}
  \Pbb(m_{\bm\omega}(X^+,X^-)\le -M)
  \ge
  \frac{1}{4}.
  \end{equation*}
  \item The projected deterministic drift is small:
  \begin{equation*}
  \sup_{\bm\omega\in\mathcal U_{\rm tr}}
  \left|
  \langle\nabla R_\ell(\bm\omega),\bm v_\theta\rangle
  \right|
  \le
  \delta_{\rm drift}.
  \end{equation*}
  \item The leading projected stochastic component is small:
  \begin{equation*}
  \sup_{\bm\omega\in\mathcal U_{\rm tr}}
  \left\{
  \Var(\Psi_{\bm\omega,\bm v_\theta}^{(1,0)}(A))
  +
  \Var(\Psi_{\bm\omega,\bm v_\theta}^{(0,1)}(C))
  \right\}
  \le
  \varepsilon_1.
  \end{equation*}
  \item The residual pair-interaction component remains nonzero:
  \begin{equation*}
  \inf_{\bm\omega\in\mathcal U_{\rm tr}}
  \Var(\Psi_{\bm\omega,\bm v_\theta}^{(1,1)}(A,C))
  >
  0.
  \end{equation*}
\end{enumerate}
Thus a one-layer Transformer with all listed parameters trainable contains a neighborhood that is bad in the projected trainable direction \(\bm v_\theta\). This is a projected-component statement for the constructed direction.
\end{corollary}

\begin{proof}
Use the base construction in Proposition~\ref{prop:one-layer-transformer-degenerate-region}, but replace \(M\) there by \(M+s\). Along the compact segment \(\mathcal L_{\gamma,M,s,r}\), the pair \((A,C)=(-1,1)\) has probability \(1/4\) and margin
\begin{equation*}
m_{\bm\omega(\theta)}(-1,1)
=
-\gamma\cdot\theta
\le
-(M+s).
\end{equation*}
All score values on the finite input family \(\{s_{-1},s_{+1}\}\) are smooth functions of the trainable parameter vector \(\bm\omega\), because they are obtained by matrix multiplication and softmax. Therefore, for each point \(\bm\omega(\theta)\) on the compact segment, there is a trainable-parameter neighborhood on which
\begin{equation*}
m_{\bm\omega}(-1,1)\le -M.
\end{equation*}
A finite subcover of the compact segment gives an open neighborhood \(\mathcal U_0\) of \(\mathcal L_{\gamma,M,s,r}\) such that this inequality holds throughout \(\mathcal U_0\). Hence
\begin{equation*}
\Pbb(m_{\bm\omega}(X^+,X^-)\le -M)
\ge
\Pbb(A=-1,C=1)
=
\frac{1}{4}
\end{equation*}
for every \(\bm\omega\in\mathcal U_0\).

Next consider the scalar projected pairwise-gradient kernel
\begin{equation*}
\Psi_{\bm\omega,\bm v_\theta}(A,C)
=
\left\langle
\nabla_{\bm\omega}\ell(m_{\bm\omega}(A,C)),
\bm v_\theta
\right\rangle .
\end{equation*}
Define its positive-class, negative-class, and residual pair-interaction components by
\begin{equation*}
\mu_{\bm\omega}
=
\E[\Psi_{\bm\omega,\bm v_\theta}(A,C)],
\end{equation*}
\begin{equation*}
\Psi_{\bm\omega,\bm v_\theta}^{(1,0)}(a)
=
\E[\Psi_{\bm\omega,\bm v_\theta}(a,C)]-\mu_{\bm\omega},
\qquad
\Psi_{\bm\omega,\bm v_\theta}^{(0,1)}(c)
=
\E[\Psi_{\bm\omega,\bm v_\theta}(A,c)]-\mu_{\bm\omega},
\end{equation*}
and
\begin{equation*}
\Psi_{\bm\omega,\bm v_\theta}^{(1,1)}(a,c)
=
\Psi_{\bm\omega,\bm v_\theta}(a,c)
-\E[\Psi_{\bm\omega,\bm v_\theta}(a,C)]
-\E[\Psi_{\bm\omega,\bm v_\theta}(A,c)]
+\mu_{\bm\omega}.
\end{equation*}
On the segment \(\mathcal L_{\gamma,M,s,r}\), this is exactly the scalar kernel \(G_\theta(A,C)\) in Proposition~\ref{prop:one-layer-transformer-degenerate-region}, with \(M\) replaced by \(M+s\). Therefore
\begin{equation*}
\E[\Psi_{\bm\omega(\theta),\bm v_\theta}(A,C)]
=
\gamma^2\cdot\theta,
\end{equation*}
and hence
\begin{equation*}
\left|
\E[\Psi_{\bm\omega(\theta),\bm v_\theta}(A,C)]
\right|
\le
\gamma^2
\left(
\frac{M+s}{\gamma}+r
\right)
=
\gamma\cdot(M+s)+\gamma^2\cdot r.
\end{equation*}
Choose \(\gamma>0\) small enough that
\begin{equation*}
\gamma\cdot(M+s)+\gamma^2\cdot r
\le
\frac{\delta_{\rm drift}}{2}.
\end{equation*}
The map
\begin{equation*}
\bm\omega
\mapsto
\E[\Psi_{\bm\omega,\bm v_\theta}(A,C)]
\end{equation*}
is continuous on the finite input family, so after possibly shrinking \(\mathcal U_0\) to a smaller open neighborhood \(\mathcal U_1\) of the same compact segment,
\begin{equation*}
\sup_{\bm\omega\in\mathcal U_1}
\left|
\langle\nabla R_\ell(\bm\omega),\bm v_\theta\rangle
\right|
\le
\delta_{\rm drift}.
\end{equation*}

On the segment, Proposition~\ref{prop:one-layer-transformer-degenerate-region} gives exact first-order projection collapse:
\begin{equation*}
\Var(\Psi_{\bm\omega(\theta),\bm v_\theta}^{(1,0)}(A))
+
\Var(\Psi_{\bm\omega(\theta),\bm v_\theta}^{(0,1)}(C))
=0.
\end{equation*}
The two first-order projection variances are continuous functions of \(\bm\omega\), again because the input family is finite and the score map is smooth. Thus, after shrinking to an open neighborhood \(\mathcal U_2\subseteq\mathcal U_1\),
\begin{equation*}
\sup_{\bm\omega\in\mathcal U_2}
\left\{
\Var(\Psi_{\bm\omega,\bm v_\theta}^{(1,0)}(A))
+
\Var(\Psi_{\bm\omega,\bm v_\theta}^{(0,1)}(C))
\right\}
\le
\varepsilon_1.
\end{equation*}

Finally, on the segment the canonical variance is
\begin{equation*}
\Var(\Psi_{\bm\omega(\theta),\bm v_\theta}^{(1,1)}(A,C))
=
\gamma^4\cdot\theta^2
\ge
\gamma^2\cdot(M+s)^2
>
0.
\end{equation*}
By continuity and compactness, after one last shrinking to an open neighborhood
\begin{equation*}
\mathcal U_{\rm tr}\subseteq\mathcal U_2
\end{equation*}
of \(\mathcal L_{\gamma,M,s,r}\), we have
\begin{equation*}
\inf_{\bm\omega\in\mathcal U_{\rm tr}}
\Var(\Psi_{\bm\omega,\bm v_\theta}^{(1,1)}(A,C))
\ge
\frac{1}{2}\cdot\gamma^2\cdot(M+s)^2
>
0.
\end{equation*}
All four displayed properties now hold on \(\mathcal U_{\rm tr}\), which proves the corollary.
\end{proof}

\begin{corollary}[Two-layer and finite-depth Transformer embedding]
\label{cor:finite-depth-transformer-embedding}
There is a two-layer residual self-attention scorer whose pairwise square-loss dynamics satisfy all conclusions of Proposition~\ref{prop:one-layer-transformer-degenerate-region} on the same interval \(\mathcal U_{\gamma,M,r}\). More generally, for every finite integer \(L\ge1\), the same conclusions hold for some \(L\)-layer residual self-attention scorer.
\end{corollary}

\begin{proof}
Use Proposition~\ref{prop:one-layer-transformer-degenerate-region} as the first layer. For the second layer, choose the value map and feed-forward map to be zero while keeping the residual connections. The self-attention output and feed-forward output of the second layer are then zero, so the second residual block is the identity map on every token representation. In particular, the classification-token representation remains
\begin{equation*}
\bm r_z=(1,\alpha_z\cdot z,0)
\end{equation*}
for every \(z\in\{-1,1\}\). Attach the same trainable scalar head
\begin{equation*}
\bm a_\theta=(0,\gamma\cdot\theta,0)^\top.
\end{equation*}
The resulting two-layer scorer is therefore identical to the scorer in Proposition~\ref{prop:one-layer-transformer-degenerate-region} on the input family \(\{s_{-1},s_{+1}\}\). The pairwise margins, square-loss gradient kernel, conditional expectations, Hoeffding projections, and mini-batch variance are identical as well. Hence all conclusions of Proposition~\ref{prop:one-layer-transformer-degenerate-region} hold for the two-layer scorer.

The finite-depth statement is the same embedding repeated. For any \(L>2\), choose the value map and feed-forward map of every additional layer \(\ell=2,\ldots,L\) to be zero while keeping residual connections. Each added block is again the identity map, so the final scorer remains identical to the one-layer scorer on \(\{s_{-1},s_{+1}\}\). The same pairwise dynamics and variance identities therefore hold for every finite \(L\ge1\).
\end{proof}

\begin{corollary}[Full-parameter two-layer Transformer neighborhood with a projected bad component]
\label{cor:full-param-two-layer-projected-bad}
Fix \(M>0\), a slack \(s>0\), a line radius \(r>0\), and tolerances
\begin{equation*}
\delta_{\rm drift}>0,
\qquad
\varepsilon_1>0.
\end{equation*}
Consider the two-layer residual self-attention architecture obtained by embedding the one-layer construction of Proposition~\ref{prop:one-layer-transformer-degenerate-region} into the first layer and initializing the second layer as an identity residual block, as in Corollary~\ref{cor:finite-depth-transformer-embedding}. Let the trainable parameter vector \(\bm\omega\) collect all query, key, value, output, and feed-forward parameters in both layers, together with the scaled scalar head coordinate \(\theta\) in
\begin{equation*}
\bm a_\theta=(0,\gamma\cdot\theta,0)^\top .
\end{equation*}
Let \(\bm v_\theta\) be the unit direction corresponding to this scalar coordinate. Then there are \(\gamma>0\), a compact line segment
\begin{equation*}
\mathcal L^{(2)}_{\gamma,M,s,r}
=
\left\{
\bm\omega^{(2)}(\theta):
\theta\in
\left[
\frac{M+s}{\gamma},
\frac{M+s}{\gamma}+r
\right]
\right\},
\end{equation*}
and an open neighborhood \(\mathcal U_{\rm tr}^{(2)}\) of this segment in the full two-layer trainable parameter space such that the following statements hold for the pairwise square loss.
\begin{enumerate}[leftmargin=5mm]
  \item Severe misrankings persist:
  \begin{equation*}
  \inf_{\bm\omega\in\mathcal U_{\rm tr}^{(2)}}
  \Pbb(m_{\bm\omega}(X^+,X^-)\le -M)
  \ge
  \frac{1}{4}.
  \end{equation*}
  \item The projected deterministic drift is small:
  \begin{equation*}
  \sup_{\bm\omega\in\mathcal U_{\rm tr}^{(2)}}
  \left|
  \langle\nabla R_\ell(\bm\omega),\bm v_\theta\rangle
  \right|
  \le
  \delta_{\rm drift}.
  \end{equation*}
  \item The leading projected stochastic component is small:
  \begin{equation*}
  \sup_{\bm\omega\in\mathcal U_{\rm tr}^{(2)}}
  \left\{
  \Var(\Psi_{\bm\omega,\bm v_\theta}^{(1,0)}(A))
  +
  \Var(\Psi_{\bm\omega,\bm v_\theta}^{(0,1)}(C))
  \right\}
  \le
  \varepsilon_1.
  \end{equation*}
  \item The residual pair-interaction component remains nonzero:
  \begin{equation*}
  \inf_{\bm\omega\in\mathcal U_{\rm tr}^{(2)}}
  \Var(\Psi_{\bm\omega,\bm v_\theta}^{(1,1)}(A,C))
  >
  0.
  \end{equation*}
\end{enumerate}
Thus a two-layer Transformer with all listed parameters trainable contains a neighborhood that is bad in the projected trainable direction \(\bm v_\theta\). This is a projected-component statement for the constructed direction.
\end{corollary}

\begin{proof}
Start from the one-layer construction in Proposition~\ref{prop:one-layer-transformer-degenerate-region}, replacing \(M\) there by \(M+s\). Embed it as the first layer of a two-layer residual self-attention scorer. Initialize the second layer as in Corollary~\ref{cor:finite-depth-transformer-embedding}: its value map and feed-forward map are zero, and residual connections are kept. Along the line segment \(\mathcal L^{(2)}_{\gamma,M,s,r}\), all first-layer and second-layer parameters are fixed at these values, and only the scalar head coordinate \(\theta\) varies.

At every point of \(\mathcal L^{(2)}_{\gamma,M,s,r}\), the second residual block is the identity map on the finite input family \(\{s_{-1},s_{+1}\}\). Therefore the two-layer score agrees exactly with the one-layer score:
\begin{equation*}
f_{\bm\omega^{(2)}(\theta)}(s_z)
=
\gamma\cdot\theta\cdot\alpha_z\cdot z,
\qquad
z\in\{-1,1\}.
\end{equation*}
Consequently, for a positive value \(A\) and a negative value \(C\),
\begin{equation*}
m_{\bm\omega^{(2)}(\theta)}(A,C)
=
\gamma\cdot\theta\cdot(h_A-h_C),
\qquad
h_z=\alpha_z\cdot z.
\end{equation*}
On the event \((A,C)=(-1,1)\), which has probability \(1/4\),
\begin{equation*}
m_{\bm\omega^{(2)}(\theta)}(-1,1)
=
-\gamma\cdot\theta
\le
-(M+s)
\end{equation*}
for every
\begin{equation*}
\theta\in
\left[
\frac{M+s}{\gamma},
\frac{M+s}{\gamma}+r
\right].
\end{equation*}
All score values on \(\{s_{-1},s_{+1}\}\) are smooth functions of the full two-layer parameter vector \(\bm\omega\): they are obtained by matrix multiplication, softmax, residual addition, and the final linear head. Since \(\mathcal L^{(2)}_{\gamma,M,s,r}\) is compact, the strict slack \(s>0\) implies that there is an open neighborhood \(\mathcal U_0^{(2)}\) of the segment such that
\begin{equation*}
m_{\bm\omega}(-1,1)\le -M
\qquad
\text{for all }\bm\omega\in\mathcal U_0^{(2)}.
\end{equation*}
Hence
\begin{equation*}
\Pbb(m_{\bm\omega}(X^+,X^-)\le -M)
\ge
\Pbb(A=-1,C=1)
=
\frac{1}{4}
\end{equation*}
throughout \(\mathcal U_0^{(2)}\).

Next define the scalar projected pairwise-gradient kernel
\begin{equation*}
\Psi_{\bm\omega,\bm v_\theta}(A,C)
=
\left\langle
\nabla_{\bm\omega}\ell(m_{\bm\omega}(A,C)),
\bm v_\theta
\right\rangle .
\end{equation*}
Because \(\bm v_\theta\) is the unit direction of the scalar head coordinate, on the segment this projected kernel is exactly the one-dimensional kernel from Proposition~\ref{prop:one-layer-transformer-degenerate-region}, with \(M\) replaced by \(M+s\):
\begin{equation*}
\Psi_{\bm\omega^{(2)}(\theta),\bm v_\theta}(A,C)
=
\gamma^2\cdot\theta\cdot(1-A\cdot C).
\end{equation*}
Therefore
\begin{equation*}
\E[\Psi_{\bm\omega^{(2)}(\theta),\bm v_\theta}(A,C)]
=
\gamma^2\cdot\theta,
\end{equation*}
and for \(\theta\) on the segment,
\begin{equation*}
\left|
\E[\Psi_{\bm\omega^{(2)}(\theta),\bm v_\theta}(A,C)]
\right|
\le
\gamma^2
\left(
\frac{M+s}{\gamma}+r
\right)
=
\gamma\cdot(M+s)+\gamma^2\cdot r.
\end{equation*}
Choose \(\gamma>0\) so that
\begin{equation*}
\gamma\cdot(M+s)+\gamma^2\cdot r
\le
\frac{\delta_{\rm drift}}{2}.
\end{equation*}
The map
\begin{equation*}
\bm\omega
\mapsto
\E[\Psi_{\bm\omega,\bm v_\theta}(A,C)]
\end{equation*}
is continuous on the finite input family. Shrinking \(\mathcal U_0^{(2)}\) if necessary gives an open neighborhood \(\mathcal U_1^{(2)}\) of the segment such that
\begin{equation*}
\sup_{\bm\omega\in\mathcal U_1^{(2)}}
\left|
\langle\nabla R_\ell(\bm\omega),\bm v_\theta\rangle
\right|
\le
\delta_{\rm drift}.
\end{equation*}

For the Hoeffding projections, define
\begin{equation*}
\mu_{\bm\omega}
=
\E[\Psi_{\bm\omega,\bm v_\theta}(A,C)],
\end{equation*}
\begin{equation*}
\Psi_{\bm\omega,\bm v_\theta}^{(1,0)}(a)
=
\E[\Psi_{\bm\omega,\bm v_\theta}(a,C)]-\mu_{\bm\omega},
\qquad
\Psi_{\bm\omega,\bm v_\theta}^{(0,1)}(c)
=
\E[\Psi_{\bm\omega,\bm v_\theta}(A,c)]-\mu_{\bm\omega},
\end{equation*}
and
\begin{equation*}
\Psi_{\bm\omega,\bm v_\theta}^{(1,1)}(a,c)
=
\Psi_{\bm\omega,\bm v_\theta}(a,c)
-\E[\Psi_{\bm\omega,\bm v_\theta}(a,C)]
-\E[\Psi_{\bm\omega,\bm v_\theta}(A,c)]
+\mu_{\bm\omega}.
\end{equation*}
On the line segment,
\begin{equation*}
\E[\Psi_{\bm\omega^{(2)}(\theta),\bm v_\theta}(a,C)]
=
\E[\Psi_{\bm\omega^{(2)}(\theta),\bm v_\theta}(A,C)]
\end{equation*}
for each \(a\in\{-1,1\}\), and the analogous identity holds when conditioning on \(C=c\). Thus
\begin{equation*}
\Var(\Psi_{\bm\omega^{(2)}(\theta),\bm v_\theta}^{(1,0)}(A))
+
\Var(\Psi_{\bm\omega^{(2)}(\theta),\bm v_\theta}^{(0,1)}(C))
=
0.
\end{equation*}
The two projection variances are continuous functions of \(\bm\omega\), again because the input family is finite and the score map is smooth. Hence there is an open neighborhood \(\mathcal U_2^{(2)}\subseteq\mathcal U_1^{(2)}\) of the same segment such that
\begin{equation*}
\sup_{\bm\omega\in\mathcal U_2^{(2)}}
\left\{
\Var(\Psi_{\bm\omega,\bm v_\theta}^{(1,0)}(A))
+
\Var(\Psi_{\bm\omega,\bm v_\theta}^{(0,1)}(C))
\right\}
\le
\varepsilon_1.
\end{equation*}

Finally, on the segment the residual pair-interaction term is
\begin{equation*}
\Psi_{\bm\omega^{(2)}(\theta),\bm v_\theta}^{(1,1)}(A,C)
=
-\gamma^2\cdot\theta\cdot A\cdot C,
\end{equation*}
so
\begin{equation*}
\Var(\Psi_{\bm\omega^{(2)}(\theta),\bm v_\theta}^{(1,1)}(A,C))
=
\gamma^4\cdot\theta^2
\ge
\gamma^2\cdot(M+s)^2
>
0.
\end{equation*}
By continuity and compactness, after shrinking once more to an open neighborhood
\begin{equation*}
\mathcal U_{\rm tr}^{(2)}
\subseteq
\mathcal U_2^{(2)}
\end{equation*}
of \(\mathcal L^{(2)}_{\gamma,M,s,r}\), we obtain
\begin{equation*}
\inf_{\bm\omega\in\mathcal U_{\rm tr}^{(2)}}
\Var(\Psi_{\bm\omega,\bm v_\theta}^{(1,1)}(A,C))
\ge
\frac{1}{2}\cdot\gamma^2\cdot(M+s)^2
>
0.
\end{equation*}
The severe-misranking, drift, first-order projection, and residual-variance conclusions all hold on \(\mathcal U_{\rm tr}^{(2)}\). The proof controls a full two-layer parameter neighborhood, while the bad component is the projected direction \(\bm v_\theta\).
\end{proof}

\section{Localized Escape-Time Bounds for Stopped SGD}
\label{app:localized-escape-time}

This section supplies the general probability tool used by Section~\ref{sec:bad-regions}. Formally, the escape time is the hitting time of the complement of a stopped region, or of a projected interval along a fixed direction. The main-text use is the projected interval along the optimistic steepest correction direction. The result is stated for a stopped SGD trajectory rather than for a specific loss, so that the escape-time argument remains separate from the AUC-specific Hoeffding calculation. Appendix~\ref{app:pairwise-bad-region-dynamics} determines whether the local directional second moment scales as \(B^{-1}\) or \(B^{-2}\); this section converts that scale into the corresponding projected escape-time lower bound used in the main text.

We first record two one-dimensional forms. They are useful when the relevant movement is projected onto a single direction \(\bm v\). The subsequent proposition is the vector-valued local version used by the main text.

\begin{proposition}[One-dimensional stopped local escape]
\label{prop:app-stopped-local-escape}
Consider an adapted process
\begin{equation*}
S_{t+1}=S_t+\Delta_{t+1},
\qquad
S_0=0,
\end{equation*}
with respect to a filtration \((\mathcal F_t)_{t\ge0}\), and define
\begin{equation*}
\tau_r=\inf\{t\ge0:\ |S_t|\ge r\}.
\end{equation*}
Assume that, on \(\{t<\tau_r\}\),
\begin{equation*}
\E[\Delta_{t+1}\mid\mathcal F_t]=0,
\qquad
\E[\Delta_{t+1}^2\mid\mathcal F_t]\le\nu_B^2
\qquad
\text{a.s.}
\end{equation*}
Assume also that \(\E[(S_{t\wedge\tau_r})^2]<\infty\) for every finite \(t\), and that either \(\Pbb(\tau_r<\infty)=1\) or \(\E[\tau_r]=\infty\). Then
\begin{equation*}
\nu_B=0
\quad\Longrightarrow\quad
\tau_r=\infty\ \text{a.s.},
\end{equation*}
and otherwise
\begin{equation*}
\E[\tau_r]\ge\frac{r^2}{\nu_B^2}.
\end{equation*}
Consequently, if \(\nu_B^2\le C_1/B\), then \(\E[\tau_r]\ge (r^2/C_1)\cdot B\); if \(\nu_B^2\le C_2/B^2\), then \(\E[\tau_r]\ge (r^2/C_2)\cdot B^2\).
\end{proposition}

\begin{proof}
If \(\nu_B=0\), then on \(\{t<\tau_r\}\),
\begin{equation*}
\E[\Delta_{t+1}^2\mid\mathcal F_t]\le0,
\end{equation*}
so \(\Delta_{t+1}=0\) almost surely before exit. Since \(S_0=0\), induction gives \(S_t=0\) for every \(t\), and \(\tau_r=\infty\) almost surely.

Assume \(\nu_B>0\). If \(\E[\tau_r]=\infty\), the lower bound is immediate. It remains to consider the case \(\tau_r<\infty\) almost surely and \(\E[\tau_r]<\infty\). Define
\begin{equation*}
M_t=S_{t\wedge\tau_r}.
\end{equation*}
Then
\begin{equation*}
M_{t+1}=M_t+\Delta_{t+1}\cdot\1_{\{t<\tau_r\}}.
\end{equation*}
Because \(\1_{\{t<\tau_r\}}\) is \(\mathcal F_t\)-measurable and \(\E[\Delta_{t+1}\mid\mathcal F_t]=0\) on \(\{t<\tau_r\}\), \((M_t)\) is a square-integrable martingale. Moreover,
\begin{equation*}
\E[M_{t+1}^2-M_t^2\mid\mathcal F_t]
=
\E[\Delta_{t+1}^2\cdot\1_{\{t<\tau_r\}}\mid\mathcal F_t]
\le
\nu_B^2\cdot\1_{\{t<\tau_r\}}.
\end{equation*}
Taking expectations and summing from \(0\) to \(n-1\) gives
\begin{equation*}
\E[M_n^2]
\le
\nu_B^2\cdot\E[n\wedge\tau_r].
\end{equation*}
On \(\{\tau_r\le n\}\), one has \(|M_n|=|S_{\tau_r}|\ge r\), hence
\begin{equation*}
r^2\cdot\Pbb(\tau_r\le n)
\le
\E[M_n^2]
\le
\nu_B^2\cdot\E[n\wedge\tau_r].
\end{equation*}
Letting \(n\to\infty\), the probability converges to \(1\), and monotone convergence gives
\begin{equation*}
r^2\le\nu_B^2\cdot\E[\tau_r].
\end{equation*}
The two scaling statements follow by substituting the corresponding upper bounds on \(\nu_B^2\).
\end{proof}

\begin{proposition}[Small-drift one-dimensional stopped local escape]
\label{prop:app-small-drift-stopped-local-escape}
Under the setup of Proposition~\ref{prop:app-stopped-local-escape}, replace the zero-drift condition by
\begin{equation*}
\left|\E[\Delta_{t+1}\mid\mathcal F_t]\right|\le m_B,
\qquad
\E[\Delta_{t+1}^2\mid\mathcal F_t]\le\nu_B^2
\qquad
\text{a.s. on }\{t<\tau_r\}.
\end{equation*}
Then for every integer \(T\ge1\),
\begin{equation*}
\bigl(r-m_B\cdot T\bigr)_+^2
\cdot
\Pbb(\tau_r\le T)
\le
\nu_B^2\cdot T.
\end{equation*}
\end{proposition}

\begin{proof}
For \(s\ge1\), define
\begin{equation*}
\mu_s=\E[\Delta_s\mid\mathcal F_{s-1}],
\qquad
\widetilde\mu_s=\mu_s\cdot\1_{\{s-1<\tau_r\}},
\end{equation*}
and
\begin{equation*}
\widetilde\xi_s
=
(\Delta_s-\mu_s)\cdot\1_{\{s-1<\tau_r\}}.
\end{equation*}
Let
\begin{equation*}
\widetilde A_t=\sum_{s=1}^t\widetilde\mu_s,
\qquad
\widetilde M_t=\sum_{s=1}^t\widetilde\xi_s.
\end{equation*}
The stopped process decomposes as
\begin{equation*}
S_{t\wedge\tau_r}
=
\widetilde A_t+\widetilde M_t.
\end{equation*}
Since \(\1_{\{s-1<\tau_r\}}\) is \(\mathcal F_{s-1}\)-measurable,
\begin{equation*}
\E[\widetilde\xi_s\mid\mathcal F_{s-1}]=0,
\end{equation*}
so \((\widetilde M_t)\) is a martingale. Its increments satisfy
\begin{equation*}
\E[\widetilde\xi_s^2\mid\mathcal F_{s-1}]
=
\1_{\{s-1<\tau_r\}}
\Var(\Delta_s\mid\mathcal F_{s-1})
\le
\nu_B^2.
\end{equation*}
Therefore
\begin{equation*}
\E[\widetilde M_T^2]\le\nu_B^2\cdot T.
\end{equation*}
Doob's \(L^2\) maximal inequality gives, for every \(a>0\),
\begin{equation*}
\Pbb\!\left(\max_{0\le t\le T}|\widetilde M_t|\ge a\right)
\le
\frac{\nu_B^2\cdot T}{a^2}.
\end{equation*}
The stopped drift satisfies
\begin{equation*}
|\widetilde A_t|
\le
\sum_{s=1}^t m_B\cdot\1_{\{s-1<\tau_r\}}
\le
m_B\cdot t
\le
m_B\cdot T
\qquad
\text{for }t\le T.
\end{equation*}
If \(\tau_r\le T\), then
\begin{equation*}
r\le |S_{\tau_r}|
\le
|\widetilde M_{\tau_r}|+|\widetilde A_{\tau_r}|
\le
|\widetilde M_{\tau_r}|+m_B\cdot T.
\end{equation*}
Hence
\begin{equation*}
\{\tau_r\le T\}
\subseteq
\left\{
\max_{0\le t\le T}|\widetilde M_t|
\ge
(r-m_B\cdot T)_+
\right\}.
\end{equation*}
If \(r-m_B\cdot T\le0\), the desired inequality is trivial. Otherwise apply Doob's inequality with \(a=r-m_B\cdot T\) and multiply both sides by \((r-m_B\cdot T)^2\).
\end{proof}

Let \((\mathcal F_t)_{t\ge0}\) be a filtration and let \((\theta_t)_{t\ge0}\) be an adapted parameter trajectory in \(\R^d\). Fix a measurable region \(\mathcal U\subset\R^d\), an initial point \(\theta_0\in\mathcal U\), and define the first hitting time of the complement of \(\mathcal U\) by
\begin{equation*}
\tau_{\mathcal U}
=
\inf\{t\ge0:\theta_t\notin\mathcal U\}.
\end{equation*}
Assume that the Euclidean ball \(\mathcal B(\theta_0,r)=\{\theta\in\R^d:\|\theta-\theta_0\|_2<r\}\) is contained in \(\mathcal U\). Then leaving \(\mathcal U\) before time \(T\) requires a displacement of at least \(r\) before time \(T\). The next proposition turns this geometric fact into an escape-time bound.

\begin{proposition}[Localized SGD escape-time lower bound]
\label{prop:localized-sgd-escape-time}
Let
\begin{equation*}
\Delta_{t+1}:=\theta_{t+1}-\theta_t
\end{equation*}
and suppose that, on the event \(\{t<\tau_{\mathcal U}\}\), the stopped increments satisfy
\begin{equation*}
\left\|\E[\Delta_{t+1}\mid \mathcal F_t]\right\|_2\le m_B,
\qquad
\E[\|\Delta_{t+1}\|_2^2\mid \mathcal F_t]\le \nu_B^2
\qquad\text{a.s.}
\end{equation*}
for constants \(m_B,\nu_B\ge0\). Then, for every integer \(T\ge1\),
\begin{equation*}
\bigl(r-m_B\cdot T\bigr)_+^2\cdot
\Pbb(\tau_{\mathcal U}\le T)
\le
\nu_B^2\cdot T.
\end{equation*}
If \(\nu_B=0\) and \(m_B=0\), then \(\tau_{\mathcal U}=\infty\) almost surely. If \(m_B\le 8\cdot \nu_B^2/r\) and \(0<\nu_B^2\le r^2/32\), then
\begin{equation*}
\E[\tau_{\mathcal U}]
\ge
\frac{3\cdot r^2}{128\cdot \nu_B^2}.
\end{equation*}
In particular, the positive-variance zero-drift case \(m_B=0\) is included. Consequently, under the same small-drift and small-step conditions, if \(\nu_B^2\le C_1/B\) for a \(B\)-independent constant \(C_1>0\), then for all sufficiently large \(B\),
\begin{equation*}
\E[\tau_{\mathcal U}]
\ge
\frac{3\cdot r^2}{128\cdot C_1}\cdot B,
\end{equation*}
whereas if \(\nu_B^2\le C_2/B^2\) for a \(B\)-independent constant \(C_2>0\), then for all sufficiently large \(B\),
\begin{equation*}
\E[\tau_{\mathcal U}]
\ge
\frac{3\cdot r^2}{128\cdot C_2}\cdot B^2.
\end{equation*}
\end{proposition}

\begin{proof}
Define the stopped increment
\begin{equation*}
\widetilde\Delta_{t+1}
:=
\Delta_{t+1}\cdot \1_{\{t<\tau_{\mathcal U}\}},
\end{equation*}
and decompose it into its predictable and centered parts:
\begin{equation*}
\bm{\mu}_{t+1}
:=
\E[\widetilde\Delta_{t+1}\mid \mathcal F_t],
\qquad
\bm{\xi}_{t+1}
:=
\widetilde\Delta_{t+1}-\bm{\mu}_{t+1}.
\end{equation*}
Let
\begin{equation*}
\bm{A}_T:=\sum_{t=0}^{T-1}\bm{\mu}_{t+1},
\qquad
\bm{M}_T:=\sum_{t=0}^{T-1}\bm{\xi}_{t+1}.
\end{equation*}
Then \((\bm{M}_T)_{T\ge0}\) is a square-integrable martingale with respect to \((\mathcal F_T)\), because
\begin{equation*}
\E[\bm{\xi}_{t+1}\mid \mathcal F_t]=\bm{0}.
\end{equation*}
Moreover, since \(\1_{\{t<\tau_{\mathcal U}\}}\) is \(\mathcal F_t\)-measurable,
\begin{equation*}
\|\bm{\mu}_{t+1}\|_2
=
\left\|
\1_{\{t<\tau_{\mathcal U}\}}
\E[\Delta_{t+1}\mid \mathcal F_t]
\right\|_2
\le
m_B\cdot \1_{\{t<\tau_{\mathcal U}\}}
\le m_B.
\end{equation*}
Therefore
\begin{equation*}
\|\bm{A}_T\|_2\le m_B\cdot T.
\end{equation*}
For the martingale term, use
\begin{equation*}
\|\bm{\xi}_{t+1}\|_2^2
\le
2\cdot \|\widetilde\Delta_{t+1}\|_2^2
+2\cdot \|\bm{\mu}_{t+1}\|_2^2.
\end{equation*}
This bound is sufficient for a finite-time small-drift inequality, but it introduces an unnecessary factor. A sharper route uses the Hilbert-space variance identity:
\begin{equation*}
\E[\|\bm{\xi}_{t+1}\|_2^2\mid \mathcal F_t]
=
\E[\|\widetilde\Delta_{t+1}\|_2^2\mid \mathcal F_t]
-\|\bm{\mu}_{t+1}\|_2^2
\le
\E[\|\widetilde\Delta_{t+1}\|_2^2\mid \mathcal F_t].
\end{equation*}
The stopped second moment is
\begin{equation*}
\E[\|\widetilde\Delta_{t+1}\|_2^2\mid \mathcal F_t]
=
\1_{\{t<\tau_{\mathcal U}\}}
\E[\|\Delta_{t+1}\|_2^2\mid \mathcal F_t]
\le
\nu_B^2\cdot \1_{\{t<\tau_{\mathcal U}\}}
\le
\nu_B^2.
\end{equation*}
Martingale increments are orthogonal in \(L^2\), so
\begin{equation*}
\E[\|\bm{M}_T\|_2^2]
=
\sum_{t=0}^{T-1}\E[\|\bm{\xi}_{t+1}\|_2^2]
\le
\nu_B^2\cdot T.
\end{equation*}
Thus Markov's inequality gives, for every \(a>0\),
\begin{equation*}
\Pbb(\|\bm{M}_T\|_2\ge a)
\le
\frac{\nu_B^2\cdot T}{a^2}.
\end{equation*}

If \(\tau_{\mathcal U}\le T\), then because \(\mathcal B(\theta_0,r)\subset\mathcal U\), the stopped trajectory must satisfy
\begin{equation*}
\|\theta_{\tau_{\mathcal U}}-\theta_0\|_2\ge r.
\end{equation*}
The stopped displacement equals
\begin{equation*}
\theta_{\tau_{\mathcal U}\wedge T}-\theta_0
=
\sum_{t=0}^{T-1}\widetilde\Delta_{t+1}
=
\bm{A}_T+\bm{M}_T.
\end{equation*}
Therefore, on \(\{\tau_{\mathcal U}\le T\}\),
\begin{equation*}
r
\le
\|\bm{A}_T+\bm{M}_T\|_2
\le
\|\bm{A}_T\|_2+\|\bm{M}_T\|_2
\le
m_B\cdot T+\|\bm{M}_T\|_2.
\end{equation*}
Hence
\begin{equation*}
\{\tau_{\mathcal U}\le T\}
\subseteq
\left\{
\|\bm{M}_T\|_2
\ge
\bigl(r-m_B\cdot T\bigr)_+
\right\}.
\end{equation*}
If \(r-m_B\cdot T\le0\), the desired inequality is trivial. Otherwise applying the preceding Markov bound with \(a=r-m_B\cdot T\) yields
\begin{equation*}
\Pbb(\tau_{\mathcal U}\le T)
\le
\frac{\nu_B^2\cdot T}{(r-m_B\cdot T)^2},
\end{equation*}
which proves the finite-time bound.

If \(\nu_B=0\) and \(m_B=0\), the finite-time bound gives
\begin{equation*}
r^2\cdot\Pbb(\tau_{\mathcal U}\le T)\le0
\end{equation*}
for every \(T\ge1\). Hence \(\Pbb(\tau_{\mathcal U}\le T)=0\) for every finite \(T\), and \(\tau_{\mathcal U}=\infty\) almost surely.

It remains to prove the expected escape-time statement. Let
\begin{equation*}
T_\star:=\left\lfloor\frac{r^2}{16\cdot \nu_B^2}\right\rfloor.
\end{equation*}
The condition \(\nu_B^2\le r^2/32\) implies \(r^2/(16\cdot \nu_B^2)\ge 2\), and hence
\begin{equation*}
T_\star
\ge
\frac{r^2}{32\cdot \nu_B^2}.
\end{equation*}
If \(m_B\le 8\cdot \nu_B^2/r\), then
\begin{equation*}
m_B\cdot T_\star
\le
\frac{8\cdot \nu_B^2}{r}\cdot
\frac{r^2}{16\cdot \nu_B^2}
=
\frac r2.
\end{equation*}
Applying the finite-time bound at the integer time \(T_\star\) gives
\begin{equation*}
\Pbb(\tau_{\mathcal U}\le T_\star)
\le
\frac{\nu_B^2\cdot T_\star}{(r/2)^2}
\le
\frac{\nu_B^2}{(r/2)^2}\cdot
\frac{r^2}{16\cdot \nu_B^2}
=
\frac{1}{4}.
\end{equation*}
Therefore
\begin{equation*}
\Pbb(\tau_{\mathcal U}>T_\star)\ge \frac34.
\end{equation*}
Since \(\tau_{\mathcal U}\) is nonnegative,
\begin{equation*}
\E[\tau_{\mathcal U}]
\ge
T_\star\cdot \Pbb(\tau_{\mathcal U}>T_\star)
\ge
\frac{3}{4}\cdot \frac{r^2}{32\cdot \nu_B^2}
=
\frac{3\cdot r^2}{128\cdot \nu_B^2}.
\end{equation*}
The final two bounds follow by substituting \(\nu_B^2\le C_1/B\) and \(\nu_B^2\le C_2/B^2\), respectively.
\end{proof}

\begin{corollary}[Projected escape-time consequence of first-order variance collapse]
\label{cor:escape-time-variance-collapse}
Fix a unit direction \(\bm v\), and define the projected process
\begin{equation*}
S_t=\langle\theta_t-\theta_0,\bm v\rangle,
\qquad
\tau_{\bm v,r}=\inf\{t\ge0:\ |S_t|\ge r\}.
\end{equation*}
If, before \(\tau_{\bm v,r}\), the projected stopped SGD increment satisfies the small-drift and small-step conditions in Proposition~\ref{prop:localized-sgd-escape-time} and has conditional second moment of order \(B^{-1}\), then the proposition gives an order-\(B\) lower-bound scale for \(\E[\tau_{\bm v,r}]\). If the first-order stochastic term vanishes in direction \(\bm v\) and the projected conditional second moment is of order \(B^{-2}\), then the lower-bound scale becomes order \(B^2\).
\end{corollary}

\begin{proof}
Apply Proposition~\ref{prop:localized-sgd-escape-time} to the one-dimensional trajectory \(S_t\) with region \((-r,r)\), first with \(\nu_B^2=O(B^{-1})\) and then with \(\nu_B^2=O(B^{-2})\). The AUC-specific reason these two cases arise is the two-sample Hoeffding variance split of the directional pairwise mini-batch gradient: nonzero first-order projections produce a \(B^{-1}\) term, while vanishing first-order projections leave only the \(B^{-2}\) canonical term. For a fixed step size, the centered variance contribution is multiplied by \(\eta^2\). The corollary applies to the conditional second moment after the projected drift is separately controlled by the small-drift condition, so the order in \(B\) is unchanged.
\end{proof}

\section{Limitations}
\label{app:limitations}

The theory studies projected escape along the steepest severe-error correction direction because this is an optimistic direction for escape. Among unit directions, it is the most direct local direction for reducing the severe-error functional, and therefore gives a favorable projected direction to test. Since there are infinitely many possible directions and they cannot all be enumerated, we focus on this steepest and most optimistic escape direction. If escape is slow even along this direction, escape along other directions may be even harder.

\clearpage
\section*{NeurIPS Paper Checklist}

\begin{enumerate}

\item {\bf Claims}
    \item[] Question: Do the main claims made in the abstract and introduction accurately reflect the paper's contributions and scope?
    \item[] Answer: \answerYes{}
    \item[] Justification: The abstract and introduction state the paper's theoretical mechanism, surrogate-weight claims, warm-up/pretraining interpretation, and experimental scope.
    \item[] Guidelines:
    \begin{itemize}
        \item The answer \answerNA{} means that the abstract and introduction do not include the claims made in the paper.
        \item The abstract and/or introduction should clearly state the claims made, including the contributions made in the paper and important assumptions and limitations. A \answerNo{} or \answerNA{} answer to this question will not be perceived well by the reviewers. 
        \item The claims made should match theoretical and experimental results, and reflect how much the results can be expected to generalize to other settings. 
        \item It is fine to include aspirational goals as motivation as long as it is clear that these goals are not attained by the paper. 
    \end{itemize}

\item {\bf Limitations}
    \item[] Question: Does the paper discuss the limitations of the work performed by the authors?
    \item[] Answer: \answerYes{}
    \item[] Justification: Appendix~\ref{app:limitations} explains why the theory studies projected escape along the steepest severe-error correction direction as an optimistic direction among infinitely many possible directions.
    \item[] Guidelines:
    \begin{itemize}
        \item The answer \answerNA{} means that the paper has no limitation while the answer \answerNo{} means that the paper has limitations, but those are not discussed in the paper. 
        \item The authors are encouraged to create a separate ``Limitations'' section in their paper.
        \item The paper should point out any strong assumptions and how robust the results are to violations of these assumptions (e.g., independence assumptions, noiseless settings, model well-specification, asymptotic approximations only holding locally). The authors should reflect on how these assumptions might be violated in practice and what the implications would be.
        \item The authors should reflect on the scope of the claims made, e.g., if the approach was only tested on a few datasets or with a few runs. In general, empirical results often depend on implicit assumptions, which should be articulated.
        \item The authors should reflect on the factors that influence the performance of the approach. For example, a facial recognition algorithm may perform poorly when image resolution is low or images are taken in low lighting. Or a speech-to-text system might not be used reliably to provide closed captions for online lectures because it fails to handle technical jargon.
        \item The authors should discuss the computational efficiency of the proposed algorithms and how they scale with dataset size.
        \item If applicable, the authors should discuss possible limitations of their approach to address problems of privacy and fairness.
        \item While the authors might fear that complete honesty about limitations might be used by reviewers as grounds for rejection, a worse outcome might be that reviewers discover limitations that aren't acknowledged in the paper. The authors should use their best judgment and recognize that individual actions in favor of transparency play an important role in developing norms that preserve the integrity of the community. Reviewers will be specifically instructed to not penalize honesty concerning limitations.
    \end{itemize}

\item {\bf Theory assumptions and proofs}
    \item[] Question: For each theoretical result, does the paper provide the full set of assumptions and a complete (and correct) proof?
    \item[] Answer: \answerYes{}
    \item[] Justification: The main text states the principal assumptions and proof sketches, and Appendices~\ref{app:pairwise-sgd-dynamics}--\ref{app:localized-escape-time} provide the formal statements and proofs used by the main results.
    \item[] Guidelines:
    \begin{itemize}
        \item The answer \answerNA{} means that the paper does not include theoretical results. 
        \item All the theorems, formulas, and proofs in the paper should be numbered and cross-referenced.
        \item All assumptions should be clearly stated or referenced in the statement of any theorems.
        \item The proofs can either appear in the main paper or the supplemental material, but if they appear in the supplemental material, the authors are encouraged to provide a short proof sketch to provide intuition. 
        \item Inversely, any informal proof provided in the core of the paper should be complemented by formal proofs provided in appendix or supplemental material.
        \item Theorems and Lemmas that the proof relies upon should be properly referenced. 
    \end{itemize}

    \item {\bf Experimental result reproducibility}
    \item[] Question: Does the paper fully disclose all the information needed to reproduce the main experimental results of the paper to the extent that it affects the main claims and/or conclusions of the paper (regardless of whether the code and data are provided or not)?
    \item[] Answer: \answerYes{}
    \item[] Justification: Appendix~\ref{app:real-experiment-protocol}--\ref{app:training-details} gives task construction, split sizes, seeds, model settings, optimization hyperparameters, warm-up and margin grids, checkpoint selection, and environment details.
    \item[] Guidelines:
    \begin{itemize}
        \item The answer \answerNA{} means that the paper does not include experiments.
        \item If the paper includes experiments, a \answerNo{} answer to this question will not be perceived well by the reviewers: Making the paper reproducible is important, regardless of whether the code and data are provided or not.
        \item If the contribution is a dataset and\slash or model, the authors should describe the steps taken to make their results reproducible or verifiable. 
        \item Depending on the contribution, reproducibility can be accomplished in various ways. For example, if the contribution is a novel architecture, describing the architecture fully might suffice, or if the contribution is a specific model and empirical evaluation, it may be necessary to either make it possible for others to replicate the model with the same dataset, or provide access to the model. In general, releasing code and data is often one good way to accomplish this, but reproducibility can also be provided via detailed instructions for how to replicate the results, access to a hosted model (e.g., in the case of a large language model), releasing of a model checkpoint, or other means that are appropriate to the research performed.
        \item While NeurIPS does not require releasing code, the conference does require all submissions to provide some reasonable avenue for reproducibility, which may depend on the nature of the contribution. For example
        \begin{enumerate}
            \item If the contribution is primarily a new algorithm, the paper should make it clear how to reproduce that algorithm.
            \item If the contribution is primarily a new model architecture, the paper should describe the architecture clearly and fully.
            \item If the contribution is a new model (e.g., a large language model), then there should either be a way to access this model for reproducing the results or a way to reproduce the model (e.g., with an open-source dataset or instructions for how to construct the dataset).
            \item We recognize that reproducibility may be tricky in some cases, in which case authors are welcome to describe the particular way they provide for reproducibility. In the case of closed-source models, it may be that access to the model is limited in some way (e.g., to registered users), but it should be possible for other researchers to have some path to reproducing or verifying the results.
        \end{enumerate}
    \end{itemize}

\item {\bf Open access to data and code}
    \item[] Question: Does the paper provide open access to the data and code, with sufficient instructions to faithfully reproduce the main experimental results, as described in supplemental material?
    \item[] Answer: \answerNo{}
    \item[] Justification: Appendix~\ref{app:datasets-tasks} lists public data sources and Appendix~\ref{app:training-details} gives the experimental protocol, but this submission does not include an anonymized code package with exact reproduction commands.
    \item[] Guidelines:
    \begin{itemize}
        \item The answer \answerNA{} means that paper does not include experiments requiring code.
        \item Please see the NeurIPS code and data submission guidelines (\url{https://neurips.cc/public/guides/CodeSubmissionPolicy}) for more details.
        \item While we encourage the release of code and data, we understand that this might not be possible, so \answerNo{} is an acceptable answer. Papers cannot be rejected simply for not including code, unless this is central to the contribution (e.g., for a new open-source benchmark).
        \item The instructions should contain the exact command and environment needed to run to reproduce the results. See the NeurIPS code and data submission guidelines (\url{https://neurips.cc/public/guides/CodeSubmissionPolicy}) for more details.
        \item The authors should provide instructions on data access and preparation, including how to access the raw data, preprocessed data, intermediate data, and generated data, etc.
        \item The authors should provide scripts to reproduce all experimental results for the new proposed method and baselines. If only a subset of experiments are reproducible, they should state which ones are omitted from the script and why.
        \item At submission time, to preserve anonymity, the authors should release anonymized versions (if applicable).
        \item Providing as much information as possible in supplemental material (appended to the paper) is recommended, but including URLs to data and code is permitted.
    \end{itemize}

\item {\bf Experimental setting/details}
    \item[] Question: Does the paper specify all the training and test details (e.g., data splits, hyperparameters, how they were chosen, type of optimizer) necessary to understand the results?
    \item[] Answer: \answerYes{}
    \item[] Justification: Appendix~\ref{app:datasets-tasks} reports dataset construction and split sizes, and Appendix~\ref{app:training-details} specifies models, optimizers, learning rates, batch sizes, schedules, warm-up and margin grids, checkpointing, and selection rules.
    \item[] Guidelines:
    \begin{itemize}
        \item The answer \answerNA{} means that the paper does not include experiments.
        \item The experimental setting should be presented in the core of the paper to a level of detail that is necessary to appreciate the results and make sense of them.
        \item The full details can be provided either with the code, in appendix, or as supplemental material.
    \end{itemize}

\item {\bf Experiment statistical significance}
    \item[] Question: Does the paper report error bars suitably and correctly defined or other appropriate information about the statistical significance of the experiments?
    \item[] Answer: \answerYes{}
    \item[] Justification: The main and expanded appendix tables report mean held-out AUC with sample standard deviation, and Appendix~\ref{app:real-experiment-protocol}--\ref{app:training-details} defines the reported quantity, seed set, and checkpoint-selection protocol.
    \item[] Guidelines:
    \begin{itemize}
        \item The answer \answerNA{} means that the paper does not include experiments.
        \item The authors should answer \answerYes{} if the results are accompanied by error bars, confidence intervals, or statistical significance tests, at least for the experiments that support the main claims of the paper.
        \item The factors of variability that the error bars are capturing should be clearly stated (for example, train/test split, initialization, random drawing of some parameter, or overall run with given experimental conditions).
        \item The method for calculating the error bars should be explained (closed form formula, call to a library function, bootstrap, etc.)
        \item The assumptions made should be given (e.g., Normally distributed errors).
        \item It should be clear whether the error bar is the standard deviation or the standard error of the mean.
        \item It is OK to report 1-sigma error bars, but one should state it. The authors should preferably report a 2-sigma error bar than state that they have a 96\% CI, if the hypothesis of Normality of errors is not verified.
        \item For asymmetric distributions, the authors should be careful not to show in tables or figures symmetric error bars that would yield results that are out of range (e.g., negative error rates).
        \item If error bars are reported in tables or plots, the authors should explain in the text how they were calculated and reference the corresponding figures or tables in the text.
    \end{itemize}

\item {\bf Experiments compute resources}
    \item[] Question: For each experiment, does the paper provide sufficient information on the computer resources (type of compute workers, memory, time of execution) needed to reproduce the experiments?
    \item[] Answer: \answerYes{}
    \item[] Justification: Appendix~\ref{app:training-details} reports the GPU type, memory, one-GPU-per-run scheduling, and a log-derived total compute estimate for all experiments involved in this paper.
    \item[] Guidelines:
    \begin{itemize}
        \item The answer \answerNA{} means that the paper does not include experiments.
        \item The paper should indicate the type of compute workers CPU or GPU, internal cluster, or cloud provider, including relevant memory and storage.
        \item The paper should provide the amount of compute required for each of the individual experimental runs as well as estimate the total compute. 
        \item The paper should disclose whether the full research project required more compute than the experiments reported in the paper (e.g., preliminary or failed experiments that didn't make it into the paper). 
    \end{itemize}
    
\item {\bf Code of ethics}
    \item[] Question: Does the research conducted in the paper conform, in every respect, with the NeurIPS Code of Ethics \url{https://neurips.cc/public/EthicsGuidelines}?
    \item[] Answer: \answerYes{}
    \item[] Justification: Appendix~\ref{app:ethics-impacts} states that the work studies optimization dynamics, uses existing public datasets and pretrained models, and does not involve new human-subject data collection or a deployed decision system.
    \item[] Guidelines:
    \begin{itemize}
        \item The answer \answerNA{} means that the authors have not reviewed the NeurIPS Code of Ethics.
        \item If the authors answer \answerNo, they should explain the special circumstances that require a deviation from the Code of Ethics.
        \item The authors should make sure to preserve anonymity (e.g., if there is a special consideration due to laws or regulations in their jurisdiction).
    \end{itemize}

\item {\bf Broader impacts}
    \item[] Question: Does the paper discuss both potential positive societal impacts and negative societal impacts of the work performed?
    \item[] Answer: \answerYes{}
    \item[] Justification: Appendix~\ref{app:ethics-impacts} discusses potential positive impacts from improved understanding of AUC optimization and potential negative impacts from deployment in high-stakes domains without validation, calibration, subgroup analysis, or oversight.
    \item[] Guidelines:
    \begin{itemize}
        \item The answer \answerNA{} means that there is no societal impact of the work performed.
        \item If the authors answer \answerNA{} or \answerNo, they should explain why their work has no societal impact or why the paper does not address societal impact.
        \item Examples of negative societal impacts include potential malicious or unintended uses (e.g., disinformation, generating fake profiles, surveillance), fairness considerations (e.g., deployment of technologies that could make decisions that unfairly impact specific groups), privacy considerations, and security considerations.
        \item The conference expects that many papers will be foundational research and not tied to particular applications, let alone deployments. However, if there is a direct path to any negative applications, the authors should point it out. For example, it is legitimate to point out that an improvement in the quality of generative models could be used to generate Deepfakes for disinformation. On the other hand, it is not needed to point out that a generic algorithm for optimizing neural networks could enable people to train models that generate Deepfakes faster.
        \item The authors should consider possible harms that could arise when the technology is being used as intended and functioning correctly, harms that could arise when the technology is being used as intended but gives incorrect results, and harms following from (intentional or unintentional) misuse of the technology.
        \item If there are negative societal impacts, the authors could also discuss possible mitigation strategies (e.g., gated release of models, providing defenses in addition to attacks, mechanisms for monitoring misuse, mechanisms to monitor how a system learns from feedback over time, improving the efficiency and accessibility of ML).
    \end{itemize}
    
\item {\bf Safeguards}
    \item[] Question: Does the paper describe safeguards that have been put in place for responsible release of data or models that have a high risk for misuse (e.g., pre-trained language models, image generators, or scraped datasets)?
    \item[] Answer: \answerNA{}
    \item[] Justification: Appendix~\ref{app:ethics-impacts} states that the work does not release a new high-risk model, scraped dataset, or human-subject dataset. The paper reports experiments on existing public datasets and pretrained models.
    \item[] Guidelines:
    \begin{itemize}
        \item The answer \answerNA{} means that the paper poses no such risks.
        \item Released models that have a high risk for misuse or dual-use should be released with necessary safeguards to allow for controlled use of the model, for example by requiring that users adhere to usage guidelines or restrictions to access the model or implementing safety filters. 
        \item Datasets that have been scraped from the Internet could pose safety risks. The authors should describe how they avoided releasing unsafe images.
        \item We recognize that providing effective safeguards is challenging, and many papers do not require this, but we encourage authors to take this into account and make a best faith effort.
    \end{itemize}

\item {\bf Licenses for existing assets}
    \item[] Question: Are the creators or original owners of assets (e.g., code, data, models), used in the paper, properly credited and are the license and terms of use explicitly mentioned and properly respected?
    \item[] Answer: \answerYes{}
    \item[] Justification: Appendix~\ref{app:assets-licenses} credits the public datasets, pretrained models, and standard software assets, explicitly lists the corresponding licenses or terms, and states that the manuscript does not redistribute third-party assets.
    \item[] Guidelines:
    \begin{itemize}
        \item The answer \answerNA{} means that the paper does not use existing assets.
        \item The authors should cite the original paper that produced the code package or dataset.
        \item The authors should state which version of the asset is used and, if possible, include a URL.
        \item The name of the license (e.g., CC-BY 4.0) should be included for each asset.
        \item For scraped data from a particular source (e.g., website), the copyright and terms of service of that source should be provided.
        \item If assets are released, the license, copyright information, and terms of use in the package should be provided. For popular datasets, \url{paperswithcode.com/datasets} has curated licenses for some datasets. Their licensing guide can help determine the license of a dataset.
        \item For existing datasets that are re-packaged, both the original license and the license of the derived asset (if it has changed) should be provided.
        \item If this information is not available online, the authors are encouraged to reach out to the asset's creators.
    \end{itemize}

\item {\bf New assets}
    \item[] Question: Are new assets introduced in the paper well documented and is the documentation provided alongside the assets?
    \item[] Answer: \answerNA{}
    \item[] Justification: The paper does not release a new dataset, model checkpoint, or benchmark asset; it reports experiments on existing public datasets and pretrained models, as stated in Appendix~\ref{app:assets-licenses}.
    \item[] Guidelines:
    \begin{itemize}
        \item The answer \answerNA{} means that the paper does not release new assets.
        \item Researchers should communicate the details of the dataset\slash code\slash model as part of their submissions via structured templates. This includes details about training, license, limitations, etc. 
        \item The paper should discuss whether and how consent was obtained from people whose asset is used.
        \item At submission time, remember to anonymize your assets (if applicable). You can either create an anonymized URL or include an anonymized zip file.
    \end{itemize}

\item {\bf Crowdsourcing and research with human subjects}
    \item[] Question: For crowdsourcing experiments and research with human subjects, does the paper include the full text of instructions given to participants and screenshots, if applicable, as well as details about compensation (if any)? 
    \item[] Answer: \answerNA{}
    \item[] Justification: Appendix~\ref{app:ethics-impacts} states that the work does not involve new data collection, crowdsourcing, interaction with human subjects, or interventions on human participants.
    \item[] Guidelines:
    \begin{itemize}
        \item The answer \answerNA{} means that the paper does not involve crowdsourcing nor research with human subjects.
        \item Including this information in the supplemental material is fine, but if the main contribution of the paper involves human subjects, then as much detail as possible should be included in the main paper. 
        \item According to the NeurIPS Code of Ethics, workers involved in data collection, curation, or other labor should be paid at least the minimum wage in the country of the data collector. 
    \end{itemize}

\item {\bf Institutional review board (IRB) approvals or equivalent for research with human subjects}
    \item[] Question: Does the paper describe potential risks incurred by study participants, whether such risks were disclosed to the subjects, and whether Institutional Review Board (IRB) approvals (or an equivalent approval/review based on the requirements of your country or institution) were obtained?
    \item[] Answer: \answerNA{}
    \item[] Justification: Appendix~\ref{app:ethics-impacts} states that the work does not involve crowdsourcing, human-subject interaction, interventions on human participants, or a new human-subject dataset.
    \item[] Guidelines:
    \begin{itemize}
        \item The answer \answerNA{} means that the paper does not involve crowdsourcing nor research with human subjects.
        \item Depending on the country in which research is conducted, IRB approval (or equivalent) may be required for any human subjects research. If you obtained IRB approval, you should clearly state this in the paper. 
        \item We recognize that the procedures for this may vary significantly between institutions and locations, and we expect authors to adhere to the NeurIPS Code of Ethics and the guidelines for their institution. 
        \item For initial submissions, do not include any information that would break anonymity (if applicable), such as the institution conducting the review.
    \end{itemize}

\item {\bf Declaration of LLM usage}
    \item[] Question: Does the paper describe the usage of LLMs if it is an important, original, or non-standard component of the core methods in this research? Note that if the LLM is used only for writing, editing, or formatting purposes and does \emph{not} impact the core methodology, scientific rigor, or originality of the research, declaration is not required.
    \item[] Answer: \answerNA{}
    \item[] Justification: The core method development does not involve LLMs as important, original, or non-standard components. Qwen2.5-1.5B-Instruct is used only as one experimental model family, as described in Appendix~\ref{app:training-details}.
    \item[] Guidelines:
    \begin{itemize}
        \item The answer \answerNA{} means that the core method development in this research does not involve LLMs as any important, original, or non-standard components.
        \item Please refer to our LLM policy in the NeurIPS handbook for what should or should not be described.
    \end{itemize}

\end{enumerate}

\end{document}